\documentclass[11pt]{article}

\usepackage[margin=1in]{geometry}
\usepackage[T1]{fontenc}
\usepackage{lmodern}

\usepackage[round]{natbib}

\usepackage{microtype}
\usepackage{graphicx}
\usepackage{subcaption}
\usepackage{booktabs}
\usepackage{algorithm}
\usepackage{algorithmic}
\usepackage{wrapfig}
\usepackage{authblk} 
\usepackage[colorlinks=true,linkcolor=blue!60!black,citecolor=blue!60!black,urlcolor=blue!60!black]{hyperref}

\usepackage{amsmath,amssymb,mathtools,amsthm}
\usepackage{tikz}
\usetikzlibrary{arrows.meta,positioning,calc}
\usepackage[textsize=tiny]{todonotes}

\theoremstyle{plain}
\newtheorem{theorem}{Theorem}[section]
\newtheorem{proposition}[theorem]{Proposition}
\newtheorem{lemma}[theorem]{Lemma}
\newtheorem{corollary}[theorem]{Corollary}
\theoremstyle{definition}
\newtheorem{definition}[theorem]{Definition}
\newtheorem{assumption}[theorem]{Assumption}
\theoremstyle{remark}
\newtheorem{remark}[theorem]{Remark}

\newcommand{\D}{D}

\newcommand{\Di}{D^{-i}}
\newcommand{\A}{\mathcal{A}}
\newcommand{\U}{\mathcal{U}}
\newcommand{\Prob}{\mathbb{P}}

\newcommand{\R}{\mathbb{R}}

\newcommand{\eps}{\varepsilon}
\newcommand{\del}{\delta}
\newcommand{\slearn}{\sigma_{\mathrm{learn}}}
\newcommand{\sunl}{\sigma_{\mathrm{unlearn}}}
\newcommand{\dm}{\delta_{\mathrm m}}
\newcommand{\ds}{\delta_{\mathrm s}}
\newcommand{\epsgdp}{\varepsilon_{\mathrm{GDP}}}

\title{Certified Per-Instance Unlearning Using Individual Sensitivity Bounds}

\author[1]{Hanna Benarroch\thanks{Corresponding author: \texttt{hanna.benarroch@ens.psl.eu}}}
\author[2]{Jamal Atif}
\author[1]{Olivier Cappé}

\affil[1]{DI ENS, École normale supérieure, Université PSL, CNRS, 75005 Paris, France}
\affil[2]{CMAP, École polytechnique, Institut Polytechnique de Paris, 91120 Palaiseau, France}

\begin{document}

\maketitle


\begin{abstract}
Certified machine unlearning can be achieved through noise injection, with 
differential privacy guarantees obtained by calibrating the noise to worst-case 
sensitivity. Yet, the influence of individual training points is highly heterogeneous, 
making such uniform calibration unnecessarily conservative for most deletion requests.
We consider \emph{certified per-instance unlearning}, a data-dependent notion 
of certified deletion in which the unlearning mechanism and its certificate can 
adapt to the point being removed. Our framework turns point-specific sensitivity 
information into formal unlearning guarantees, allowing 
the amount of noise to reflect the difficulty of the deletion request rather than 
the worst case. 
When training is deterministic, we derive explicit 
per-instance unlearning certificates. For noisy gradient descent, in the 
case of ridge regression, similar certificates are obtained through high-probability sensitivity 
bounds. Experiments on NYU-Depth with frozen DINOv2 features show substantial variation 
in deletion difficulty across points, and demonstrate that per-instance calibration can 
substantially improve utility over uniform worst-case calibration.
\end{abstract}

\section{Introduction}
\label{sec:intro}

Modern machine learning systems increasingly operate under regulatory and contractual
constraints that may necessitate the \emph{post-hoc} removal of individual training
examples, such as those arising from the ``right to be forgotten''
\citep{marino2025position}. The gold standard is to retrain the model from scratch on the
dataset with the target example removed, but such retraining is often computationally
prohibitive in practice. Machine unlearning seeks to approximate the effect of retraining
at a substantially lower cost, either by enforcing exact equivalence with retraining
\emph{(exact unlearning)} \citep{cao2015towards,bourtoule2021machine}, or by matching the
retrained model only in distribution using randomized unlearning \emph{(approximate unlearning)}
\citep{ginart2019,sekhari2021}. We focus on approximate unlearning, and, within it, on
\emph{certified} unlearning, where the deletion procedure comes with explicit, provable
guarantees. Such guarantees are naturally connected to differential privacy (DP), which
quantifies how the output distribution of a randomized mechanism changes when a single
training point is removed. While DP provides a natural language for reasoning about
certified deletion, its guarantees do not transfer to the unlearning setting unchanged:
differential privacy certifies a mechanism \emph{ex ante}, uniformly over all datasets and
all points, whereas unlearning is intrinsically \emph{post-hoc}, as it concerns a specific
trained model, a specific dataset, and a specific deletion request.

This has two consequences. The first concerns how unlearning guarantees should be
\emph{defined}. Several existing approaches measure unlearning against a full retraining
from scratch \citep{neel2020descenttodeletegradientbasedmethodsmachine,guo2020certified,
koloskova2025certifiedunlearningneuralnetworks,mu2025rewindtodeletecertifiedmachineunlearning},
which in particular requires the learning algorithm itself to be randomized. Following recent work
\citep{lu2025systemawareunlearningalgorithmsuse,basaran2025certifiedunlearningapproachaccess,VanWaerebeke2025},
we instead adopt the definition of \citet{sekhari2021}, which compares two executions of
the same learning-and-unlearning pipeline on adjacent datasets.
The second consequence concerns how guarantees should be \emph{calibrated} across data points.
By construction, uniform DP guarantees treat all data points identically,
regardless of their actual influence on the learned model. Many unlearning methods
inherit this uniform calibration
\citep{chien2025langevinunlearningnewperspective,neel2020descenttodeletegradientbasedmethodsmachine,
koloskova2025certifiedunlearningneuralnetworks,mu2025rewindtodeletecertifiedmachineunlearning}.
Yet per-example influence is highly heterogeneous, as prior work has shown empirically
\citep{thudi2024gradientslookalikesensitivity,zhao2024makesunlearninghard}.
In particular, \citet{sepahvand2025leveragingperinstanceprivacymachine} show that
\emph{per-instance privacy losses} correlate with how hard a data point is to unlearn.
To our knowledge, however, such per-instance quantities have not yet been used to
calibrate an unlearning mechanism so that it comes with a formal certificate.

Inspired by this observation, we examine the limitations of existing unlearning 
definitions for capturing heterogeneous data influence, and introduce 
\emph{certified per-instance unlearning}, which allows both the unlearning 
procedure and its certificate to depend on the specific deletion request.
To instantiate this definition, we study noisy gradient-based unlearning, following 
\citet{chien2025langevinunlearningnewperspective}. Building on the 
shifted-interpolation accountant of \citet{bok2024shiftedinterpolationdifferentialprivacy}, 
our accounting is driven by a sequence of \emph{per-instance sensitivities}, 
measuring, at each iteration, the change induced by removing the 
point to be forgotten. These individual, point-dependent discrepancies are propagated through 
the optimization iterations and define the individual unlearning certificates.

Our analysis covers two ways in which models are trained in practice. When the training is
deterministic, the per-instance sensitivities can be computed exactly along the training
trajectory and yield explicit certificates for smooth convex objectives.
When the training is noisy, for instance to protect privacy at training time, the learning noise contributes to
the certificate, but the individual gradients become random and can only be controlled
probabilistically. We carry out this analysis for ridge regression, where the linear-Gaussian
structure of the iterates makes such high-probability bounds explicit and leads to \((\eps,\del)\)
per-instance unlearning guarantees. 
Finally, both results carry over to group deletion
(i.e., removing more than one data point)  and
to cyclic mini-batch training and unlearning, the regime of choice for larger problems, in
which our certificates remain non-vacuous (see Appendices~\ref{app:group-unlearning} and~\ref{app:mini-batch}).
To our knowledge, these are the first certified unlearning results that calibrate 
the unlearning noise to the point being deleted.

In Section~\ref{sec:experiments-linear}, we show empirically that our per-instance
calibration is not merely formal: the amount of noise required to obtain the same
unlearning certificate can vary by about \emph{a factor of ten} across deletion points. We
also compare our method against Langevin unlearning, learning-time privacy via
DP-GD, and Newton-step unlearning. All proofs are deferred to the appendix.

\section{Related Work}
\label{sec:related}
 
\paragraph{From exact to approximate unlearning.}
Exact unlearning asks the post-deletion model to coincide with the minimizer of the retain
ERM, either by construction \citep{cao2015towards} or by restructuring training so that this
minimizer becomes cheap to recover: SISA \citep{bourtoule2021machine} shards the dataset and
retrains only the shard that contained the deleted point. The target itself
dissolves as soon as training is run with (stochastic) gradient methods: the learned parameter
is not the minimizer of the ERM equation, so no exact solution is left to reproduce. Even
for ridge regression, the closed-form leave-one-out update that \citet{izzo2021approx}
exploit is valid only at the exact penalized least-squares solution
(Appendix~\ref{app:exact-ridge}). This motivates the probabilistic formulations of unlearning by
\citet{ginart2019} and \citet{sekhari2021} (approximate unlearning), which we refine to the per-instance level
(Section~\ref{sec:definition}). 
 
\paragraph{First-order certified unlearning.}
Descent-to-Delete \citep{neel2020descenttodeletegradientbasedmethodsmachine} runs gradient
steps on the retain ERM objective and perturbs the output.
\citet{chien2025langevinunlearningnewperspective} show that injecting noise at every step is
less conservative than perturbing once at the end, and propose the Langevin procedure we
build on, accounted for through R\'enyi divergences. 
\citet{mu2025rewindtodeletecertifiedmachineunlearning} rewind to an earlier checkpoint,
which depends less on the deleted point, before resuming training on the retain objective;
\citet{koloskova2025certifiedunlearningneuralnetworks} certify noisy gradient-based
unlearning through GDP and the contraction of the hockey-stick divergence. We also achieve 
unlearning via noisy fine-tuning, but adapt the accountant to calibrate the noise to the sensitivity of the  
data point being deleted (Section~\ref{sec:accounting}).
 
\paragraph{Second-order certified unlearning.}
\citet{guo2020certified} and \citet{sekhari2021} instead delete by a Newton step on a
strongly convex objective (the influence-function approximation of the leave-one-out
minimizer), certified by absorbing its residual into Gaussian noise, injected respectively into
the training objective or into the final parameter. Since these corrections are exact only 
at the ERM minimizer, their certificates inherit the optimization error at finite training 
time and may become vacuous away from convergence. Our guarantee, by contrast, holds directly 
along the finite optimization trajectory; Section~\ref{sec:experiments-linear} quantifies the
resulting gap.
 
\paragraph{Individual privacy.}
Per-instance DP \citep{wang2018perinstancedifferentialprivacy} and individual privacy accounting
\citep{feldman2021individual,koskela2023individual} measure the privacy loss of each example
under a fixed mechanism, and data-dependent calibration of privacy parameters has been studied
by \citet{ligett2017accuracy} and \citet{redberg2021privately}. These works quantify or adapt
privacy before release, whereas we calibrate a post-hoc deletion to the request at hand.

\section{Certified Per-Instance Unlearning}
\label{sec:certified-per-instance}

\subsection{Problem Setting}
\label{sec:setting}

We consider supervised learning with a training dataset
$\D=\{(x_j,y_j)\}_{j=1}^n \subset \R^p \times \R^d$.
Given a per-example loss $\ell(\theta; x,y)$, 
the empirical risk minimization (ERM) objective is
\begin{equation}
\label{eq:erm_obj}
f_{\D}(\theta) \;:=\;
\sum_{j=1}^n \ell(\theta; x_j,y_j).
\end{equation}
For a forget set $\D_f=\{(x_j,y_j)\}_{j\in I_f}\subset\D$, $f_{\D\setminus\D_f}$ is the objective obtained by restricting the sum
in~\eqref{eq:erm_obj} to $j\notin I_f$.
For both learning and unlearning, we consider (Gaussian-perturbed) full gradient updates
that we conventionally write as discrete Langevin updates.

\paragraph{Learning algorithm $\A$.}
Given a step size $\eta$ and an initial noise level $\slearn\ge 0$, the discrete Langevin
update is
\begin{equation}
\theta_{k+1}
= \theta_k - \eta \nabla_\theta f_{\D}(\theta_k)
   + \sqrt{2\eta\slearn^2}\,\xi_k,
\label{eq:langevin-learn}
\end{equation}
with $\xi_k \sim \mathcal N(0,I)$, run for $T$ steps, so that $\theta_T=\A(D)$ is a random variable. If $\slearn=0$, 
$\theta_T$ is a deterministic function of the data.

\paragraph{Unlearning $\U$ via retain set fine-tuning.}
Given an unlearning request for $D_f$, the unlearning procedure $\U$ runs $K$
additional Langevin steps on $\D\setminus\D_f$ with a calibrated noise level $\sunl$:
\begin{equation}
\theta_{k+1}
= \theta_k - \eta \nabla_\theta f_{\D\setminus\D_f}(\theta_k)
   + \sqrt{2\eta\sunl^2}\,\xi_k,
\label{eq:langevin-unlearn}
\end{equation}
with $\xi_k\sim\mathcal N(0,I)$, and releases the random variable $\theta_{T+K}= \U(\A(\D),\,\D\setminus\D_f,\,D_f)$.

It is important to keep in mind that unlearning according to the above scheme can only be
useful if $K \ll T$, meaning that
unlearning is indeed much less computationally demanding than full retraining.
In the main paper, unlearning uses gradient steps corresponding to the full leave-one-out
objective. Appendix~\ref{app:mini-batch} extends the construction and its accounting to
cyclic mini-batch training and unlearning. Note that mini-batch randomization, as considered
by~\citet{sepahvand2025leveragingperinstanceprivacymachine}, does create another sort of
difficulty for the analysis and, to the best of our knowledge, prevents the derivation of 
quantitative per-instance unlearning certificates.

\subsection{Defining Per-Instance Unlearning}
\label{sec:definition}

We first recall the standard notion of differential privacy \citep{dwork2006tcc}. An algorithm $\A$ is differentially private if
\[
\forall \D,\ \forall z\in\D,\ \forall S,\quad
\Prob(\A(\D)\in S)\le e^\eps \Prob(\A(\D\setminus \{z\})\in S)+\del.
\]

Certified unlearning adapts this definition by requiring the output of the unlearning
procedure to be statistically indistinguishable from an output that was not trained on the
forgotten data, a formalization first due to \citet{ginart2019}: 
an algorithm $\U$ is a certified unlearning algorithm for the learning algorithm $\A$ if
\[
\forall D,\ \forall D_f \subset D,\ \forall S,\quad
\Prob(\U(\A(D),D\setminus D_f,D_f)\in S)
\le e^\eps \Prob(\A(D\setminus D_f)\in S)+\del.
\]

The guarantee is stated relative to a reference algorithm $\A$, typically
retraining from scratch on the retain set. In particular, obtaining a nontrivial 
distributional comparison generally requires randomness in the learning algorithm $\A$, 
which makes the definition inconvenient when the underlying learning procedure is 
deterministic \citep{VanWaerebeke2025,Georgiev2024}.
Instead, \citet{sekhari2021} compare $\U$ to itself on adjacent unlearning scenarios:
an algorithm $\U$ is a certified unlearning algorithm for the learning algorithm $\A$ if
\[
\forall D,\ \forall D_f \subset D,\ \forall S,\quad
\Prob(\U(\A(D),D\setminus D_f,D_f)\in S)
\le e^\eps \Prob(\U(\A(D\setminus D_f),D\setminus D_f,\varnothing)\in S)+\del.
\]

In both cases, however, the guarantee is uniform over all datasets and forget sets. We instead 
consider the guarantee associated with one fixed deletion request $(D,D_f)$. Per-instance 
differential privacy (pDP) \citep{wang2018perinstancedifferentialprivacy} provides the right framework;
an algorithm $\A$ is $(\eps,\del)$-pDP for $(D,z)$ if the DP guarantee holds for that specific 
dataset and that specific point.
Transposing it to unlearning, we introduce \emph{per-instance unlearning},
where both the guarantee and the calibration of the unlearning procedure may depend on the
retain and forget sets.

\begin{definition}[$(\eps,\del)$ Per-instance unlearning]
\label{def:per_instance_unlearning}
Let $\A$ be a learning algorithm, $D$ a dataset and $D_f\subset D$ a forget set. 
Let
$\U_f := \U(D, D_f)$ denote the unlearning procedure calibrated for forgetting $D_f$ from $D$.
We say that the algorithm $\U_f$ satisfies $(\eps,\del)$-per-instance unlearning for
the request $(D,D_f)$ and learning algorithm $\A$ if, for any measurable set $S$,
\begin{equation}
\Prob\bigl(\U_f(\A(D),D \setminus D_f, D_f)\in S\bigr)
\le
e^\eps
\Prob\bigl(\U_f(\A(D \setminus D_f),D \setminus D_f, \varnothing)\in S\bigr)
+\del.
\end{equation}
\end{definition}

Once the request is fixed, the same calibrated procedure $\U_f$ appears on both sides of the 
comparison. The certificate is therefore conditional on $(D,D_f)$: it guarantees 
$(\eps,\del)$-closeness for this particular deletion request, rather than uniformly over all 
possible requests. Because the unlearning guarantee pertains to a complete learning-then-unlearning cycle,
it does not directly cover comparisons between different forget sets 
(forgetting $z_i$ on one side and $z_j$ on the other) or 
sequences of unlearning requests. It is, however, possible to adapt our results when the
forget set $D_f$ contains more than one point, as detailed in Appendix~\ref{app:group-unlearning}. 
Since the calibrated level $\sunl$ depends on the
request and is common to both sides, the only information about $D_f$ outside the guarantee is
$\sunl$ itself. We ignore this fact in the main text and refer to Appendix~\ref{app:sigma-leak}
for a precise evaluation of the privacy cost of this disclosure.

Let $\sunl(D,D_f)$ denote the smallest noise level 
sufficient to certify $(D,D_f)$ at the target privacy level $(\eps,\del)$. Per-instance unlearning may use 
this level directly, whereas uniform calibration must use at least
\[
\sup_{(D,D_f)}\sunl(D,D_f).
\]
Thus, per-instance calibration can inject strictly less noise on non-worst-case requests 
while preserving the same $(\eps,\del)$ guarantee.
The perspective adopted in our work differs from that of \citet{wang2018perinstancedifferentialprivacy},
where pDP is mainly descriptive: the mechanism is fixed and one measures a 
request-dependent privacy loss.
In our case, it is instead prescriptive: we fix the desired 
$(\eps,\del)$ guarantee for a request and use it to calibrate the unlearning mechanism. 
Conversely, once the mechanism is fixed, our analysis again induces a request-dependent 
profile $\eps(D,D_f)$.

\section{Per-Instance Privacy Accounting}
\label{sec:accounting}

We now instantiate Definition~\ref{def:per_instance_unlearning}. For simplicity we present
the deletion of a single data point $(x_i,y_i)$ and write $\Di:=\D\setminus\{(x_i,y_i)\}$,
so that the retain objective is the leave-one-out objective
\begin{equation}
\label{eq:erm_obj_minus_i}
f_{\Di}(\theta)
\;:=\;
\sum_{j\neq i} \ell(\theta; x_j,y_j).
\end{equation}
We write $\U_i:=\U(D,\{(x_i,y_i)\})$ for the corresponding unlearning procedure; 
Appendix~\ref{app:group-unlearning} extends every result below to group
deletion requests.

Our analysis requires the deterministic gradient map underlying the Langevin dynamics to be
contractive, which the following standing assumption guarantees.

\begin{assumption}
  \label{ass:contraction}
   The maps $x\mapsto x-\eta\nabla f_{\D}(x)$ and
 $x\mapsto x-\eta\nabla f_{\Di}(x)$ are $c$-contractive with $c\le 1$.
\end{assumption}

If $f_{\D}$ and $f_{\Di}$ are convex and $L$-smooth, then
Assumption~\ref{ass:contraction} holds with $c=1$ for any
$0<\eta<2/L$. If they are additionally $m$-strongly convex, then it
holds with $c=\max\{|1-\eta m|,\,|1-\eta L|\}<1$
\citep{nesterov2004introductory}.
We fix $\eta=1/L$ throughout, yielding $c=1$ in the convex case and
$c=1-\eta m$ in the strongly convex case.

\subsection{Per-Instance Sensitivities}
\label{sec:gdp-accounting}
 
The privacy gap is measured between two parameter trajectories
$(\theta_k)_{k\ge0}$ and $(\theta'_k)_{k\ge0}$ (see Figure~\ref{fig:learning-unlearning})
initialized at the same deterministic point $\theta_0=\theta'_0$.
The (actual) trajectory $(\theta_k)$ follows $\A$ on the full dataset $\D$ for $T$ iterations
with noise level $\slearn$, then $K$ unlearning steps via $\U_i$ on the retain set $\Di$
with noise level $\sunl$. The second trajectory $(\theta'_k)$ is the (hypothetical)
retraining path, obtained by running the same learning and unlearning procedures directly
on $\Di$ (this is exactly the comparison in Definition~\ref{def:per_instance_unlearning}). 
This privacy gap is driven by how much the deleted point moves the update. For an unlearning
request concerning $(x_i,y_i)$, we define 
\begin{equation}
\label{eq:sensitivity}
\Delta_{i, k} := \bigl\|
\eta\bigl(\nabla f_{\Di}(\theta_k)-\nabla f_{\D}(\theta_k)\bigr) \bigr\|
= \eta \bigl\| \nabla_\theta \ell(\theta_k; x_i, y_i)\bigr\|
\end{equation}
as the \emph{per-instance sensitivity} at time $k$; it vanishes for $k\ge T$, where the two
compared trajectories follow the same update map.

\begin{figure}[!htbp]
\centering
\begin{tikzpicture}[>=Stealth, x=0.8cm, y=0.7cm, line cap=round, line join=round, thick]
 
  \draw[->] (0,0) -- (8.2,0) node[below] {$k$};
  \foreach \x/\lab in {0/{0}, 4.0/{$T$}, 7.0/{$T+K$}}{
    \draw (\x,0.08)--(\x,-0.08) node[below=3pt] {\lab};
  }
  \draw[densely dashed] (4.0,-0.2)--(4.0,3.8);
  \draw[densely dashed] (7.0,-0.2)--(7.0,3.8);
 
  \coordinate (start) at (0,1.6);
  \node[anchor=north east] at ($(start)+(-0.05,-0.02)$)
     {$\theta_0=\theta'_0$};
 
  \path (start)            coordinate (u0)
        (1.0,2.15)         coordinate (u1)
        (2.0,2.85)         coordinate (u2)
        (3.0,2.55)         coordinate (u3)
        (4.0,3.20)         coordinate (u4)
        (5.0,2.45)         coordinate (u5)
        (6.0,2.35)         coordinate (u6)
        (7.0,2.10)         coordinate (u7);
  \draw[black] (u0)--(u1)--(u2)--(u3)--(u4)--(u5)--(u6)--(u7);
  \node[anchor=south east] at ($(u4)+(-0.02,0.02)$) {$\theta_T$};
  \node[anchor=south west] at ($(u7)+(0.06,0.02)$) {$\theta_{T+K}$};
 
  \path (start)            coordinate (l0)
        (1.0,1.40)         coordinate (l1)
        (2.0,1.10)         coordinate (l2)
        (3.0,1.05)         coordinate (l3)
        (4.0,0.70)         coordinate (l4)
        (5.0,0.90)         coordinate (l5)
        (6.0,1.25)         coordinate (l6)
        (7.0,1.50)         coordinate (l7);
  \draw[black, dashed] (l0)--(l1)--(l2)--(l3)--(l4)--(l5)--(l6)--(l7);
  \node[anchor=north east] at ($(l4)+(-0.06,0.08)$) {$\theta'_T$};
  \node[anchor=north west] at ($(l7)+(-0.06,-0.02)$) {$\theta'_{T+K}$};
 
  \draw[<->] (2.2,2.7) -- (2.2,1.15);
  \node[right] at (2.15,1.6) {$\Delta_{i,k}>0$};
 
  \draw[<->] (4.7,2.6) -- (4.7,0.9);
  \node at (5.6,1.7) {$\Delta_{i,k}=0$};
 
\end{tikzpicture}
 
\caption{
Learn--unlearn vs.\ retraining trajectories. The solid path $(\theta_k)$ corresponds to 
training on the full dataset $\D$ followed by unlearning, while the dashed path $(\theta'_k)$ 
corresponds to (fictitiously) retraining on the retain dataset $\Di$.
}
\label{fig:learning-unlearning}
\end{figure}
 
To turn this sequence into a certificate, we rely on the shifted interpolation framework of
\citet{bok2024shiftedinterpolationdifferentialprivacy}, which tracks privacy loss along the
optimization trajectory in terms of $\mu$-Gaussian differential privacy (GDP)
\citep{dong2019gdp}. Two features motivate this choice. First, the conversion to
$(\eps,\del)$ guarantees is exact, using the privacy curve of the Gaussian mechanism, which
improves on earlier analyses of Langevin dynamics based on R\'enyi differential privacy
\citep{chien2025langevinunlearningnewperspective,chourasia2021,ryffel2022}. Second, the
framework exploits \emph{privacy amplification by iteration}
\citep{Feldman_2018,altschuler23a}, whereby the repeated application of noisy contractive
updates yields strictly stronger guarantees than Gaussian composition. We write
$\epsgdp(\mu,\del)$ for the value of $\eps$ such that $\mu$-GDP implies
$(\eps,\del)$-differential privacy; formal definitions and the conversion are recalled in
Appendix~\ref{app:gdp}.
 
The framework takes as input a deterministic upper bound on the sensitivity at each iteration, 
namely a sequence $\{s_{i,k}\}_{k=0}^{T-1}$ with $s_{i,k}\ge 0$, independent of the realized 
trajectory $\{\theta_k\}$, such that $\Delta_{i,k}\le s_{i,k}$ for all $k<T$. All dependence 
on the deleted point is therefore captured by the sequence $\{s_{i,k}\}$, whose construction 
is the main technical challenge; Theorems~\ref{thm:deterministic} and~\ref{thm:ridge-gdp-main} then use these bounds to 
calibrate $\sunl$ and obtain $(\eps,\del)$ per-instance unlearning guarantees. 
When the learning dynamics are deterministic ($\slearn=0$), these 
sensitivities can be computed exactly during training, yielding an explicit certificate 
for any smooth convex objective (Section~\ref{sec:deterministic}). When training 
is itself randomized ($\slearn>0$), the sensitivities become random and require probabilistic 
control; for ridge regression, we derive high-probability upper bounds in Section~\ref{sec:ridge}.

To simplify notation, we write 
\begin{equation}
\label{eq:prefix}
\begin{aligned}
V(\sigma)&:=\sum_{k=0}^{T-1}2\eta\slearn^2\,c^{\,2(T+K-1-k)}
+ \sum_{k=T}^{T+K-1}2\eta\sigma^2\,c^{\,2(T+K-1-k)},\\
N_{i,j}&:=\!\!\sum_{k<\min(j,T)}\!\! c^{\,T+K-1-k}\,s_{i,k}, \quad j=1,\dots,T+K,
\qquad N_i:=N_{i,T+K}.
\end{aligned}
\end{equation}

\subsection{Deterministic Training}
\label{sec:deterministic}
 
When $\slearn=0$, the sensitivities
$\Delta_{i,k}$ of~\eqref{eq:sensitivity} only depend
on the data and the public initialization: $s_{i,k}=\Delta_{i,k}$ is admissible.
In this case, a model for which no deletion is ever requested pays
nothing: the utility cost of unlearning is incurred only at deletion time.
 
\begin{theorem}[Certified per-instance unlearning with deterministic training]
\label{thm:deterministic}
Let Assumption~\ref{ass:contraction} hold with $c\le1$. Run the learning phase
deterministically, with $\slearn=0$, so that the iterates $\theta_0,\dots,\theta_T$ generated
by~\eqref{eq:langevin-learn} on $\D$ are deterministic and the sensitivities
$\Delta_{i,k}$ of~\eqref{eq:sensitivity} are computed exactly.
Let $(x_i,y_i)$ be the point to unlearn,
$(\eps,\del)$ the target per-instance unlearning guarantee, and $K\ge 1$ the unlearning
horizon.
Let $N_i$ be the sum~\eqref{eq:prefix} formed
from $s_{i,k}=\Delta_{i,k}$, and let
$\mu_{\mathrm{target}}:=\sup\{\mu>0:\epsgdp(\mu,\del)\le\eps\}$.
Then the unlearning mechanism $\U_i$, run with noise level
\begin{equation}
\label{eq:mu-deterministic}
\sunl
=
\frac{N_i}
{\mu_{\mathrm{target}}\sqrt{2\eta\sum_{k=0}^{K-1} c^{2k}}},
\end{equation}
satisfies $(\eps,\del)$ per-instance unlearning.
\end{theorem}
 
Only the scalars $N_i=c^K\tilde N_i$, where $\tilde N_i$ is defined as $\tilde N_i:=\sum_{k<T}c^{T-1-k}\Delta_{i,k}$,
must be stored during training. The per-sample gradient norms are available on the fly during
full-batch training, so the $n$ scalars $\tilde N_i$ are maintained through the recursion
$\tilde N_i\leftarrow c\,\tilde N_i+\Delta_{i,k}$ in $O(Tn)$
operations and $O(n)$ memory. 
A deletion request then requires $K$ additional gradient steps, costing $O(Knpd)$. 
The same certificate holds for a group $I_f$ of $R$ points, with the norm of the summed
gradients of the group in place of $\Delta_{i,k}$ (Appendix~\ref{app:group-unlearning}). A
group known at training time needs one more scalar, at $O(Rpd)$ operations per step;
otherwise nothing more is stored, and $\sum_{i\in I_f}N_i\ge N_{I_f}$ still gives a valid,
though larger, noise level.

\subsection{Noisy Training: Ridge Regression}
\label{sec:ridge}
 
This section considers noisy training, $\slearn>0$, which may serve as a form of regularization 
or be calibrated to provide learning-time differential privacy under suitable sensitivity control; 
see Appendix~\ref{app:ridge-dp-guarantee}.
Unlike the deterministic setting, noisy training makes the sensitivities $\Delta_{i,k}$ themselves 
random, whereas the accounting of Section~\ref{sec:gdp-accounting} (like
the analyses based on privacy amplification by iteration
\citep{Feldman_2018,chourasia2021,bok2024shiftedinterpolationdifferentialprivacy})
requires bounding them deterministically. Uniform calibration needs a single such bound, valid
for every sample, which a Lipschitz loss or clipping of the per-sample gradients readily
supplies. Per-instance calibration needs one bound per point, which clipping cannot give: it
caps every sample at the same ceiling, erasing the very heterogeneity we aim to measure.
 
We therefore restrict our attention to ridge regression, whose linear-Gaussian iterates admit
sharp \emph{high-probability} bounds on each $\Delta_{i,k}$, and hence a per-instance
calibration of the unlearning noise (Theorem~\ref{thm:ridge-gdp-main}). In the multi-output
case, with inputs $X\in\R^{n\times p}$, outputs $Y\in\R^{n\times d}$, parameter
$\theta\in\R^{p\times d}$, regularizer $\lambda\ge0$, and deterministic initialization $\theta_0$, the ERM objective and
the sensitivity read
\begin{gather*}
f_{\D}(\theta)
=
\sum_{j=1}^n \tfrac12\|x_j^\top\theta-y_j\|_2^2
+\tfrac{\lambda}{2}\|\theta\|_F^2,\\
\Delta_{i, k} = \eta \|x_i\|_2 \|x_i^\top\theta_k - y_i\|_2.
\end{gather*}
 
\begin{theorem}[Certified per-instance unlearning for ridge regression]
\label{thm:ridge-gdp-main}
Consider multi-output ridge regression trained with algorithm $\A$, learning noise
$\slearn>0$, for $T$ iterations; then Assumption~\ref{ass:contraction} holds with contraction 
factor $c=1-\eta \lambda \le 1$. Let $(x_i,y_i)$ be the point to unlearn,
$(\eps,\del)$ the target per-instance unlearning guarantee, $K\ge 1$ the unlearning
horizon, and choose $(\dm,\ds)$ with $\dm+\ds=\del$. Set $M:=I_p-\eta(X^\top X+\lambda I_p)$,
$B:=X^\top Y$, and for each $k\ge 1$,
\begin{gather*}
  u_{i,k} := x_i^\top\!\Bigl(M^k\theta_0+\eta\sum_{j=0}^{k-1}M^j B\Bigr)-y_i,\\
  v_{i,k} := 2\eta\slearn^2\sum_{j=0}^{k-1}\|(M^j)^\top x_i\|_2^2.
\end{gather*}
 
\paragraph{(i) Bounding $\Delta_{i,k}$ with high probability.}
At time $k$, the residual $x_i^\top\theta_k - y_i \sim \mathcal{N}(u_{i,k},\, v_{i,k}I_d)$, so
$\|x_i^\top\theta_k - y_i\|_2^2/v_{i,k} \sim \chi'^2_d\bigl(\|u_{i,k}\|_2^2/v_{i,k}\bigr)$, 
a non-central chi-squared distribution with $d$ degrees of freedom and non-centrality parameter $\|u_{i,k}\|_2^2/v_{i,k}$.
Let $q_{i,k}(\cdot)$ be the quantile function of this distribution and set
\[
  s_{i,0}^{\ds}
  := \eta \|x_i\|_2 \|x_i^\top\theta_0 - y_i\|_2
\]
and, for $k\ge 1$,
\[
  s_{i,k}^{\ds}
  := \eta\,\|x_i\|_2\sqrt{v_{i,k}\, q_{i,k}\!\left(1-\tfrac{\ds}{T}\right)}.
\]
Then $\;\Prob\,\bigl(\forall k< T,\;\Delta_{i,k}\le s_{i,k}^{\ds}\bigr)\ge 1-\ds$.
 
\paragraph{(ii) Calibrating the unlearning noise.}
Let $N_{i,j}$ and $N_i$ be the sums~\eqref{eq:prefix} formed from
the sequence $\{s_{i,k}^{\ds}\}_{k<T}$, and let
$\mu_{\mathrm{target}}:=\sup\{\mu>0:\epsgdp(\mu,\dm)\le\eps\}$. Suppose the learning noise
satisfies
\begin{equation}
\label{eq:sufficient-main}
\slearn^{2}\ \le\
\min_{1\le j\le T}\ \frac{N_i\,N_{i,j}}
{2\eta\,\mu_{\mathrm{target}}^{2}\sum_{k<j} c^{2(T+K-1-k)}}.
\end{equation}
Then the unlearning mechanism $\U_i$, run with
noise level
\[
  \sunl
  :=
  \inf\Bigl\{\sigma> 0:\ N_i/\sqrt{V(\sigma)}\ \le\ \mu_{\mathrm{target}}\Bigr\}
\]
satisfies $(\eps,\del)$ per-instance unlearning.
\end{theorem}
If~\eqref{eq:sufficient-main} fails, the same conclusion holds with $N_i/\sqrt{V(\sigma)}$
replaced by the square root of the value of a convex quadratic program (Appendix~\ref{app:qp}).
 
Both $u_{i,k}$ and $v_{i,k}$ are available in closed form, with nothing stored along the
training trajectory. After a one-off $O(np^2+p^3)$ diagonalization of $X^\top X$, shared by
all requests, a request costs $O(Tpd)$ for the residual statistics and $O(T)$ for the
quantiles. The certificate is therefore dominated by
the $K$ unlearning steps themselves: $O(Knpd)$ against $O(Tnpd)$ for retraining from
scratch, a factor $T/K$. The same holds for a group of $R$ points, which need not be known in
advance: its bounds are computed from the data at request time
(Appendix~\ref{app:group-unlearning}).

As a final comment on this part, we note that for ridge regression a closed-form leave-one-out update
exists that performs exact unlearning for models trained until exact convergence. However, using this leave-one-out
formula in realistic contexts where $T$ is not much larger than $p$ (and hence convergence is not fully reached)
results in invalid certificates (see Appendix~\ref{app:exact-ridge} and the Newton-step baseline of
Section~\ref{sec:experiments-linear}).

\section{Experiments}
\label{sec:experiments-linear}

We implement the approach analyzed in Theorem~\ref{thm:ridge-gdp-main} to perform
per-instance Langevin unlearning for ridge regression, which we refer to as PILU;
Algorithms~\ref{alg:learning} and~\ref{alg:unlearning} in Appendix~\ref{app:algo} describe
the learning phase and the processing of a deletion request, respectively.

We consider dense depth regression on NYU-Depth
V2~\citep{silberman2012indoor} with a fixed-feature pipeline based on
DINOv2 patch and class tokens~\citep{oquab2024dinov2}. Each RGB image
is encoded as a $16\times16$ grid of $1536$-dimensional patch embeddings,
yielding $256$ patch-level feature vectors per image. We
keep $150$ training, $25$ validation and $25$ test images, and flatten the
training patches into a dataset of size $n=38\,400$. A ridge regression with an additional bias term is trained on top, with
regularization $\lambda=0.1\,n$ selected on the validation images
(Appendix~\ref{app:nyuv2-implementation}); the effective
optimization dimension is $p=1537$ and the output dimension is $d=1$. 
A fixed learning noise level
$\slearn=0.1$ is injected during training, while
$\sunl$ is calibrated per point to meet a target
$(\eps,\del)$ guarantee with $\del=1/n$.
All RMSE values reported in the figures are expressed in meters.

\begin{figure}[!htbp]
  \centering
  \includegraphics[width=0.75\linewidth]{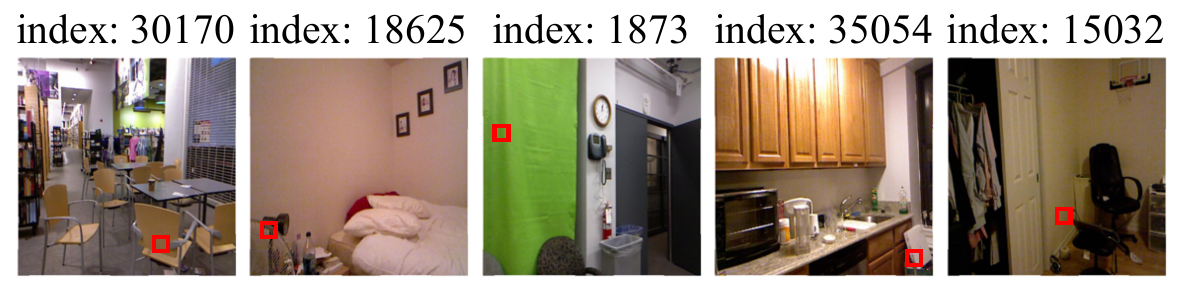}
  \caption{Representative NYU-Depth samples selected for unlearning, ordered from 
   the easiest to the hardest unlearning regime (see Appendix~\ref{app:nyuv2-implementation}).}
  \label{fig:representative-samples}
\end{figure}
\begin{figure}[!t]
    \centering
    \begin{subfigure}[t]{0.47\linewidth}
        \centering
        \includegraphics[width=\linewidth]{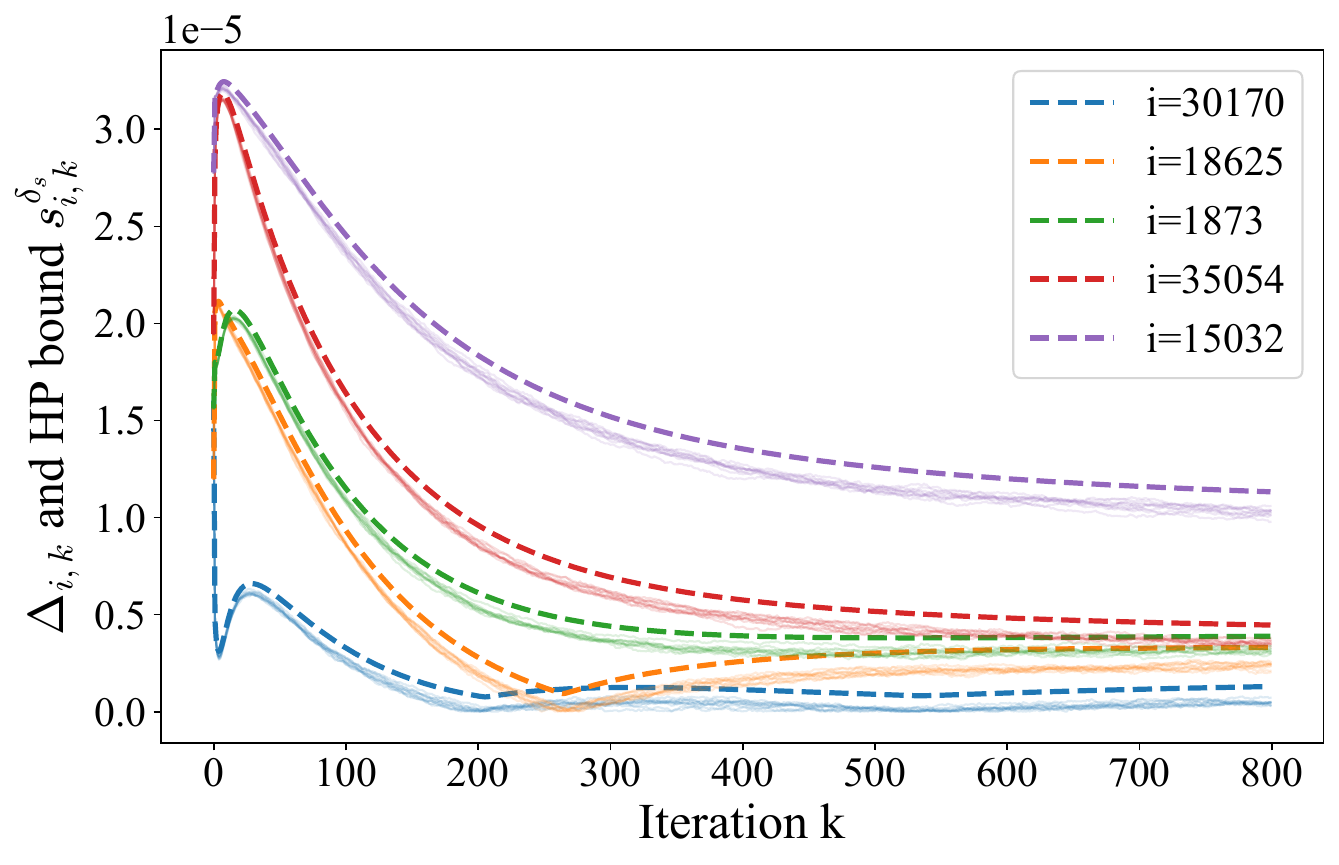}
        \caption{Empirical trajectories $\Delta_{i,k}$ (solid lines) and the corresponding high-probability 
        sensitivity bounds $s_{i,k}^{\delta_s}$ (dashed lines) during training. }
        \label{fig:nyuv2-hp-sensitivity-map}
    \end{subfigure}
    \hfill
    \begin{subfigure}[t]{0.49\linewidth}
        \centering
        \includegraphics[width=\linewidth]{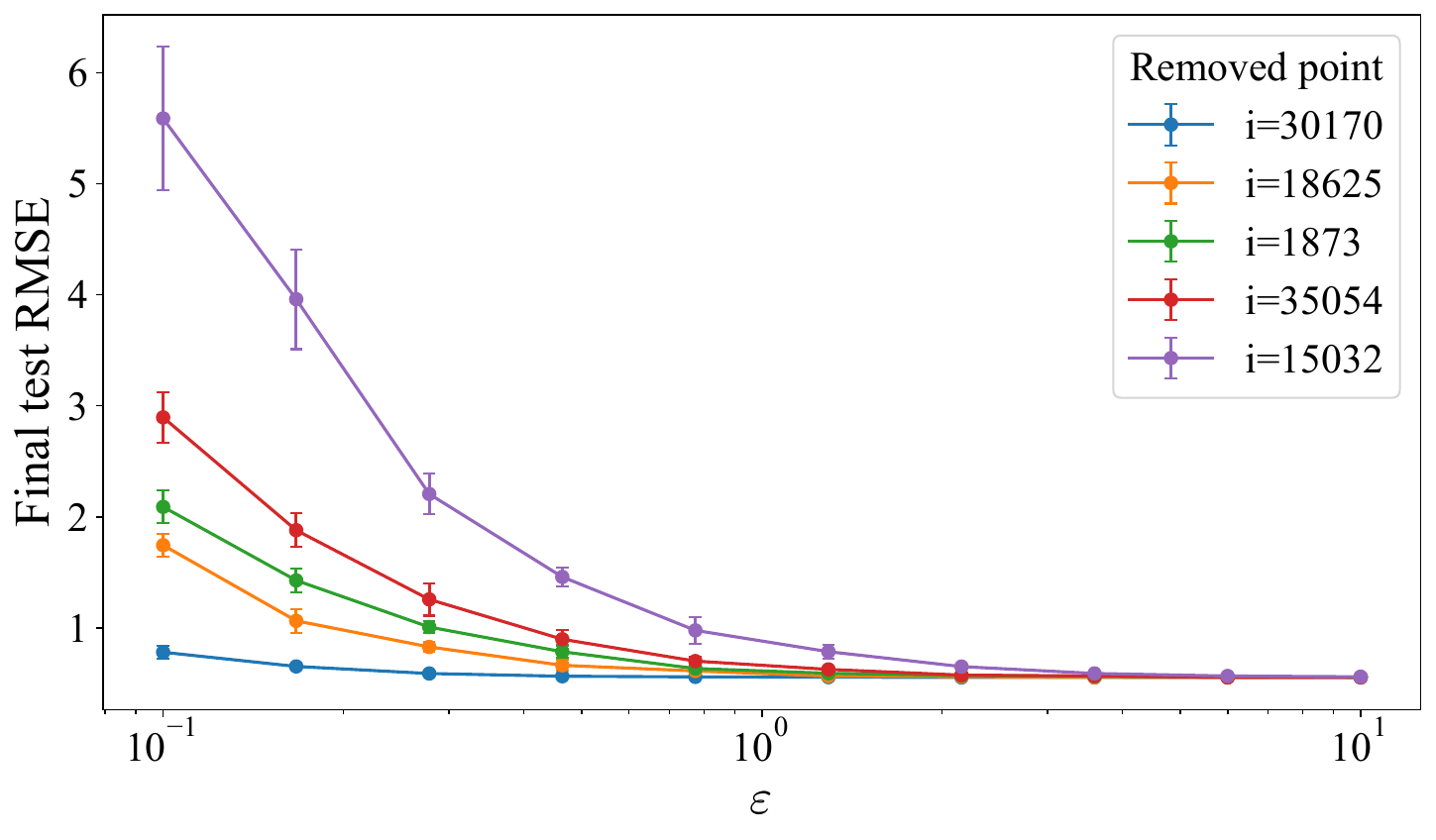}
        \caption{Privacy--utility trade-off of PILU ($T=800, K=80$). 
        Each curve corresponds to a distinct removed training point.}
        \label{fig:nyuv2-privacy-utility}
    \end{subfigure}
    \caption{Per-instance heterogeneity for five representative data points of NYU-Depth.}
    \label{fig:nyuv2-side-by-side}
\end{figure}

Rather than unlearning every training patch, we select five
representative data points (Figure~\ref{fig:representative-samples}) 
spanning distinct sensitivity regimes (see Appendix~\ref{app:nyuv2-implementation}). 
For each selected point, Theorem~\ref{thm:ridge-gdp-main} provides a deterministic
high-probability upper bound $\{s_{i,k}^{\ds}\}_{k< T}$ on the empirical sensitivity
trajectory $\{\Delta_{i,k}\}_{k< T}$. As shown in
Figure~\ref{fig:nyuv2-hp-sensitivity-map}, both these quantities vary substantially across
points and iterations, highlighting that a single uniform sensitivity bound would be overly
conservative.

Figure~\ref{fig:nyuv2-privacy-utility} shows the privacy--utility trade-off of
PILU. As expected, a stronger unlearning guarantee, i.e., smaller $\eps$,
requires more noise and increases the final test error. Crucially, at any fixed
$\eps$, post-unlearning RMSE varies substantially across points,
showing that the request dependence captured by our certificate translates into 
substantial utility differences in practice
(see Figure~\ref{fig:sigma-per-instance} in Appendix~\ref{app:fig-privacy-utility} for
the corresponding values of $\sunl$, varying by a factor of about ten).
The cyclic mini-batch extension of PILU with batch size $b=512$ exhibits the same
per-instance heterogeneity with an unlearning budget of only slightly more than
one pass over the dataset, highlighting the computational advantage of
mini-batching; see Figure~\ref{fig:nyuv2-minibatch-tradeoff-b512} in
Appendix~\ref{app:mini-batch}.

\begin{figure}[!htbp]
    \centering
    \includegraphics[width=0.5\linewidth]{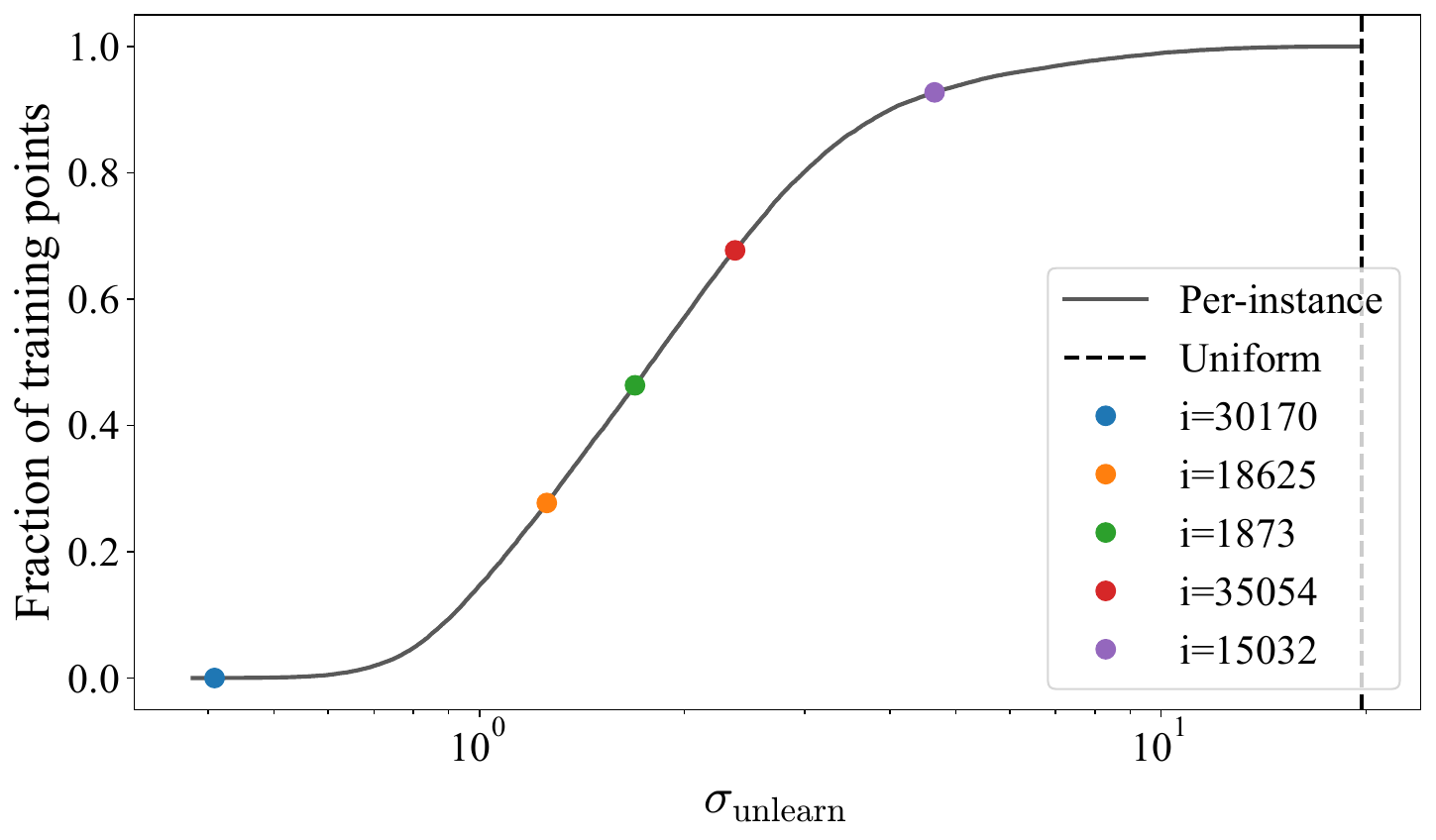}
    \caption{Distribution of the per-instance noise level $\sigma_{\mathrm{unlearn}}$ over all
    training points of NYU-Depth ($\eps=1$, $T=800$, $K=80$). Markers: the five
    representative points of Figure~\ref{fig:nyuv2-privacy-utility}.}
    \label{fig:nyu-sigma-main}
\end{figure}

Figure~\ref{fig:nyu-sigma-main} extends this comparison to all $n$ training points: the
per-instance noise level is smaller, by a median factor of 11, than the uniform level 
certifying every training point, and by a factor of at least $2$ for $98.9\%$ of the points.
\begin{figure}[!t]
  \centering
  \begin{subfigure}[t]{0.49\linewidth}
    \centering
    \includegraphics[width=\linewidth]{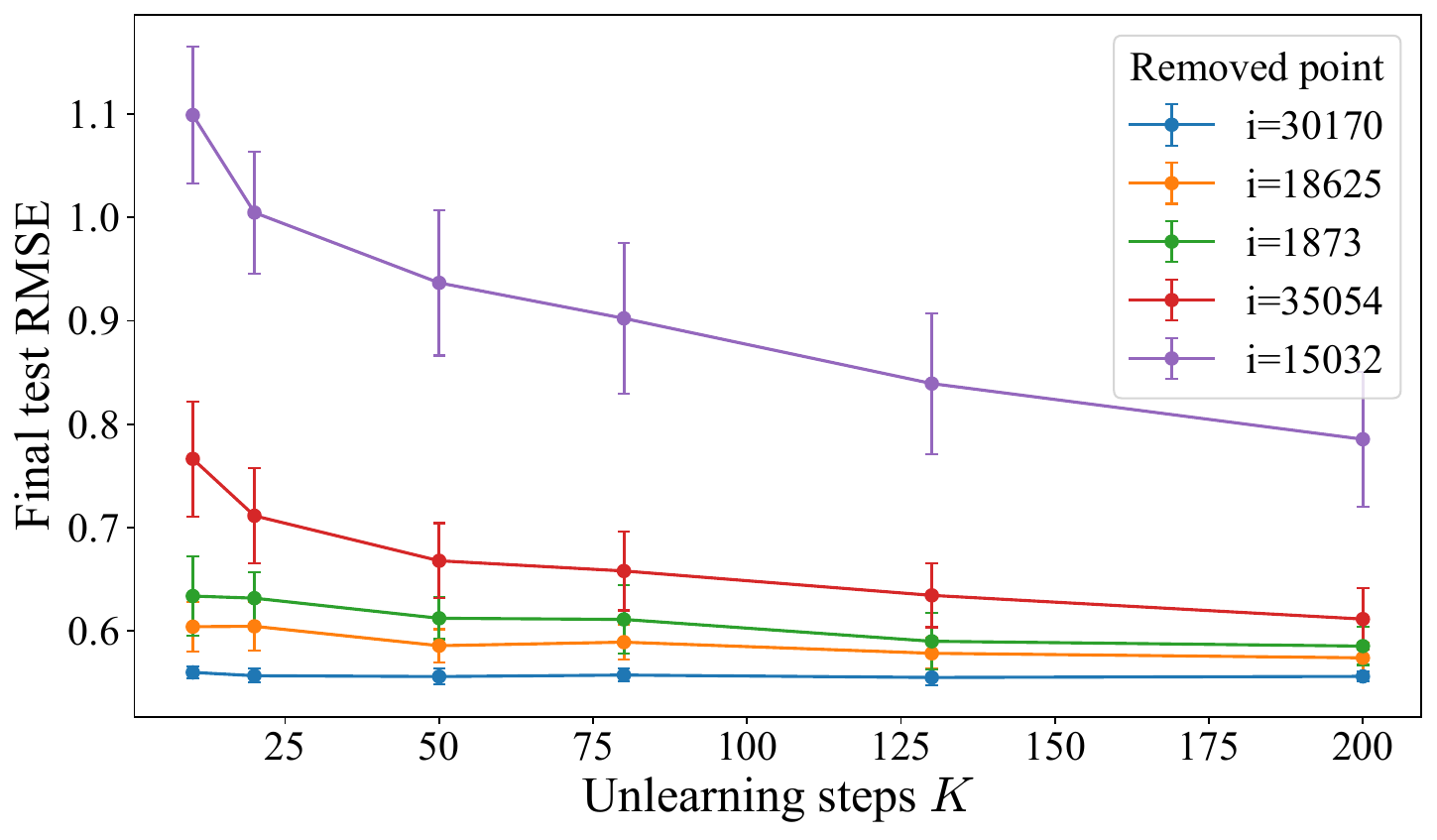}
    \caption{Effect of $K$ at fixed $\varepsilon=1$ and $T=800$.}
    \label{fig:ablation-K}
  \end{subfigure}
  \hfill
  \begin{subfigure}[t]{0.49\linewidth}
    \centering
    \includegraphics[width=\linewidth]{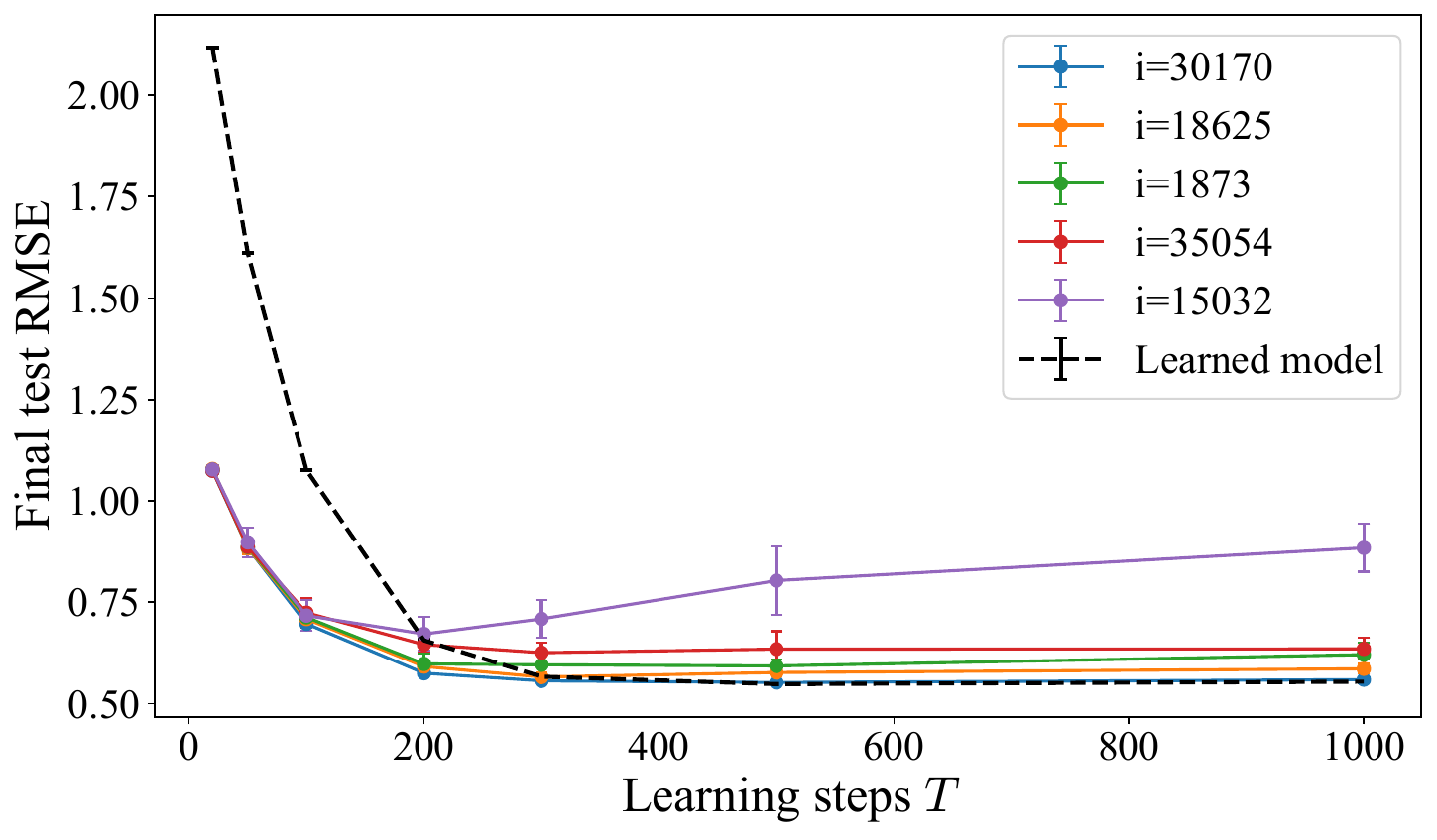}
    \caption{Effect of $T$ at fixed $\varepsilon=1$ and $K=80$.}
    \label{fig:ablation-T}
  \end{subfigure}
  \caption{
  Ablation study of PILU on $K$ and $T$ for five representative data points of NYU-Depth.
  }
  \label{fig:ablation-TK}
\end{figure}

Figure~\ref{fig:ablation-TK} highlights the role of $T$ and $K$.
Increasing $K$ monotonically improves the RMSE but with rapidly
diminishing returns,
suggesting that a small $K$ already suffices for most points, consistent
with the goal of avoiding full retraining.
For the learning horizon $T$, although longer training improves convergence, it
also increases the privacy cost incurred during learning by accumulating
sensitivities along the trajectory, which degrades the final test error.

We then compare PILU against three baselines.

\begin{figure}[!t]
    \centering
    \begin{subfigure}[t]{0.328\linewidth}
        \centering
        \includegraphics[width=\linewidth]{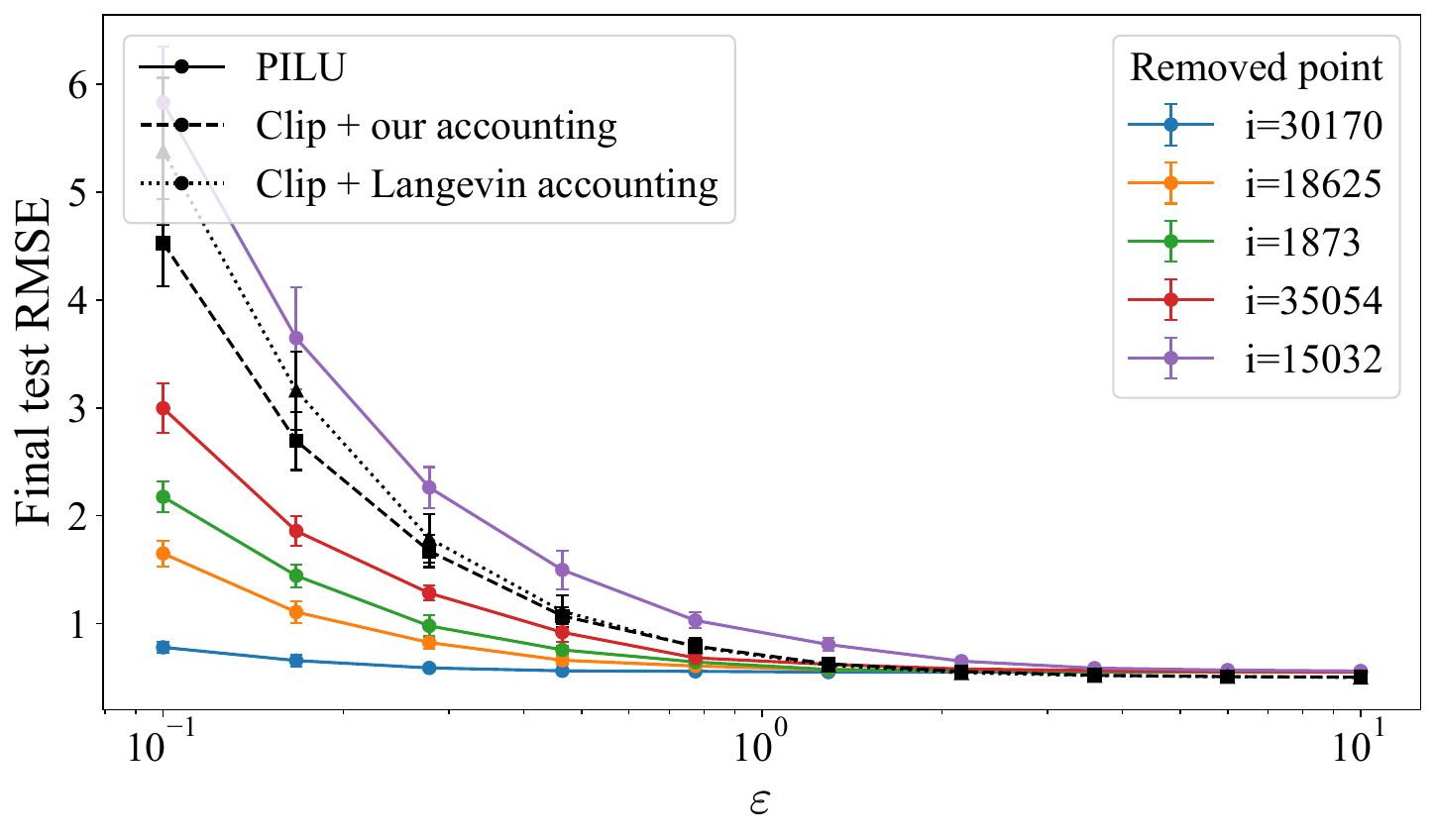}
        \caption{PILU vs.\ Langevin Unlearning.}
        \label{fig:nyuv2-ours-vs-langevin}
    \end{subfigure}
    \hfill
    \begin{subfigure}[t]{0.328\linewidth}
        \centering
        \includegraphics[width=\linewidth]{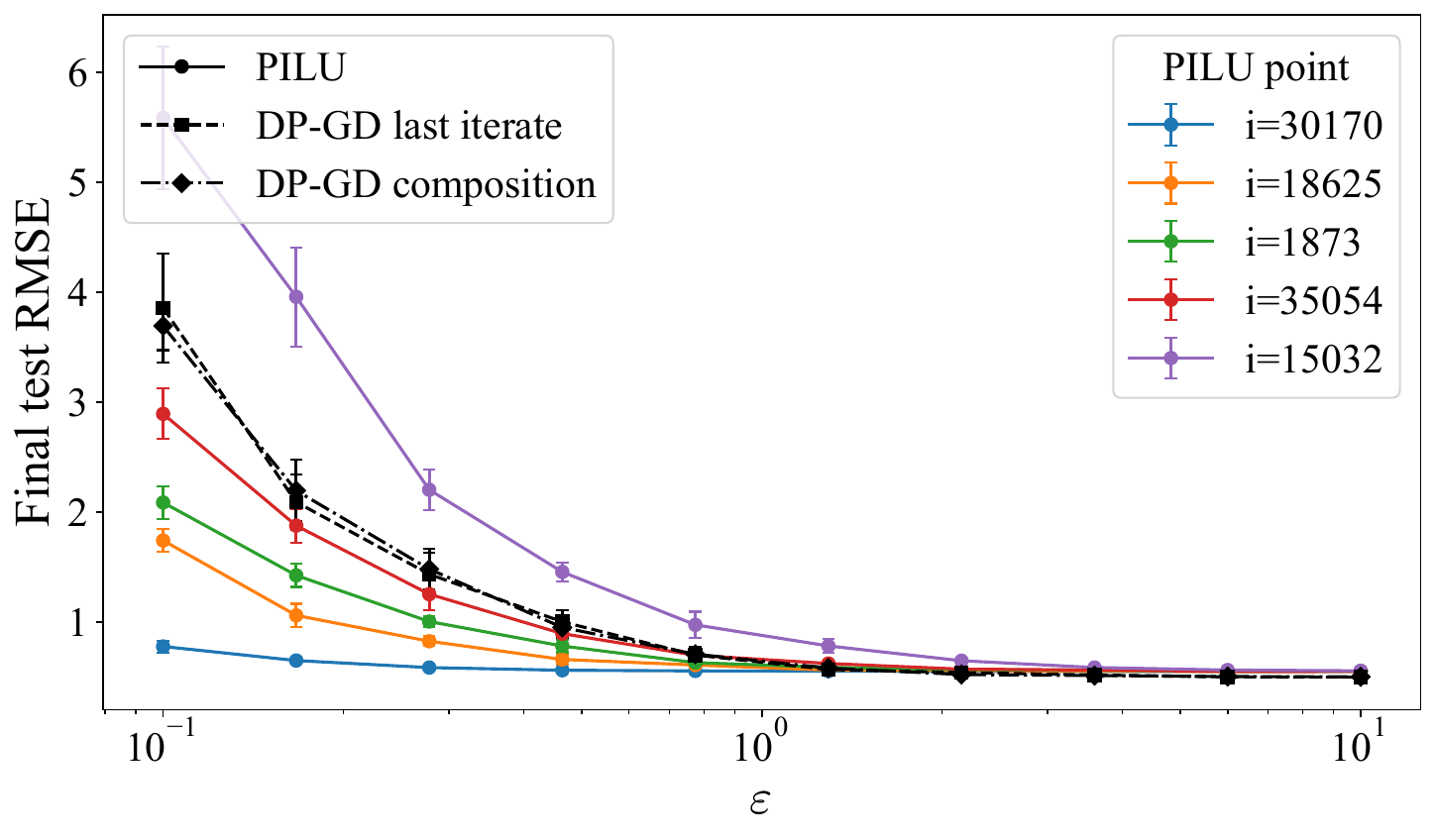}
        \caption{PILU vs.\ DP-GD.}
        \label{fig:nyuv2-ours-vs-dpgd}
    \end{subfigure}
    \hfill
    \begin{subfigure}[t]{0.328\linewidth}
        \centering
        \includegraphics[width=\linewidth]{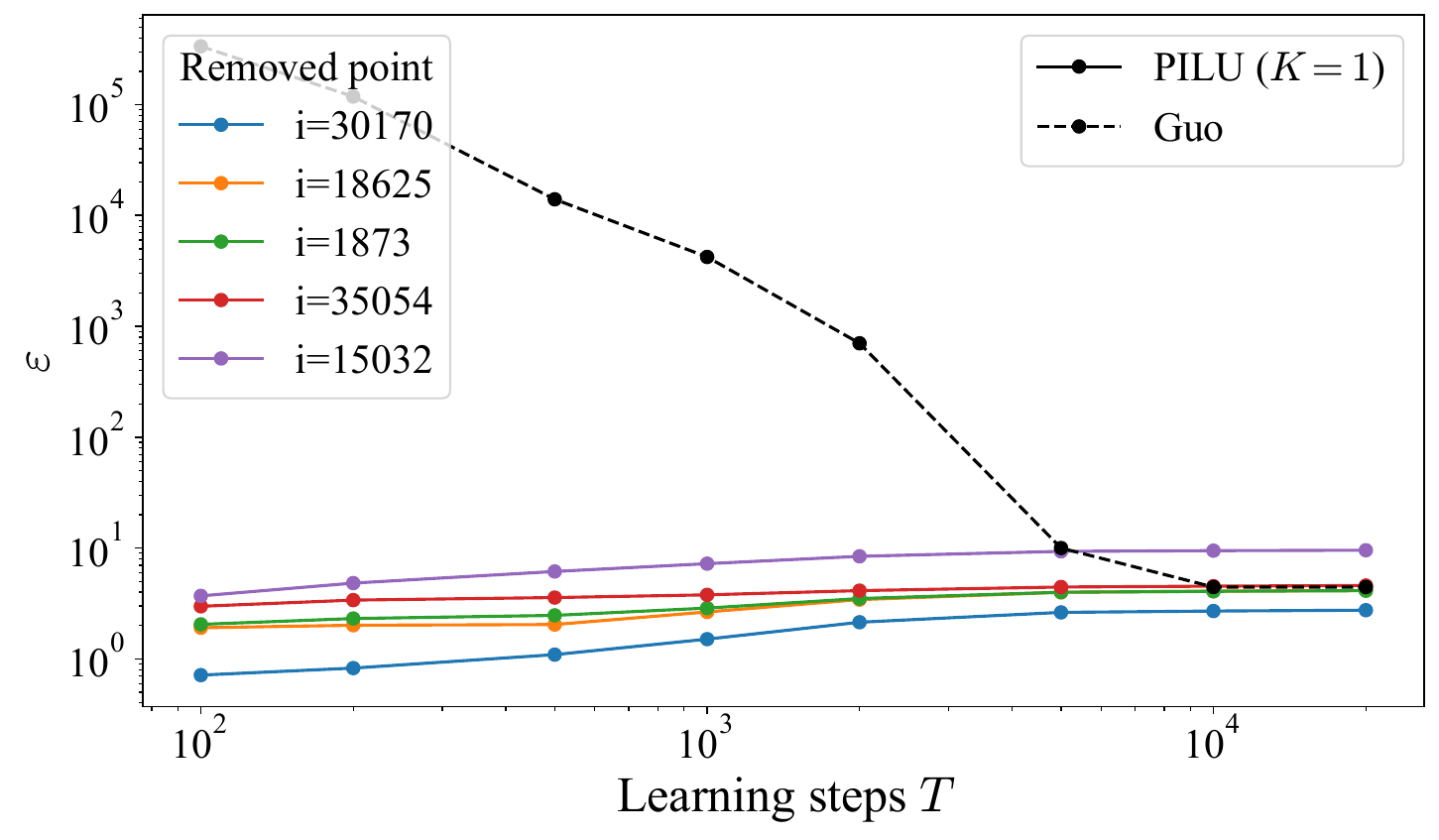}
        \caption{PILU vs.\ Newton step.}
        \label{fig:nyuv2-ours-vs-guo}
    \end{subfigure}

    \caption{Comparison of PILU with three baseline methods on NYU-Depth.
    (a, b): privacy--utility trade-off ($T=800, K=80$).
    (c): unlearning guarantee $\varepsilon$ as a function of the training horizon $T$ ($K=1$).
    }
    \label{fig:nyuv2-tradeoff-baselines}
\end{figure}

\emph{Langevin unlearning \citep{chien2025langevinunlearningnewperspective}.}
This method runs noisy (Langevin) gradient descent for both learning and unlearning,
with a single noise level fixed in advance. It requires bounded per-sample gradients,
which we enforce by clipping (clipping preserves convexity for GLMs \citep{song2021evading}).
First, as shown in Figure~\ref{fig:nyuv2-ours-vs-langevin}, under the same \emph{uniform}
sensitivity bound (the clipping constant), our accounting (dashed) improves on that of
uniform Langevin unlearning (dotted) for small $\eps$ (test RMSE $4.5$ vs.\ $5.4$ at $\eps=0.1$),
and the two coincide within error bars for $\eps\gtrsim0.5$. The gain has two sources:
the GDP conversion is exact whereas the R\'enyi conversion used by \citet{chien2025langevinunlearningnewperspective}
is lossy, and using different learning and unlearning noise values is more efficient.
Second, per-instance sensitivities (solid) substantially improve utility on most
data points (RMSE between $0.8$ and $3.0$ at $\eps=0.1$). 

\emph{Learning-time privacy (DP-GD) \citep{Abadi_2016,bok2024shiftedinterpolationdifferentialprivacy}.}
This method enforces privacy during training through clipped noisy gradient
descent, so that the unlearning guarantee is paid for before any deletion request.
We account for it either by composition
\citep{Abadi_2016,dong2019gdp} or on the last iterate with the privacy amplification by iteration bound (PABI)
\citep{bok2024shiftedinterpolationdifferentialprivacy}, both with the same GDP
conversion as PILU. The two accountants nearly coincide here: with $\eta\lambda T\approx0.8$, the
PABI bound is only about $2.4\%$ below the composition bound. As shown in
Figure~\ref{fig:nyuv2-ours-vs-dpgd}, DP-GD yields a single global trade-off, whereas
PILU adapts to the removed point: for $\eps\le0.5$, it achieves substantially higher
utility on most data points at the same target $\eps$ (RMSE between $0.8$ and $3.0$
against $3.7$--$3.9$ at $\eps=0.1$).

\emph{Newton step \citep{guo2020certified}.}
This method adds a random linear term of scale $\sigma_{\mathrm{Guo}}$ to the training
objective and removes a point with a single Newton step, with a certificate proportional
to the gradient residual left by this step. For ridge regression, this step is the
closed-form leave-one-out update: exact at convergence, it
otherwise leaves the gradient of the training objective at $\theta_T$ as residual
(Appendix~\ref{app:newton-baseline}). Figure~\ref{fig:nyuv2-ours-vs-guo} reports the
guarantee $\eps$ of each method as a function of $T$, with
$\sigma_{\mathrm{Guo}}=\slearn=\sunl=0.5$ and $K=1$ for PILU. The Newton certificate is
vacuous at realistic horizons ($\eps\approx4\times10^{3}$ at $T=1000$) and becomes
competitive only near convergence ($T\ge5000$), whereas PILU certifies $\eps$ between
roughly $1$ and $10$ at every horizon.

\section{Conclusion}
\label{sec:conclusion}

In this work, we develop a point-dependent approach to certified machine unlearning, in which
the noise calibration adapts to the sample being removed. For deterministic training, our
analysis applies to any contractive dynamics,
including smooth convex objectives, and yields non-vacuous guarantees
in regimes where uniform worst-case calibration can severely degrade post-unlearning utility.
For noisy training, although we obtain a fully certified result only for ridge
regression, where high-probability per-instance sensitivity bounds can be computed explicitly, this model
captures a practically relevant regime, as illustrated by our NYU-Depth experiments, which
use frozen features and train only the final linear head.
Two directions remain open: extending the full certification to noisy
training beyond ridge regression requires per-instance probabilistic bounds that we do not
currently have, and beyond the convex case, controlling these sensitivities under non-linear
training dynamics is the main obstacle to principled certified \emph{per-instance}
unlearning in deep models.

\bibliographystyle{apalike}
\bibliography{references1}

\clearpage
\appendix
\section{Gaussian Differential Privacy}
\label{app:gdp}

We recall the definition of Gaussian Differential Privacy (GDP) introduced by
\citet{dong2019gdp}, together with its conversion to $(\varepsilon,\delta)$-differential
privacy.

\begin{definition}[Gaussian Differential Privacy {\citep{dong2019gdp}}]
For two distributions $P$ and $Q$, define the \emph{trade-off function}
$\mathcal{T}(P,Q):[0,1]\to[0,1]$ by
\[
\mathcal{T}(P,Q)(\alpha)
\;:=\;
\inf_{\phi:\,\Prob_P(\phi=1)\le \alpha}\Prob_Q(\phi=0),
\]
where the infimum ranges over all (possibly randomized) tests $\phi$.
An algorithm $\A$ satisfies $\mu$-Gaussian Differential Privacy ($\mu$-GDP) if, for any
adjacent datasets $\D,\D'$,
\[
\mathcal{T}\!\big(\text{Law}(\A(\D)),\, \text{Law}(\A(\D'))\big)
\;\ge\;
G(\mu),
\]
where $G(\mu)$ is the trade-off function between $\mathcal N(0,1)$ and
$\mathcal N(\mu,1)$.
\end{definition}

A $\mu$-GDP guarantee implies $(\varepsilon,\delta)$-differential privacy for any
$\varepsilon>0$, with
\begin{equation}
  \label{eq:mu-eps-relation}
\delta
=
\Phi\!\left(-\frac{\varepsilon}{\mu}+\frac{\mu}{2}\right)
-
e^{\varepsilon}\,
\Phi\!\left(-\frac{\varepsilon}{\mu}-\frac{\mu}{2}\right),
\end{equation}

where $\Phi$ denotes the standard normal cumulative distribution function.
Accordingly, for any $\delta\in(0,1)$, one can obtain an exact conversion from
$\mu$-GDP to $(\varepsilon,\delta)$-DP by defining
$\varepsilon_{\mathrm{GDP}}(\mu,\delta)$ as the unique value of $\varepsilon$
that satisfies~\eqref{eq:mu-eps-relation}.

\section{Shifted-Interpolation Accounting}
\label{app:accounting}

Theorems~\ref{thm:deterministic} and~\ref{thm:ridge-gdp-main} both follow from a single
accounting argument, based on the shifted interpolated processes of
\citet{bok2024shiftedinterpolationdifferentialprivacy}. We develop it once, for an arbitrary
deterministic sensitivity sequence and a learning noise $\slearn\ge0$. We then prove
Theorem~\ref{thm:deterministic} ($\slearn=0$) in Appendix~\ref{app:proof-deterministic} and
Theorem~\ref{thm:ridge-gdp-main} ($\slearn>0$) in Appendix~\ref{app:proof-thm-ridge}.

\subsection{Setting and Notation}
\label{app:setting}

Fix the point $(x_i,y_i)$ to be deleted and write
\[
\phi_{\D}(\theta)=\theta-\eta\nabla f_{\D}(\theta),
\qquad
\phi'(\theta)=\theta-\eta\nabla f_{\Di}(\theta).
\]
By Assumption~\ref{ass:contraction}, both maps are $c$-contractive. The learn-then-unlearn
procedure applies $\phi_{\D}$ during the $T$ learning steps and $\phi'$ during the $K$
unlearning steps, which we encode through
\[
\phi_k=
\begin{cases}
\phi_{\D}, & k<T,\\
\phi', & k\ge T.
\end{cases}
\]
As in Section~\ref{sec:gdp-accounting} (Figure~\ref{fig:learning-unlearning}), we compare
this procedure with the retraining path on $\Di$, run with the same noise schedule and the
same deterministic initialization $\theta_0=\theta'_0$:
\begin{equation}
\label{eq:trajectories}
\theta_{k+1}=\phi_k(\theta_k)+Z_{k+1},
\qquad
\theta'_{k+1}=\phi'(\theta'_k)+Z'_{k+1},
\qquad k=0,\dots,T+K-1,
\end{equation}
where $Z_{k+1},Z'_{k+1}\sim\mathcal N(0,2\eta\sigma_k^2I)$ are independent of each other and
of the past, with $\sigma_k=\slearn$ for $k<T$ and $\sigma_k=\sunl$ for $k\ge T$. When
$\slearn=0$, the learning phase is deterministic.

In this notation, the per-instance sensitivity~\eqref{eq:sensitivity} reads
$\Delta_{i,k}=\|\phi_k(\theta_k)-\phi'(\theta_k)\|$, evaluated along the realized
trajectory; it vanishes for $k\ge T$. We assume that we are given deterministic numbers
$s_{i,k}\ge0$ such that
\begin{equation}
\label{eq:sensitivity-bound}
\Delta_{i,k}\le s_{i,k}
\quad\text{almost surely},
\qquad k<T,
\end{equation}
and we set $s_{i,k}:=0$ for $k\ge T$. Only the contractivity of $\phi'$ and
\eqref{eq:sensitivity-bound} are used below: the learning maps $\phi_k$, $k<T$, may be
arbitrary, which we exploit for the virtual mechanism of
Appendix~\ref{app:proof-thm-ridge}.

We use the sums $N_{i,j}$, $N_i$ and $V(\sigma)$ of~\eqref{eq:prefix}; since $s_{i,k}=0$ for
$k\ge T$, $N_{i,j}=N_i$ for $j\ge T$. We also need the prefix of $V(\sigma)$,
\begin{equation}
\label{eq:proof-prefix}
V_j(\sigma)
:=
\sum_{k<j}2\eta\sigma_k^2c^{\,2(T+K-1-k)},
\qquad j=1,\dots,T+K,
\end{equation}
where $\sigma_k=\slearn$ for $k<T$ and $\sigma_k=\sigma$ for $k\ge T$, so that
$V(\sigma)=V_{T+K}(\sigma)$. Appendices~\ref{app:interpolation}--\ref{app:qp} prove the
following result.

\begin{proposition}[Deterministic-sensitivity accounting]
\label{prop:accounting}
Assume~\eqref{eq:sensitivity-bound}, $\sunl>0$ and $N_i>0$. If
\begin{equation}
\label{eq:feas-crit}
\frac{V_j(\sunl)}{V(\sunl)}
\le
\frac{N_{i,j}}{N_i},
\qquad
j=1,\dots,T,
\end{equation}
then
\[
\mathcal T(\theta_{T+K},\theta'_{T+K})
\ge
G(\mu_i),
\qquad
\mu_i=\frac{N_i}{\sqrt{V(\sunl)}}.
\]
If~\eqref{eq:feas-crit} fails, the same bound holds with $\mu_i$ replaced by
$\sqrt{Q_i^\star}>N_i/\sqrt{V(\sunl)}$, where $Q_i^\star$ is the optimal value of
problem~\eqref{eq:shift-allocation-qp} below.
\end{proposition}

\subsection{Shifted Interpolated Process}
\label{app:interpolation}

Following \citet{bok2024shiftedinterpolationdifferentialprivacy}, we introduce an
auxiliary process $\{\tilde\theta_k\}_{k=0}^{T+K}$ that starts on the retraining path and
ends on the unlearned one. Set $\tilde\theta_0=\theta_0=\theta'_0$ and
\begin{equation}
\label{eq:interpolated}
\tilde\theta_{k+1}
=
\lambda_{k+1}\phi_k(\theta_k)
+
(1-\lambda_{k+1})\phi'(\tilde\theta_k)
+
Z_{k+1},
\qquad k=0,\dots,T+K-1,
\end{equation}
where the shift parameters $\lambda_{k+1}\in[0,1]$ are deterministic and
$\lambda_{T+K}=1$, so that $\tilde\theta_{T+K}=\theta_{T+K}$. The process
$\tilde\theta$ uses the same noise $Z_{k+1}$ as $\theta$; this coupling is what allows us to
control the distance between the two processes deterministically.

Define deterministic sequences by $z_0=0$ and
\begin{equation}
\label{eq:a-def}
z_{k+1}
=
(1-\lambda_{k+1})\bigl(cz_k+s_{i,k}\bigr),
\qquad
a_{k+1}
=
\lambda_{k+1}\bigl(cz_k+s_{i,k}\bigr),
\end{equation}
so that
\begin{equation}
\label{eq:az-id}
z_{k+1}+a_{k+1}
=
cz_k+s_{i,k}.
\end{equation}
Here $z_k$ bounds the residual gap between $\theta_k$ and $\tilde\theta_k$, while
$a_{k+1}$ is the part of the gap absorbed at step $k$, which is paid for in privacy.

\paragraph{Claim.}
$\|\theta_k-\tilde\theta_k\|\le z_k$ almost surely for all $k$.

\medskip
The claim holds at $k=0$. Assume it holds at step $k$. The shared noise cancels, so
\[
\|\theta_{k+1}-\tilde\theta_{k+1}\|
=
(1-\lambda_{k+1})
\|\phi_k(\theta_k)-\phi'(\tilde\theta_k)\|,
\]
and, by the triangle inequality,
\begin{equation}
\label{eq:decomposition}
\|\phi_k(\theta_k)-\phi'(\tilde\theta_k)\|
\le
\|\phi_k(\theta_k)-\phi'(\theta_k)\|
+
\|\phi'(\theta_k)-\phi'(\tilde\theta_k)\|
\le
s_{i,k}+cz_k,
\end{equation}
where the first term is bounded by~\eqref{eq:sensitivity-bound} and the second by the
contractivity of $\phi'$ and the induction hypothesis. Hence
$\|\theta_{k+1}-\tilde\theta_{k+1}\|\le z_{k+1}$, which proves the claim.

\subsection{One-Step Bound and Composition}
\label{app:composition}

The privacy cost of each step is controlled by the shifted-step lemma of
\citet{bok2024shiftedinterpolationdifferentialprivacy}, which we use in the following form.

\begin{lemma}[Shifted step bound {\citep[Lemma~3.2]{bok2024shiftedinterpolationdifferentialprivacy}}]
\label{lem:shifted-step}
Let $\phi'$ be $c$-contractive, and let $\theta,\tilde\theta,\theta'$ be random variables
such that $\|\theta-\tilde\theta\|\le z$ and $\|\phi(\theta)-\phi'(\theta)\|\le s$ almost
surely, for deterministic $z,s\ge0$. Let $Z,Z'\sim\mathcal N(0,\sigma^2I)$ be independent of
each other and of $(\theta,\tilde\theta,\theta')$. Then, for any $\lambda\in[0,1]$,
\[
\mathcal T\bigl(
\lambda\phi(\theta)+(1-\lambda)\phi'(\tilde\theta)+Z,\,
\phi'(\theta')+Z'
\bigr)
\ge
\mathcal T(\tilde\theta,\theta')
\otimes
G\!\left(\frac{\lambda(cz+s)}{\sigma}\right).
\]
When $\lambda=0$, the bound
$\mathcal T\bigl(\phi'(\tilde\theta)+Z,\phi'(\theta')+Z'\bigr)\ge\mathcal T(\tilde\theta,\theta')$
holds for any $\sigma\ge0$, by post-processing.
\end{lemma}

The proof is that of \citet[Lemma~3.2]{bok2024shiftedinterpolationdifferentialprivacy},
with one modification: the sensitivity $\|\phi(\theta)-\phi'(\theta)\|$ is bounded by $s$ at
the realized iterate $\theta$ rather than uniformly over the parameter space. It enters only
through the triangle inequality~\eqref{eq:decomposition}, whose first term is at most $s$
almost surely and whose second term is at most $cz$ by contractivity of $\phi'$; the rest of
the argument is unchanged.

We apply Lemma~\ref{lem:shifted-step} at step $k$ with
$(\theta,\tilde\theta,\theta')=(\theta_k,\tilde\theta_k,\theta'_k)$, $\phi=\phi_k$,
$z=z_k$ (by the claim), $s=s_{i,k}$ and $\sigma=\sqrt{2\eta}\,\sigma_k$. Since
$\lambda_{k+1}(cz_k+s_{i,k})=a_{k+1}$,
\[
\mathcal T(\tilde\theta_{k+1},\theta'_{k+1})
\ge
\mathcal T(\tilde\theta_k,\theta'_k)
\otimes
G\!\left(\frac{a_{k+1}}{\sqrt{2\eta}\,\sigma_k}\right).
\]
Iterating from $k=T+K-1$ down to $k=0$, and using $\tilde\theta_{T+K}=\theta_{T+K}$,
$\tilde\theta_0=\theta'_0$ (so that $\mathcal T(\tilde\theta_0,\theta'_0)=G(0)$)
and the composition rule $G(\mu_1)\otimes G(\mu_2)=G\bigl(\sqrt{\mu_1^2+\mu_2^2}\bigr)$
\citep{dong2019gdp}, we obtain
\begin{equation}
\label{eq:composition}
\mathcal T(\theta_{T+K},\theta'_{T+K})
\ge
G\bigl(\sqrt{Q(a)}\bigr),
\qquad
Q(a):=\sum_{k=0}^{T+K-1}\frac{a_{k+1}^2}{2\eta\sigma_k^2},
\end{equation}
with the convention that a term $a^2/0$ equals $0$ if $a=0$ and $+\infty$ otherwise. In
particular, if $\slearn=0$, a finite bound requires $\lambda_{k+1}=0$, hence $a_{k+1}=0$,
for $k<T$: the learning steps then cost nothing.

\subsection{Optimal Shift Allocation}
\label{app:qp}

Since~\eqref{eq:composition} holds for every admissible choice of shifts, we minimize $Q$
over the shift-allocation problem
\begin{equation}
\label{eq:shift-allocation-qp}
\begin{aligned}
\min_{a,z}\quad & Q(a)\\
\text{subject to}\quad
& z_0=z_{T+K}=0,\\
& z_{k+1}=cz_k+s_{i,k}-a_{k+1},
&& k=0,\dots,T+K-1,\\
& 0\le a_{k+1}\le cz_k+s_{i,k},
&& k=0,\dots,T+K-1.
\end{aligned}
\end{equation}
Through~\eqref{eq:a-def}, these constraints are exactly those induced by
$\lambda_{k+1}\in[0,1]$ and $\lambda_{T+K}=1$.

\paragraph{Endpoint constraint.}
Multiplying~\eqref{eq:az-id} by $c^{T+K-1-k}$ and summing over $k=0,\dots,T+K-1$, the $z_k$
terms telescope. Since $z_0=z_{T+K}=0$,
\begin{equation}
\label{eq:endpoint}
\sum_{k=0}^{T+K-1}c^{T+K-1-k}a_{k+1}
=
\sum_{k=0}^{T-1}c^{T+K-1-k}s_{i,k}
=
N_i.
\end{equation}

\paragraph{Lower bound.}
Writing each term of~\eqref{eq:endpoint} as
$\bigl(\sqrt{2\eta}\,\sigma_kc^{T+K-1-k}\bigr)\cdot\bigl(a_{k+1}/(\sqrt{2\eta}\,\sigma_k)\bigr)$
and applying Cauchy--Schwarz, every feasible allocation with $Q(a)<\infty$ satisfies
\begin{equation}
\label{eq:cs}
Q(a)
\ge
\frac{N_i^2}{V(\sunl)},
\end{equation}
with equality if and only if
\begin{equation}
\label{eq:cs-candidate}
a_{k+1}^{\mathrm{CS}}
=
\frac{N_i}{V(\sunl)}\,
2\eta\sigma_k^2c^{T+K-1-k},
\qquad
k=0,\dots,T+K-1.
\end{equation}
Note that $a_{k+1}^{\mathrm{CS}}=0$ whenever $\sigma_k=0$. It remains to determine when this
allocation is feasible.

\begin{lemma}[Feasibility of the Cauchy--Schwarz allocation]
\label{lem:feasibility}
The allocation~\eqref{eq:cs-candidate} is feasible for~\eqref{eq:shift-allocation-qp} if and
only if
\[
\frac{V_j(\sunl)}{V(\sunl)}
\le
\frac{N_{i,j}}{N_i},
\qquad
j=1,\dots,T+K.
\]
For $j\ge T$, this condition holds automatically. Hence feasibility is equivalent
to~\eqref{eq:feas-crit}.
\end{lemma}

\begin{proof}
Let $z_0=0$ and generate $z_j$ from the recursion in~\eqref{eq:shift-allocation-qp} using
$a^{\mathrm{CS}}$. Since $a^{\mathrm{CS}}\ge0$, feasibility reduces to $z_j\ge0$ for all $j$,
together with $z_{T+K}=0$. Unrolling the recursion and substituting~\eqref{eq:cs-candidate}
gives
\begin{equation}
\label{eq:z-explicit}
z_j
=
c^{\,j-T-K}
\left(
N_{i,j}
-
\frac{N_i}{V(\sunl)}V_j(\sunl)
\right).
\end{equation}
Thus $z_j\ge0$ is equivalent to the stated condition, and $z_{T+K}=0$ always holds. For
$j\ge T$, $N_{i,j}=N_i$ and $V_j(\sunl)\le V(\sunl)$, so the condition is automatic.
\end{proof}

\paragraph{Proof of Proposition~\ref{prop:accounting}.}
If~\eqref{eq:feas-crit} holds, Lemma~\ref{lem:feasibility} shows that $a^{\mathrm{CS}}$ is
feasible, so it attains equality in~\eqref{eq:cs}. The optimal value
of~\eqref{eq:shift-allocation-qp} is then $Q_i^\star=N_i^2/V(\sunl)$, and
\eqref{eq:composition} gives $\mathcal T(\theta_{T+K},\theta'_{T+K})\ge G(\mu_i)$ with
$\mu_i=N_i/\sqrt{V(\sunl)}$.

If~\eqref{eq:feas-crit} fails, then necessarily $\slearn>0$, since for $\slearn=0$ we have
$V_j(\sunl)=0$ for $j\le T$. Problem~\eqref{eq:shift-allocation-qp} is still feasible, for
instance with $z_k=0$ and $a_{k+1}=s_{i,k}$ for all $k$. Its feasible set is compact and $Q$
is continuous, so the minimum $Q_i^\star$ is attained. By~\eqref{eq:cs}, and since the
equality case is infeasible, $Q_i^\star>N_i^2/V(\sunl)$. By~\eqref{eq:composition},
$\sqrt{Q_i^\star}$ is still a valid, but larger, GDP certificate. \qed

\subsection{From GDP to Per-Instance Unlearning}
\label{app:gdp-to-unlearning}

For a target $(\eps,\del')$, let
$\mu_{\mathrm{target}}:=\sup\{\mu>0:\epsgdp(\mu,\del')\le\eps\}$, as in
Theorems~\ref{thm:deterministic} ($\del'=\del$) and~\ref{thm:ridge-gdp-main} ($\del'=\dm$).
If $\mathcal T(\theta_{T+K},\theta'_{T+K})\ge G(\mu_i)$ with $\mu_i\le\mu_{\mathrm{target}}$,
then $G(\mu_i)\ge G(\mu_{\mathrm{target}})$, and the exact GDP-to-$(\eps,\del)$ conversion of
Appendix~\ref{app:gdp} shows that $\U_i$ satisfies $(\eps,\del')$ per-instance unlearning
(Definition~\ref{def:per_instance_unlearning}). If $N_i=0$, then $s_{i,k}=0$ for all $k$, the
choice $a\equiv0$, $z\equiv0$ is admissible, and the two final distributions coincide.

\section{Proof of Theorem~\ref{thm:deterministic}}
\label{app:proof-deterministic}

We apply the accounting of Appendix~\ref{app:accounting} with $\slearn=0$. The learning phase
is then deterministic, so the sensitivities
$\Delta_{i,k}$, $k<T$, are deterministic numbers, known exactly from the gradient norms
stored during training, and~\eqref{eq:sensitivity-bound} holds with equality for
$s_{i,k}:=\Delta_{i,k}$.

\paragraph{Certificate.}
Since $\sigma_k=0$ for $k<T$, we have $V_j(\sunl)=0$ for $j\le T$, so~\eqref{eq:feas-crit}
holds trivially, and
\[
V(\sunl)=2\eta\sunl^2\sum_{k=0}^{K-1}c^{2k}.
\]
If $N_i=0$, the claim is immediate (Appendix~\ref{app:gdp-to-unlearning}). Otherwise,
Proposition~\ref{prop:accounting} gives
\begin{equation}
\label{eq:det-mu}
\mathcal T(\theta_{T+K},\theta'_{T+K})
\ge
G(\mu_i),
\qquad
\mu_i
=
\frac{N_i}{\sunl\sqrt{2\eta\sum_{k=0}^{K-1}c^{2k}}},
\end{equation}
and the choice~\eqref{eq:mu-deterministic} of $\sunl$ gives exactly
$\mu_i=\mu_{\mathrm{target}}$. Appendix~\ref{app:gdp-to-unlearning} then yields
$(\eps,\del)$ per-instance unlearning, which proves Theorem~\ref{thm:deterministic}. \qed

\paragraph{Reading as amplification from an initial gap.}
With $\slearn=0$, the optimal shifts vanish during learning ($a^{\mathrm{CS}}_{k+1}=0$ for
$k<T$), and the certificate admits a more intuitive reading. The two deterministic
trajectories drift apart during learning: writing $d_k:=\|\theta_k-\theta'_k\|$, the
triangle inequality and the contractivity of $\phi'$ give $d_{k+1}\le\Delta_{i,k}+c\,d_k$, so
that, since $d_0=0$,
\begin{equation}
\label{eq:det-displacement}
d_T
\le
\sum_{k=0}^{T-1}c^{T-1-k}\Delta_{i,k}
=
\tilde N_i ,
\end{equation}
where $N_i=c^K\tilde N_i$, as in Section~\ref{sec:deterministic}. During unlearning, the two
trajectories apply the same $c$-contractive map with the same noise. The
certificate~\eqref{eq:det-mu} is therefore pure privacy amplification by iteration from an
initial gap of at most $\tilde N_i$. In R\'enyi divergence, this is the diameter-aware bound
of \citet[Proposition~2.10 and Remark~A.4]{altschuler23a};
Proposition~\ref{prop:accounting} gives its GDP counterpart directly.

\section{Proof of Theorem~\ref{thm:ridge-gdp-main}}
\label{app:proof-thm-ridge}

We now allow noise during learning, $\slearn>0$, and apply the accounting of
Appendix~\ref{app:accounting}. We first show that the calibration of
Theorem~\ref{thm:ridge-gdp-main} is valid for any deterministic sensitivity sequence
satisfying~\eqref{eq:sensitivity-bound}. We then derive the ridge high-probability bounds
$s_{i,k}^{\ds}$ and handle the fact that they only hold with probability $1-\ds$.

\subsection{Calibration and Feasibility}
\label{app:ridge-calibration}

Let $\mu_{\mathrm{target}}$ and $\sunl$ be as in Theorem~\ref{thm:ridge-gdp-main}, and let
$(s_{i,k})_{k<T}$ satisfy~\eqref{eq:sensitivity-bound}, with $N_i>0$. By definition of
$\sunl$,
\[
V(\sunl)\ge\frac{N_i^2}{\mu_{\mathrm{target}}^2}.
\]
For $j\le T$, only learning noise enters $V_j$, so condition~\eqref{eq:sufficient-main}
gives
\[
V_j(\sunl)
=
2\eta\slearn^2\sum_{k<j}c^{2(T+K-1-k)}
\le
\frac{N_iN_{i,j}}{\mu_{\mathrm{target}}^2}.
\]
Combining the two bounds, $V_j(\sunl)/V(\sunl)\le N_{i,j}/N_i$ for $j=1,\dots,T$, which
is~\eqref{eq:feas-crit}. Written in terms of $\slearn$, requiring this for every $j\le T$
is exactly
\begin{equation}
\label{eq:sufficient-condition}
\slearn^{2}
\le
\min_{1\le j\le T}
\frac{N_i\,N_{i,j}}
{2\eta\,\mu_{\mathrm{target}}^{2}
\sum_{k<j}c^{2(T+K-1-k)}},
\end{equation}
that is, condition~\eqref{eq:sufficient-main}. Proposition~\ref{prop:accounting} then gives
\begin{equation}
\label{eq:ridge-mu-bound}
\mathcal T(\theta_{T+K},\theta'_{T+K})\ge G(\mu_i),
\qquad
\mu_i=\frac{N_i}{\sqrt{V(\sunl)}}\le\mu_{\mathrm{target}},
\end{equation}
and hence, by Appendix~\ref{app:gdp-to-unlearning}, $(\eps,\dm)$ per-instance unlearning for
any mechanism of the form~\eqref{eq:trajectories} whose sensitivities satisfy
\eqref{eq:sensitivity-bound}. If~\eqref{eq:sufficient-condition} fails,
Proposition~\ref{prop:accounting} still provides the valid, but larger, certificate
$\sqrt{Q_i^\star}$.

For ridge regression, however, the sensitivities $\Delta_{i,k}$ are random and unbounded, so
no deterministic sequence satisfies~\eqref{eq:sensitivity-bound} almost surely. The rest of
this appendix constructs sequences $s_{i,k}^{\ds}$ that hold with probability at least
$1-\ds$, and then accounts for the failure event through a virtual mechanism.

\subsection{Ridge Regression and Gaussian Iterates}

We now specialize to the multi-output ridge-regression setting of Section~\ref{sec:ridge},
with $X\in\R^{n\times p}$, $Y\in\R^{n\times d}$, and parameter
$\theta\in\R^{p\times d}$. The objective
$f_{\D}(\theta)=\sum_{j=1}^n\tfrac12\|x_j^\top\theta-y_j\|_2^2+\tfrac{\lambda}{2}\|\theta\|_F^2$
has gradient $\nabla_\theta f_{\D}(\theta)=(X^\top X+\lambda I_p)\theta-B$ with
$B:=X^\top Y$, so the learning phase follows the linear recursion
\begin{equation}
\label{eq:langevin-proof}
\theta_{k+1}
=
M\theta_k+\eta B+\sqrt{2\eta\sigma_{\mathrm{learn}}^2}\,\Xi_k,
\qquad
M:=I_p-\eta(X^\top X+\lambda I_p),
\end{equation}
where $\Xi_k\in\R^{p\times d}$ has i.i.d.\ $\mathcal N(0,1)$ entries, independent across
$k$, and $\theta_0$ is deterministic. For a point $(x_i,y_i)$, we write
$r_{i,k}:=x_i^\top\theta_k-y_i\in\R^d$ for its residual.

\begin{lemma}[Gaussian law of the iterates and of the residuals]
\label{lem:gaussian-kronecker-proof}
For every $k\ge0$,
\[
\mathrm{vec}(\theta_k)
\sim
\mathcal N\!\bigl(\mathrm{vec}(m_k),\, I_d\otimes\Sigma_k\bigr),
\qquad
m_k
=
M^k\theta_0+\eta\sum_{j=0}^{k-1}M^j B,
\qquad
\Sigma_k
=
2\eta\sigma_{\mathrm{learn}}^2
\sum_{j=0}^{k-1} M^j (M^j)^\top .
\]
Consequently, $r_{i,k}\sim\mathcal N(u_{i,k},\,v_{i,k}I_d)$ with
\[
u_{i,k}
=
x_i^\top m_k-y_i,
\qquad
v_{i,k}
=
x_i^\top\Sigma_k x_i
=
2\eta\sigma_{\mathrm{learn}}^2
\sum_{j=0}^{k-1}\|(M^j)^\top x_i\|_2^2 ,
\]
which are the quantities $u_{i,k}$ and $v_{i,k}$ of Theorem~\ref{thm:ridge-gdp-main}.
\end{lemma}

\begin{proof}
The recursion~\eqref{eq:langevin-proof} is affine with additive Gaussian noise and
$\theta_0$ is deterministic, so $\mathrm{vec}(\theta_k)$ is Gaussian. The $d$ columns of
$\theta_k$ follow recursions with the same matrix $M$, driven by independent noises: they
are independent, with a common covariance $\Sigma_k$. Taking the mean and the covariance
in~\eqref{eq:langevin-proof} gives $m_{k+1}=Mm_k+\eta B$ and
$\Sigma_{k+1}=M\Sigma_kM^\top+2\eta\sigma_{\mathrm{learn}}^2 I_p$, with $m_0=\theta_0$ and
$\Sigma_0=0$, and unrolling these recursions gives the closed forms. Finally,
$r_{i,k}=(I_d\otimes x_i^\top)\,\mathrm{vec}(\theta_k)-y_i$ is an affine image of
$\mathrm{vec}(\theta_k)$, hence Gaussian with mean $x_i^\top m_k-y_i$ and covariance
$(I_d\otimes x_i^\top)(I_d\otimes\Sigma_k)(I_d\otimes x_i)=(x_i^\top\Sigma_k x_i)\,I_d$.
\end{proof}

\subsection{High-Probability Sensitivity Control}

Since $\nabla_\theta\ell(\theta_k;x_i,y_i)=x_i r_{i,k}^\top$, the per-step sensitivity is
\[
\Delta_{i,k}
=
\eta\|\nabla_\theta\ell(\theta_k;x_i,y_i)\|_F
=
\eta\|x_i\|_2\|r_{i,k}\|_2.
\]

For $k=0$, the initialization is deterministic, hence
\[
r_{i,0}=x_i^\top\theta_0-y_i,
\qquad
\Delta_{i,0}
=
\eta\|x_i\|_2\|x_i^\top\theta_0-y_i\|_2.
\]
We therefore set
\[
s_{i,0}^{\delta_{\mathrm s}}
:=
\eta\|x_i\|_2\|x_i^\top\theta_0-y_i\|_2,
\]
so that $\Delta_{i,0}\le s_{i,0}^{\delta_{\mathrm s}}$ holds deterministically.

For $k\ge 1$, Lemma~\ref{lem:gaussian-kronecker-proof} gives
\[
z_{i,k}:=\frac{r_{i,k}}{\sqrt{v_{i,k}}}
\sim
\mathcal N\!\left(\frac{u_{i,k}}{\sqrt{v_{i,k}}},\,I_d\right),
\]
and therefore
\[
\frac{\|r_{i,k}\|_2^2}{v_{i,k}}
=
\|z_{i,k}\|_2^2
\sim
\chi'^2_d\!\left(\frac{\|u_{i,k}\|_2^2}{v_{i,k}}\right).
\]
Let $q_{i,k}(\cdot)$ denote the quantile function of this distribution and define,
for $k\ge1$,
\[
s_{i,k}^{\delta_{\mathrm s}}
:=
\eta\,\|x_i\|_2
\sqrt{
v_{i,k}\,
q_{i,k}\!\left(1-\tfrac{\delta_{\mathrm s}}{T}\right)
}.
\]
Define the event
\[
\mathcal G_i
:=
\Bigl\{
\forall k<T,\;
\Delta_{i,k}\le s_{i,k}^{\delta_{\mathrm s}}
\Bigr\}.
\]
For each $k=1,\dots,T-1$, the definition of the quantile gives
\[
\Prob\!\left(\Delta_{i,k}\le s_{i,k}^{\delta_{\mathrm s}}\right)
\ge
1-\tfrac{\delta_{\mathrm s}}{T}.
\]
The bound at $k=0$ holds deterministically.  Hence a union bound gives
\[
\Prob(\mathcal G_i^c)
\le
\sum_{k=1}^{T-1}
\Prob\!\left(\Delta_{i,k}>s_{i,k}^{\delta_{\mathrm s}}\right)
\le
(T-1)\frac{\delta_{\mathrm s}}{T}
\le
\delta_{\mathrm s}.
\]
Hence
\[
\Prob(\mathcal G_i)\ge 1-\delta_{\mathrm s}.
\]

\subsection{From the Virtual Mechanism to the Unconditional Guarantee}
\label{app:virtual-mechanism}

Recall the definition of certified per-instance unlearning:
a randomized mechanism $\U_i$ satisfies
$(\varepsilon,\delta)$-per-instance unlearning for point $(x_i,y_i)$ if for all
measurable events $S$,
\begin{equation}
\label{eq:def-per-instance-proof}
\Prob \big(\U_i(\A(\D),\Di,\{(x_i,y_i)\})\in S\big)
\le
e^\eps
\Prob \big(\U_i(\A(\Di),\Di,\varnothing)\in S\big)
+\del .
\end{equation}

The event $\mathcal G_i$ cannot be used by directly conditioning the privacy
accounting on it: after conditioning on $\mathcal G_i$, the learning noises are
no longer independent isotropic Gaussians, so the deterministic-sensitivity accounting
proved above cannot be applied to the conditional process. We therefore introduce a virtual
mechanism.

Define a virtual mechanism $\widetilde{\mathcal A}$ that follows the same
Langevin recursion as $\mathcal A$, except that the contribution of point $i$
to the one-step update is deterministically clipped so that, at every
learning step $k<T$,
\[
\widetilde\Delta_{i,k}
\le
s_{i,k}^{\delta_{\mathrm s}} .
\]
The same unlearning procedure is then applied to $\widetilde{\mathcal A}$,
yielding the virtual unlearning mechanism
$\widetilde{\U_i}\circ\widetilde{\mathcal A}$.

For this virtual mechanism, the sensitivity bounds
$s_{i,k}^{\delta_{\mathrm s}}$ hold deterministically, so the calibration of
Appendix~\ref{app:ridge-calibration} applies and~\eqref{eq:ridge-mu-bound} holds,
yielding $(\varepsilon,\delta_{\mathrm m})$ per-instance unlearning for
the virtual mechanism.  Equivalently, for every measurable $S$,
\begin{equation}
\label{eq:virtual-dp}
\Prob\big(
\widetilde{\U_i}(\widetilde{\mathcal A}(\D),\Di,\{(x_i,y_i)\})\in S
\big)
\le
e^\eps
\Prob\big(
\widetilde{\U_i}(\widetilde{\mathcal A}(\Di),\Di,\varnothing)\in S
\big)
+
\delta_{\mathrm m}.
\end{equation}

Moreover, by construction, the true and virtual mechanisms produce identical
trajectories whenever $\mathcal G_i$ holds. Therefore their total variation
distance is bounded by the probability of the failure event:
\[
d_{\mathrm{TV}}
\Bigl(
\U_i(\mathcal A(\D),\Di,\{(x_i,y_i)\}),
\widetilde{\U_i}(\widetilde{\mathcal A}(\D),\Di,\{(x_i,y_i)\})
\Bigr)
\le
\Prob(\mathcal G_i^c)
\le
\delta_{\mathrm s}.
\]
On the neighboring dataset $\Di$, the point $i$ is absent, so the clipping rule
is never invoked and the virtual and true retain-side mechanisms coincide:
\[
\widetilde{\U_i}(\widetilde{\mathcal A}(\Di),\Di,\varnothing)
=
\U_i(\mathcal A(\Di),\Di,\varnothing).
\]

Thus, for any measurable $S$,
\begin{align*}
\Prob\big(
\U_i(\mathcal A(\D),\Di,\{(x_i,y_i)\})\in S
\big)
&\le
\Prob\big(
\widetilde{\U_i}(\widetilde{\mathcal A}(\D),\Di,\{(x_i,y_i)\})\in S
\big)
+
\delta_{\mathrm s}
\\
&\le
e^\eps
\Prob\big(
\widetilde{\U_i}(\widetilde{\mathcal A}(\Di),\Di,\varnothing)\in S
\big)
+
\delta_{\mathrm m}
+
\delta_{\mathrm s}
\\
&=
e^\eps
\Prob\big(
\U_i(\mathcal A(\Di),\Di,\varnothing)\in S
\big)
+
\delta_{\mathrm m}
+
\delta_{\mathrm s}.
\end{align*}
Since $\delta=\delta_{\mathrm m}+\delta_{\mathrm s}$, this proves
\eqref{eq:def-per-instance-proof} and completes the proof of
Theorem~\ref{thm:ridge-gdp-main}. \qed

\section{Training-Time DP Guarantee for Ridge Langevin}
\label{app:ridge-dp-guarantee}

Section~\ref{sec:ridge} mentions that the learning noise $\slearn$ may be calibrated to make
the learning algorithm $\A$ itself differentially private. This is not immediate for ridge
regression: the per-example gradient $x_j(x_j^\top\theta-y_j)^\top$ is unbounded in
$\theta$, so the update~\eqref{eq:langevin-learn} has no finite uniform sensitivity. For
bounded data, however, no new privacy analysis is needed. With high probability, $\A$
coincides with clipped noisy gradient descent, whose last iterate is private by privacy
amplification by iteration
\citep{Feldman_2018,bok2024shiftedinterpolationdifferentialprivacy}; this is the accountant
already used for the DP-GD baseline (Appendix~\ref{app:learning-dp-baseline}). The only
ingredient specific to ridge regression is a bound on the per-example gradients along the
trajectory (Lemma~\ref{lem:ridge-bounded}). Unlike the per-instance bounds $s_{i,k}^{\ds}$
of Theorem~\ref{thm:ridge-gdp-main}, which are computed from the data, this bound must hold
uniformly over datasets, which is why we assume bounded data.

\paragraph{Setting.}
We use add/remove adjacency over datasets of at most $n$ examples with $\|x_j\|_2\le R_x$
and $\|y_j\|_2\le R_y$. The initialization $\theta_0$ is fixed, $\lambda>0$, and
$0<\eta\le1/(nR_x^2+\lambda)$, so that Assumption~\ref{ass:contraction} holds with
$c=1-\eta\lambda$ for every such dataset. The mechanism releases $\theta_T=\A(D)$ only.

\begin{lemma}[Bounded per-example gradients]
\label{lem:ridge-bounded}
Let $\rho\in(0,1)$, $Q_T:=\sum_{k<T}c^{2k}$, and
\[
R_\theta:=\|\theta_0\|_F+\frac{\sqrt n\,R_y}{2\sqrt\lambda}
+\slearn\sqrt{2\eta Q_T}\,\bigl(\sqrt{pd}+\sqrt{2\log(T/\rho)}\bigr),
\qquad
C_\rho:=R_x\,(R_xR_\theta+R_y).
\]
With probability at least $1-\rho$, the iterates of~\eqref{eq:langevin-learn} satisfy
$\|x_j(x_j^\top\theta_k-y_j)^\top\|_F\le C_\rho$ for all $k<T$ and all $(x_j,y_j)\in D$.
\end{lemma}

\begin{proof}
By Lemma~\ref{lem:gaussian-kronecker-proof},
$\mathrm{vec}(\theta_k)\sim\mathcal N\bigl(\mathrm{vec}(m_k),\,I_d\otimes\Sigma_k\bigr)$ with
$m_k=M^k\theta_0+\eta\sum_{j<k}M^jX^\top Y$ and
$\Sigma_k=2\eta\slearn^2\sum_{j<k}M^{2j}$, where $M=I_p-\eta(X^\top X+\lambda I_p)$
satisfies $0\preceq M\preceq cI_p$ by the choice of $\eta$.

\emph{Mean.} The singular values of $\eta\sum_{j<k}M^jX^\top$ are at most
$\sqrt s/(s+\lambda)\le1/(2\sqrt\lambda)$, where $s$ ranges over the eigenvalues of
$X^\top X$. Since $\|Y\|_F\le\sqrt n\,R_y$, we get
$\|m_k\|_F\le\|\theta_0\|_F+\sqrt n\,R_y/(2\sqrt\lambda)$.

\emph{Fluctuation.} Since $\Sigma_k\preceq2\eta\slearn^2Q_T\,I_p$, Gaussian concentration
gives
\[
\|\theta_k-m_k\|_F\le\slearn\sqrt{2\eta Q_T}\,\bigl(\sqrt{pd}+\sqrt{2\log(T/\rho)}\bigr)
\]
with probability at least $1-\rho/T$.

By a union bound over $k<T$, $\|\theta_k\|_F\le R_\theta$ for all $k<T$ with probability at
least $1-\rho$. On this event,
$\|x_j(x_j^\top\theta_k-y_j)^\top\|_F=\|x_j\|_2\,\|x_j^\top\theta_k-y_j\|_2
\le R_x(R_xR_\theta+R_y)$.
\end{proof}

\begin{proposition}[Training-time DP]
\label{prop:ridge-training-dp}
Let $\rho$ be such that $\delta_0:=\del-(1+e^\eps)\rho>0$, and let
\[
\mu_\rho
:=
\sqrt{\frac{\eta}{2}}\;\frac{C_\rho}{\slearn}\,
\sqrt{\frac{(1+c)(1-c^{T})}{(1-c)(1+c^{T})}} .
\]
If $\epsgdp(\mu_\rho,\delta_0)\le\eps$, then the learning algorithm $\A$ is
$(\eps,\del)$-DP.
\end{proposition}

\begin{proof}
Let $\widetilde\A$ be $\A$ with each per-example gradient clipped at norm $C_\rho$. This is
the clipped noisy gradient descent of Appendix~\ref{app:learning-dp-baseline} with $C=C_\rho$
and $\sigma=\slearn\sqrt{2/\eta}$: removing one example changes the update by at most
$\eta C_\rho$, the noise added at each step has standard deviation $\sqrt{2\eta}\,\slearn$,
and the update map remains $c$-contractive, because clipping the gradient amounts to
replacing the squared loss by a Huber loss of the residual, which is convex with the same
smoothness constant \citep{song2021evading}. The last iterate of $\widetilde\A$ is therefore
$\mu_\rho$-GDP \citep[Theorem~4.2]{bok2024shiftedinterpolationdifferentialprivacy}; this is
the bound $\mu_{\mathrm{last}}$ of Appendix~\ref{app:learning-dp-baseline}, that is, the
analysis of Appendix~\ref{app:accounting} with constant sensitivity $\eta C_\rho$. By
Appendix~\ref{app:gdp}, $\widetilde\A$ is $(\eps,\delta_0)$-DP.

It remains to transfer this guarantee to $\A$, as in Appendix~\ref{app:virtual-mechanism}.
Run $\A$ and $\widetilde\A$ with the same noise. By Lemma~\ref{lem:ridge-bounded}, on any
admissible dataset no gradient is clipped with probability at least $1-\rho$, in which case
the two outputs coincide. Hence, for neighbors $D,D'$ and any measurable set $S$,
\[
\Prob(\A(D)\in S)
\le\Prob(\widetilde\A(D)\in S)+\rho
\le e^\eps\,\Prob(\widetilde\A(D')\in S)+\delta_0+\rho
\le e^\eps\,\Prob(\A(D')\in S)+\delta_0+(1+e^\eps)\rho ,
\]
and the right-hand side is $e^\eps\,\Prob(\A(D')\in S)+\del$.
\end{proof}

\paragraph{Numerical illustration.}
Take the dimensions of our experiments, $n=38400$, $p=1537$, $d=1$, $\lambda=0.1\,n$,
$\slearn=0.1$ and $T=800$, with normalized data ($R_x=R_y=1$), $\theta_0=0$,
$\eta=1/(nR_x^2+\lambda)$, $\del=10^{-5}$ and $\rho=10^{-8}$. Then $R_\theta\approx1.66$ and
$C_\rho\approx2.66$. The dynamics is strongly contractive ($c\approx0.91$, so
$c^T\approx0$): the last factor of $\mu_\rho$ is $\sqrt{(1+c)/(1-c)}\approx4.6$, which gives
$\mu_\rho\approx0.42$ and $\eps\approx1.64$. The learning algorithm is thus
$(1.64,10^{-5})$-DP. For comparison, composition over the $T$ steps replaces this factor by
$\sqrt T\approx28$, so that $\mu\approx2.6$, which corresponds to $\eps\approx14$ at the same
$\del$.

\section{Privacy Consequences of the Per-Instance Noise Calibration}
\label{app:sigma-leak}

Definition~\ref{def:per_instance_unlearning} compares two executions of the same calibrated
procedure. In our construction the calibration is the noise level $\sunl=g(D,D_f)$, a
deterministic function of the request, which is used on both sides of the comparison. This
appendix identifies the information about $D_f$ that lies outside the $(\eps,\del)$ statement,
shows that it reduces to the single scalar $\sunl$, and shows how to bound it. Throughout,
$D_r:=D\setminus D_f$ denotes the retain set and $\U^{\sunl}$ the unlearning procedure run with
noise level $\sunl$.

The reference side of Definition~\ref{def:per_instance_unlearning},
$\U^{\sunl}(\A(D_r),D_r,\varnothing)$, trains on the retain set only and depends on the request
only through $\sunl$. Up to $(\eps,\del)$, any test applied to the released model by an observer
who knows $D_r$ is therefore a test on the pair $(\sunl,D_r)$: beyond the guarantee, the released
model carries no information about $D_f$ other than the scalar $\sunl$. We conservatively treat
$\sunl$ as disclosed. By Theorems~\ref{thm:deterministic} and~\ref{thm:ridge-gdp-main}, it is an
increasing function of the cumulative sensitivity $N_i$ of the deleted point, so it summarizes
how influential that point was during training; it does not indicate whether the point was part
of the training set, which is the property certified by the $(\eps,\del)$ statement.

The level may not be lowered below $g(D,D_f)$, which is needed to certify the request, but it may
be raised, since the certificate is nonincreasing in $\sunl$. The next remark shows that, under
this constraint, hiding the level entirely would require serving essentially every request at
the worst-case level, that is, reverting to uniform calibration.

\begin{remark}[Adaptive levels cannot be fully hidden]
\label{rem:adaptive-not-private}
Fix $D_r$, let $L(z)$ be a randomized noise level used to delete $z$ from $D_r\cup\{z\}$, with
$L(z)\ge g(D_r\cup\{z\},\{z\})$ almost surely, and suppose that the law of $L$ is
$(\eps',\del')$-indistinguishable between any two points $z,z'$. Let $z'$ attain
$\sigma_{\max}:=\sup_{z}g(D_r\cup\{z\},\{z\})$. Applying the indistinguishability to the
event $\{L\ge\sigma_{\max}\}$, which has probability one under $z'$, gives, for every $z$,
\[
\Prob\bigl(L(z)\ge\sigma_{\max}\bigr)
\ \ge\
e^{-\eps'}\Bigl(\Prob\bigl(L(z')\ge\sigma_{\max}\bigr)-\del'\Bigr)
\ =\
e^{-\eps'}(1-\del').
\]
Every request, including the least influential one, is thus served at the worst-case level with
probability at least $e^{-\eps'}(1-\del')$, which tends to one as $\eps'$ and $\del'$ vanish:
hiding the level forces a mechanism back towards uniform calibration and forfeits the gain of
per-instance calibration. This holds for any calibration that adapts the noise to the request,
not only for ours.
\end{remark}

We therefore bound the disclosure at a controlled cost in noise by coarsening the calibration,
in the spirit of the private release of per-instance privacy parameters
\citep{redberg2021privately}.

Let $\sigma^{(1)}<\dots<\sigma^{(J)}$ be a grid of noise levels fixed before any data is seen,
with $\sigma^{(J)}\ge\sup_{(D,D_f)}g(D,D_f)$, for instance the uniform level, and let
\[
\bar g(D,D_f):=\min\bigl\{\sigma^{(j)}:\ \sigma^{(j)}\ge g(D,D_f)\bigr\}
\]
round the required level up to the grid.

\begin{proposition}[Coarsened calibration]
\label{prop:grid}
\begin{enumerate}
\item[(i)] $\U^{\bar g(D,D_f)}$ satisfies $(\eps,\del)$ per-instance unlearning for every request
$(D,D_f)$ for which $\U^{g(D,D_f)}$ does.
\item[(ii)] The released model is $(\eps,\del)$-close to a model that depends on the request only
through the index $j\in\{1,\dots,J\}$ of the level used, which carries at most $\log_2 J$ bits
about the request.
\item[(iii)] For any retain set $D_r$ and any two points $z,z'$ whose deletions from
$D_r\cup\{z\}$ and $D_r\cup\{z'\}$ use the same level, the two released models are
$(2\eps,(1+e^{\eps})\del)$-indistinguishable.
\end{enumerate}
\end{proposition}

\begin{proof}
(i) The certificates of Theorems~\ref{thm:deterministic} and~\ref{thm:ridge-gdp-main} are
nonincreasing in $\sunl$, so a level at least $g(D,D_f)$ still certifies the request.
(ii) By Definition~\ref{def:per_instance_unlearning}, the released model is
$(\eps,\del)$-close to $\U^{\sigma^{(j)}}(\A(D_r),D_r,\varnothing)$, which depends on the request
only through $j$.
(iii) Let $P_z$ and $P_{z'}$ be the laws of the two released models and $Q$ the law of
$\U^{\sigma^{(j)}}(\A(D_r),D_r,\varnothing)$, their common reference side in
Definition~\ref{def:per_instance_unlearning}. For every measurable $S$,
$P_z(S)\le e^{\eps}Q(S)+\del\le e^{\eps}\bigl(e^{\eps}P_{z'}(S)+\del\bigr)+\del
=e^{2\eps}P_{z'}(S)+(1+e^{\eps})\del$, and symmetrically.
\end{proof}

Item~(iii) states that an observer who knows $D_r$ cannot tell which of two points was deleted
when they share a level. With $J=1$ and $\sigma^{(1)}=\sigma^{(J)}$, calibration is uniform and
(iii) holds for every pair $z,z'$. A geometric grid $\sigma^{(j)}=\sigma^{(1)}r^{j-1}$ multiplies
the noise by at most $r$ and needs $J=1+\lceil\log(\sigma^{(J)}/\sigma^{(1)})/\log r\rceil$ levels.
In the experiments of Section~\ref{sec:experiments-linear}, the required noise varies by a
factor of about $10$ across the five representative points: $r=2$ then gives $J=5$ levels, that
is, at most $\log_2 5\approx2.3$ bits, for at most twice the noise of exact per-instance
calibration. Over all training points, the ratio between the uniform level and the smallest
per-instance level is between $45$ and $52$ depending on $\eps$
(Appendix~\ref{app:gain-uniform}), which requires $J=7$ levels, that is, less than $3$
bits. The grid size thus interpolates between uniform calibration ($J=1$) and fully
per-instance calibration, trading noise against disclosure.

\section{Extension to Group Unlearning}
\label{app:group-unlearning}

The results of Section~\ref{sec:accounting}, in particular Theorems~\ref{thm:deterministic} and~\ref{thm:ridge-gdp-main}, extend to group deletion.
The key observation is that the relevant sensitivity is the norm of
the \emph{sum} of the contributions of the forgotten points.

Let 
\[ 
D_f=\{(x_i,y_i):i\in I_f\}, 
\qquad 
I_f=\{i_1,\ldots,i_R\}, 
\] 
and write 
\[ 
D_{-I_f}:=D\setminus D_f. 
\]
For a group deletion request, define the per-step group sensitivity
\begin{equation}
\Delta_{I_f,k}
:=
\left\|
\eta\bigl(
\nabla f_{D_{-I_f}}(\theta_k)
-
\nabla f_D(\theta_k)
\bigr)
\right\|
=
\eta
\left\|
\sum_{i\in I_f}
\nabla_\theta \ell(\theta_k;x_i,y_i)
\right\|.
\label{eq:group-sensitivity}
\end{equation}
As in the single-point case, $\Delta_{I_f,k}=0$ for $k\ge T$,
since both compared trajectories then follow the same update map.

Throughout this appendix, Assumption~\ref{ass:contraction}
is understood with $D_{-i}$ replaced by $D_{-I_f}$: both
$f_D$ and $f_{D_{-I_f}}$ gradient maps are contractive with the
same factor $c\le 1$.

\subsection{Deterministic Training}

When $\sigma_{\mathrm{learn}}=0$, the learning trajectory is deterministic and the group
sensitivity~\eqref{eq:group-sensitivity} is available exactly at every learning step.

\begin{corollary}[Group extension of Theorem~\ref{thm:deterministic}: deterministic training]
\label{cor:group-deterministic}
Assume the group version of Assumption~\ref{ass:contraction} with contraction factor $c\le 1$,
and run the learning phase with $\sigma_{\mathrm{learn}}=0$. Define
\[
N_{I_f}
:=
\sum_{k=0}^{T-1}c^{T+K-1-k}\Delta_{I_f,k},
\qquad
\Delta_{I_f,k}
=
\eta\left\|\sum_{i\in I_f}\nabla_\theta\ell(\theta_k;x_i,y_i)\right\|.
\]
For a target $(\varepsilon,\delta)$, let
\[
\mu_{\mathrm{target}}
:=
\sup\{\mu>0:\varepsilon_{\mathrm{GDP}}(\mu,\delta)\le\varepsilon\}.
\]
Then the group unlearning procedure run for $K$ steps on $D_{-I_f}$ with
\begin{equation}
\label{eq:group-deterministic-calibration}
\sigma_{\mathrm{unlearn}}
=
\frac{N_{I_f}}
{\mu_{\mathrm{target}}\sqrt{2\eta\sum_{k=0}^{K-1} c^{2k}}}
\end{equation}
satisfies $(\varepsilon,\delta)$ per-instance unlearning for the group request $(D,D_f)$.
\end{corollary}

\begin{proof}
The proof of Theorem~\ref{thm:deterministic} (Appendix~\ref{app:proof-deterministic}) applies
verbatim, with $\Di$ replaced by $D_{-I_f}$ and $s_{i,k}$ by $\Delta_{I_f,k}$:
Proposition~\ref{prop:accounting} only uses the contractivity of the retain-set map and a
deterministic bound on the sensitivity $\|\phi_k(\theta_k)-\phi'(\theta_k)\|$, which is
$\Delta_{I_f,k}$ here. Since $\slearn=0$, condition~\eqref{eq:feas-crit} holds trivially and
\[
\mu_{I_f}
=
\frac{N_{I_f}}
{\sigma_{\mathrm{unlearn}}\sqrt{2\eta\sum_{k=0}^{K-1} c^{2k}}},
\]
which equals $\mu_{\mathrm{target}}$ for the
choice~\eqref{eq:group-deterministic-calibration}. The exact GDP-to-$(\varepsilon,\delta)$
conversion of Appendix~\ref{app:gdp-to-unlearning} concludes the proof.
\end{proof}

\paragraph{Storage and computation.}
Exact calibration for a group $I_f$ requires the group sensitivities
$\Delta_{I_f,k}$ defined in~\eqref{eq:group-sensitivity}, and hence the
corresponding group norms to be computed during training. These quantities
may be stored for a prescribed family of relevant groups, e.g., when several
training samples are associated with the same individual. As in the single-point case,
one scalar per group suffices, maintained through the recursion
$\tilde N_{I_f}\leftarrow c\,\tilde N_{I_f}+\Delta_{I_f,k}$, with
$N_{I_f}=c^K\tilde N_{I_f}$. Forming the summed gradient of a group of $R$ points costs
$O(Rpd)$ per step, so that a family of disjoint groups leaves the training complexity
$O(Tnpd)$ unchanged. For a group $I_f$
that was not stored in advance, the triangle inequality gives
\[
N_{I_f}
\le
\sum_{i\in I_f} N_i.
\]
Therefore, $\sum_{i\in I_f} N_i$ may be used in place of $N_{I_f}$
in~\eqref{eq:group-deterministic-calibration}, yielding a valid, though
potentially looser, certificate using only the individual statistics, at no additional
cost during training.

\subsection{Noisy Training: Ridge Regression}

We now specialize to the multi-output ridge-regression setting of
Theorem~\ref{thm:ridge-gdp-main}. Let
\[
X_f
:=
\begin{bmatrix}
x_{i_1}^\top\\
\vdots\\
x_{i_R}^\top
\end{bmatrix}
\in\mathbb R^{R\times p},
\qquad
Y_f
:=
\begin{bmatrix}
y_{i_1}^\top\\
\vdots\\
y_{i_R}^\top
\end{bmatrix}
\in\mathbb R^{R\times d},
\qquad
A_f:=X_f^\top X_f.
\]
The aggregate gradient contribution of the group is
\begin{equation}
G_{f,k}
:=
X_f^\top(X_f\theta_k-Y_f)
=
A_f\theta_k-X_f^\top Y_f
\in\mathbb R^{p\times d},
\label{eq:group-gradient}
\end{equation}
so that the exact group sensitivity is
\begin{equation}
\Delta_{I_f,k}
=
\eta\|G_{f,k}\|_F.
\label{eq:group-ridge-sensitivity}
\end{equation}
This quantity retains the geometry of the joint contribution of the forgotten points. In
particular, it is generally strictly smaller than
$\eta\sum_{i\in I_f}\|x_i\|_2\|x_i^\top\theta_k-y_i\|_2$, which follows only from the
triangle inequality.

Let $m_k$ and $\Sigma_k$ be the mean and covariance of the iterates given by
Lemma~\ref{lem:gaussian-kronecker-proof}, and define
\[
H_{f,k}:=A_fm_k-X_f^\top Y_f,
\qquad
S_{f,k}:=A_f\Sigma_kA_f.
\]
For $k\ge1$, $\Sigma_k$ is positive definite because $\sigma_{\mathrm{learn}}>0$, so
$S_{f,k}$ has the same range as $A_f$ and rank
$r_k=\operatorname{rank}(A_f)\le\min\{R,p\}$. Let
$\lambda_{1,k},\ldots,\lambda_{r_k,k}>0$ be its nonzero eigenvalues,
$q_{1,k},\ldots,q_{r_k,k}$ associated orthonormal eigenvectors, and
\begin{equation}
\nu_{a,k}
:=
\frac{\|q_{a,k}^\top H_{f,k}\|_2^2}{\lambda_{a,k}},
\qquad
a=1,\ldots,r_k.
\label{eq:group-noncentrality}
\end{equation}

\begin{lemma}[Law of the group sensitivity]
\label{lem:group-generalized-chi-square}
For every $k\ge1$,
\begin{equation}
\frac{\Delta_{I_f,k}^2}{\eta^2}
=
\|G_{f,k}\|_F^2
\overset{d}{=}
\sum_{a=1}^{r_k}
\lambda_{a,k}
Z_{a,k},
\label{eq:group-generalized-chi-square}
\end{equation}
where the $Z_{a,k}\sim\chi'^2_d(\nu_{a,k})$ are independent. Thus $\Delta_{I_f,k}^2/\eta^2$
follows a generalized non-central chi-squared distribution.
\end{lemma}

\begin{proof}
By Lemma~\ref{lem:gaussian-kronecker-proof}, $G_{f,k}=A_f\theta_k-X_f^\top Y_f$ is an affine
image of $\theta_k$, so
$\operatorname{vec}(G_{f,k})\sim\mathcal N\bigl(\operatorname{vec}(H_{f,k}),\,I_d\otimes S_{f,k}\bigr)$:
its $d$ columns are independent Gaussian vectors with common covariance $S_{f,k}$. These
columns lie in $\operatorname{Range}(A_f)=\operatorname{Range}(S_{f,k})$, so
$\|G_{f,k}\|_F^2=\sum_{a=1}^{r_k}\|q_{a,k}^\top G_{f,k}\|_2^2$. For each~$a$, the $d$~entries
of $q_{a,k}^\top G_{f,k}$ are independent Gaussian variables with variance $\lambda_{a,k}$
and mean vector $q_{a,k}^\top H_{f,k}$, hence
$\|q_{a,k}^\top G_{f,k}\|_2^2/\lambda_{a,k}\sim\chi'^2_d(\nu_{a,k})$, and distinct
eigendirections are independent.
\end{proof}

We write $q_{I_f,k}(\cdot)$ for the quantile function of the
distribution~\eqref{eq:group-generalized-chi-square}.

\begin{corollary}[Group extension of Theorem~\ref{thm:ridge-gdp-main}: ridge regression]
\label{cor:group-ridge}
Consider multi-output ridge regression trained with learning noise
$\sigma_{\mathrm{learn}}>0$ for $T$ iterations; then the group version of
Assumption~\ref{ass:contraction} holds with contraction factor $c=1-\eta \lambda\le 1$. Let
$D_f$ be the group to unlearn, $(\eps,\del)$ the target guarantee, $K\ge1$ the unlearning
horizon, and choose $(\dm,\ds)$ with $\dm+\ds=\del$.

\paragraph{(i) Bounding $\Delta_{I_f,k}$ with high probability.}
Set
\[
s_{I_f,0}^{\ds}
:=
\eta\bigl\|X_f^\top(X_f\theta_0-Y_f)\bigr\|_F
\]
and, for $k\ge1$,
\[
s_{I_f,k}^{\ds}
:=
\eta\sqrt{q_{I_f,k}\!\left(1-\tfrac{\ds}{T}\right)}.
\]
Then $\;\Prob\bigl(\forall k<T,\;\Delta_{I_f,k}\le s_{I_f,k}^{\ds}\bigr)\ge1-\ds$.

\paragraph{(ii) Calibrating the unlearning noise.}
Let $N_{I_f,j}$ and $N_{I_f}$ be the sums~\eqref{eq:prefix} formed from the sequence
$\{s_{I_f,k}^{\ds}\}_{k<T}$, and let
$\mu_{\mathrm{target}}:=\sup\{\mu>0:\epsgdp(\mu,\dm)\le\eps\}$. Suppose the learning noise
satisfies
\begin{equation}
\label{eq:group-ridge-learning-noise-condition}
\sigma_{\mathrm{learn}}^2
\le
\min_{1\le j\le T}
\frac{N_{I_f}\,N_{I_f,j}}
{2\eta\,\mu_{\mathrm{target}}^{2}\sum_{k<j}c^{2(T+K-1-k)}}.
\end{equation}
Then the group unlearning procedure, run on $D_{-I_f}$ with noise level
\[
\sunl
:=
\inf\Bigl\{\sigma>0:\ N_{I_f}/\sqrt{V(\sigma)}\le\mu_{\mathrm{target}}\Bigr\},
\]
satisfies $(\eps,\del)$ per-instance unlearning for the group deletion request $(D,D_f)$.
\end{corollary}

\begin{proof}
(i) The bound at $k=0$ is exact, since the initialization is deterministic. For $k\ge1$,
Lemma~\ref{lem:group-generalized-chi-square} and the definition of the quantile give
$\Prob(\Delta_{I_f,k}\le s_{I_f,k}^{\ds})\ge1-\ds/T$, and a union bound over
$k=1,\dots,T-1$ yields the claim.

(ii) The argument is that of Theorem~\ref{thm:ridge-gdp-main} and requires no modification.
The accounting of Appendix~\ref{app:accounting} and the calibration of
Appendix~\ref{app:ridge-calibration} involve the deletion request only through a
deterministic sequence bounding the sensitivities, and apply verbatim with $\Di$ replaced by
$D_{-I_f}$ and $s_{i,k}$ by $s_{I_f,k}^{\ds}$: under
condition~\eqref{eq:group-ridge-learning-noise-condition}, the feasibility
condition~\eqref{eq:feas-crit} holds and
$\mu_{I_f}=N_{I_f}/\sqrt{V(\sunl)}\le\mu_{\mathrm{target}}$. The virtual-mechanism
construction of Appendix~\ref{app:virtual-mechanism} applies as well, the clipping now acting
on the aggregate contribution $X_f^\top(X_f\theta_k-Y_f)$ of the forgotten group: it only
uses that $\{s_{I_f,k}^{\ds}\}$ bounds the sensitivities on an event of probability at least
$1-\ds$, and that the clipping is never invoked on the retain side, where the group is
absent. It yields a total-variation gap of at most $\ds$ on the deletion side and an exact
identity on the retain side, and reduces the claim to $(\eps,\dm)$ per-instance unlearning
for the virtual mechanism, which the calibration provides. Combining the two bounds gives
$\dm+\ds=\del$, as in Appendix~\ref{app:proof-thm-ridge}.
\end{proof}

\paragraph{Storage and computation.}
In contrast with the deterministic case, nothing has to be stored during training: the
bounds $s_{I_f,k}^{\ds}$ only involve the data, through $X_f$, $Y_f$ and the closed-form
moments $m_k$ and $\Sigma_k$ of Lemma~\ref{lem:gaussian-kronecker-proof}, and are computed
when the request arrives. The group therefore need not be known at training time.

\paragraph{Reduction to the single-point case.}
When $R=1$, $X_f=x_i^\top$ and
\[
A_f=x_ix_i^\top.
\]
The matrix $S_{f,k}$ has rank one, with its unique nonzero
eigenvalue equal to
\[
\lambda_{1,k}
=
v_{i,k}\|x_i\|_2^2,
\qquad
v_{i,k}=x_i^\top\Sigma_kx_i.
\]
Moreover,
\[
\nu_{1,k}
=
\frac{\|u_{i,k}\|_2^2}{v_{i,k}},
\qquad
u_{i,k}=x_i^\top m_k-y_i.
\]
Hence \eqref{eq:group-generalized-chi-square} reduces to
\[
\frac{\Delta_{i,k}^2}{\eta^2}
\overset{d}{=}
v_{i,k}\|x_i\|_2^2
\chi'^2_d
\left(
\frac{\|u_{i,k}\|_2^2}{v_{i,k}}
\right),
\]
which is exactly the distribution used in
Theorem~\ref{thm:ridge-gdp-main}. Thus the group formula recovers the single-point analysis when $R=1$.

\section{Comparison with Exact Unlearning}
\label{app:exact-ridge}

Exact unlearning is possible in ridge regression. Letting $A:=X^\top X+\lambda I_p$ and
$H:=XA^{-1}X^\top$, the leave-one-out (LOO) prediction follows in closed form from the
Sherman--Morrison formula \citep{golub1979generalized,hastie2009elements}:
\[
\hat y^{(-i)}_i=\hat y_i-\frac{\hat y_i-y_i}{1-H_{ii}},
\qquad \hat y=Hy .
\]
This formula, however, characterizes only the minimizer of the retain objective. Like the
Newton step of \citet{guo2020certified} (Appendix~\ref{app:newton-baseline}), it targets the
limit $T\to\infty$ of gradient-based training, not the output of $T$ iterations, and it
does not describe the output of cyclic mini-batch training either. It is moreover not
always defined: when $\lambda=0$ and $p\ge n$, the matrix $A$ is singular and, if $X$ has
full row rank, $H=I_n$, so that the denominator $1-H_{ii}$ vanishes. Our certificates cover
all three situations: they hold at any finite horizon $T$, extend to cyclic mini-batch
training (Appendix~\ref{app:mini-batch}), and apply with $\lambda=0$, including when
$p\ge n$, with contraction factor $c=1$ (Theorems~\ref{thm:deterministic}
and~\ref{thm:ridge-gdp-main}).

From a computational standpoint, the formula requires inverting $A$, whose cost grows as
$O(p^3)$ and whose numerical stability degrades in weakly regularized regimes. Our approach
also stores the normal matrix, but never inverts it: the certificate is obtained through
iterative updates.

In practice, using the LOO formula at finite time introduces a discrepancy with exact unlearning that
must be safeguarded by a DP mechanism to ensure valid privacy
guarantees. This is exactly the method introduced by
\citet{guo2020certified}, which uses a generalization of the LOO
formula in the non-quadratic case, based on a (second-order) Newton correction. We
compare it with PILU in Section~\ref{sec:experiments-linear}
(Figure~\ref{fig:nyuv2-ours-vs-guo}).

\section{Extension to Cyclic Mini-Batch}
\label{app:mini-batch}
 
We now extend the construction of Section~\ref{sec:setting} and both certificates,
Theorem~\ref{thm:deterministic} (deterministic training) and
Theorem~\ref{thm:ridge-gdp-main} (noisy training, ridge regression), to deterministic
cyclic mini-batch learning and unlearning. The motivation is primarily practical: for
larger models, full-batch updates may become computationally impractical, whereas cyclic
mini-batch training provides a scalable alternative while preserving the same targeted
post-hoc unlearning principle; it also yields non-vacuous certificates when the unlearning
budget is smaller than one full epoch.
 
As in Appendix~\ref{app:group-unlearning}, the deletion request enters the analysis only
through the per-instance sensitivity sequence, so the accounting of
Appendix~\ref{app:accounting} can be reused once that sequence is supplied. The
observation specific to the mini-batch case is that the deleted point can influence the
trajectory only at the updates that process \emph{its own} batch, that is, once per pass
through the data: the sensitivity sequence is therefore \emph{sparse}, and it starts with a
run of zeros, which is the single place where the accounting has to be adapted.
 
\paragraph{Cyclic mini-batch learning and unlearning.}
Let $n$ be the number of training samples and $b$ a batch size. We partition the dataset
into a deterministic sequence of contiguous batches
\[
B_1,\dots,B_{n_b},
\qquad
n_b=\Bigl\lceil \frac{n}{b}\Bigr\rceil,
\]
kept in this same order throughout learning and unlearning, and we visit them cyclically:
the update at step $k$ uses the batch of index
\[
t_k=1+(k\bmod n_b),
\qquad k=0,\dots,T+K-1 ,
\]
where the learning horizon $T$ and the unlearning horizon $K$ are both counted directly at
the update level. Each batch carries its own objective
\[
f_{B_t}(\theta)
:=
\sum_{j\in B_t}\ell(\theta;x_j,y_j)+\rho(\theta),
\]
where $\rho$ is a (possibly zero) regularizer applied at every batch visit; it is
$\rho(\theta)=\tfrac\lambda2\|\theta\|_F^2$ in the ridge instantiation of
Appendix~\ref{app:minibatch-ridge}. One batch visit performs a single Langevin-type update
\begin{equation}
\label{eq:minibatch-update}
\theta_{k+1}
=
\theta_k-\eta\nabla f_{B_{t_k}}(\theta_k)
+\sqrt{2\eta}\,\sigma_k\,\xi_k,
\qquad \xi_k\sim\mathcal N(0,I),
\end{equation}
with $\sigma_k=\slearn$ during learning ($k<T$) and $\sigma_k=\sunl$ during unlearning
($k\ge T$).
 
Given the point $i$ to be removed, let $t_i\in\{1,\dots,n_b\}$ be the unique batch index such
that $i\in B_{t_i}$. During the $K$ unlearning updates, whenever the batch containing the
target point is visited, that point is removed from the update, that is, $B_{t_i}$ is
replaced by $B_{t_i}\setminus\{i\}$, while all other batches are left unchanged. As in
Section~\ref{sec:gdp-accounting}, the trajectory $(\theta_k)$ so obtained is compared to the
fictitious retraining trajectory $(\theta'_k)$, which starts from the same initialization
$\theta'_0=\theta_0$ and runs the same schedule on $\Di$, with the same partition and the
point deleted from its batch. The two trajectories are therefore aligned batch by batch and
differ only through the presence of $(x_i,y_i)$ in $B_{t_i}$.
 
\paragraph{Step size and contraction.}
To use a single step size across the whole cyclic procedure we adopt a conservative choice
based on the largest batch smoothness: writing $L_t$ for the smoothness constant of
$f_{B_t}$ and $L_{\max}=\max_{1\le t\le n_b}L_t$, we take $\eta=1/(\lambda_{\max}(X^\top X)+\lambda)$, 
which satisfies $\eta\le1/L_{\max}$ and is therefore valid for every batch in the cycle. 
Assumption~\ref{ass:contraction} is understood batchwise:
every map $x\mapsto x-\eta\nabla f_{B_t}(x)$, together with its leave-one-out counterpart
$x\mapsto x-\eta\nabla f_{B_{t_i}\setminus\{i\}}(x)$, is $c$-contractive with the same factor
$c\le1$.
 
\paragraph{Sparse per-instance sensitivities.}
The per-instance sensitivity~\eqref{eq:sensitivity} becomes
\begin{equation}
\label{eq:minibatch-sensitivity}
\Delta_{i,k}
=
\bigl\|\eta\bigl(
\nabla f_{B_{t_k}\setminus\{i\}}(\theta_k)-\nabla f_{B_{t_k}}(\theta_k)
\bigr)\bigr\|
=
\eta\bigl\|\nabla_\theta\ell(\theta_k;x_i,y_i)\bigr\|\,
\mathbf 1_{\{t_k=t_i\}},
\qquad k<T,
\end{equation}
and $\Delta_{i,k}=0$ for $k\ge T$, since both trajectories then follow the same update map.
Let
\[
\mathcal I_i:=\{k<T:\ t_k=t_i\},
\qquad
V_i:=|\mathcal I_i|,
\qquad
\tau_i:=\min\mathcal I_i=t_i-1 ,
\]
so that $V_i\in\{\lfloor T/n_b\rfloor,\lceil T/n_b\rceil\}$ depending on the position of
$B_{t_i}$ in the cycle, and $\tau_i$ is the first update at which the deleted point
influences the dynamics.
 
\paragraph{Accounting with a sparse sensitivity sequence.}
Let $\{s_{i,k}\}_{k<T}$ be any deterministic sequence such that $\Delta_{i,k}\le s_{i,k}$ and
$s_{i,k}=0$ for $k\notin\mathcal I_i$. The shifted-interpolation analysis of
Appendix~\ref{app:accounting} applies verbatim: Lemma~\ref{lem:shifted-step} is used one
step at a time, so the update maps may change with $k$, as long as the retain-side ones are
$c$-contractive. The only modification is caused by the leading zeros of the sequence.
Recall that this analysis rests on the shift-allocation
problem~\eqref{eq:shift-allocation-qp}, whose variables are the shift increments
$a_{k+1}\ge0$ spent at each step and the residual displacements $z_k\ge0$ they leave, tied by
$z_{k+1}=cz_k+s_{i,k}\mathbf 1_{\{k<T\}}-a_{k+1}$ and by $z_0=z_{T+K}=0$. Here
$z_0=0$ together with $s_{i,k}=0$ for $k<\tau_i$ forces recursively
\[
a_{k+1}=0,
\qquad
z_{k+1}=0,
\qquad k<\tau_i :
\]
the interpolation cannot spend any privacy increment before the deleted point has first
influenced the trajectory. (After $\tau_i$ this is no longer the case: even at a step with
$s_{i,k}=0$, one may have $z_k>0$, so the admissible upper bound is $cz_k$ and $a_{k+1}$ need
not vanish.) The problem is therefore the same quadratic program, posed over the index set
$\{\tau_i,\dots,T+K-1\}$ with initial condition $z_{\tau_i}=0$, and the prefix
sums~\eqref{eq:prefix} must be started at $\tau_i$ as well:
\begin{equation}
\label{eq:minibatch-prefix}
N_{i,j}:=\!\!\sum_{\tau_i\le k<\min(j,T)}\!\! c^{\,T+K-1-k}s_{i,k},
\qquad
V_{i,j}(\sigma):=\!\!\sum_{\tau_i\le k<j}\!\! 2\eta\sigma_k^2c^{\,2(T+K-1-k)},
\qquad j=\tau_i+1,\dots,T+K,
\end{equation}
with $N_i:=N_{i,T+K}$ and $V_i(\sigma):=V_{i,T+K}(\sigma)$. The restriction is vacuous for
$N_{i,j}$, since $s_{i,k}=0$ for $k<\tau_i$, but not for $V_{i,j}$: the learning noise injected
before $\tau_i$ cannot be used to hide the deleted point, and must be discarded from the
accounting. With these sums, Lemma~\ref{lem:feasibility} and the Cauchy--Schwarz optimum of
Appendix~\ref{app:qp} give
\[
\mathcal T(\theta_{T+K},\theta'_{T+K})\ \ge\ G(\mu_i),
\qquad
\mu_i=\frac{N_i}{\sqrt{V_i(\sunl)}},
\]
whenever the feasibility condition
\begin{equation}
\label{eq:minibatch-feasibility}
\frac{V_{i,j}(\sunl)}{V_i(\sunl)}
\le
\frac{N_{i,j}}{N_i},
\qquad j=\tau_i+1,\dots,T,
\end{equation}
holds. Had the sums not been restricted to $k\ge\tau_i$, condition~\eqref{eq:minibatch-feasibility}
at $j\le\tau_i$ would read $V_{i,j}(\sunl)\le0$ with $N_{i,j}=0$, and would force
$\slearn=0$; the restriction removes this spurious obstruction, the binding constraint
becoming the one at $j=\tau_i+1$, for which $N_{i,\tau_i+1}=c^{T+K-1-\tau_i}s_{i,\tau_i}>0$.
 
The two subsections below only supply the sequence $\{s_{i,k}\}$: exactly, when learning is
deterministic (Appendix~\ref{app:minibatch-deterministic}), and with high probability for
ridge regression when learning is noisy (Appendix~\ref{app:minibatch-ridge}).
 
\subsection{Deterministic Training}
\label{app:minibatch-deterministic}
 
When $\slearn=0$, the learning trajectory is deterministic and the sparse
sensitivities~\eqref{eq:minibatch-sensitivity} are available exactly, at the $V_i$ updates
that process the batch of point $i$.
 
\begin{corollary}[Mini-batch extension of Theorem~\ref{thm:deterministic}: deterministic training]
\label{cor:minibatch-deterministic}
Assume the batchwise version of Assumption~\ref{ass:contraction} with factor $c\le1$, and run
the learning phase with $\slearn=0$, so that the iterates $\theta_0,\dots,\theta_T$ generated
by~\eqref{eq:minibatch-update} are deterministic. Define
\[
N_i
:=
\sum_{k\in\mathcal I_i}c^{\,T+K-1-k}\Delta_{i,k},
\qquad
\Delta_{i,k}=\eta\bigl\|\nabla_\theta\ell(\theta_k;x_i,y_i)\bigr\| ,
\]
and, for a target $(\eps,\del)$, let
$\mu_{\mathrm{target}}:=\sup\{\mu>0:\epsgdp(\mu,\del)\le\eps\}$. Then the mini-batch
unlearning procedure run for $K$ updates with noise level
\begin{equation}
\label{eq:minibatch-deterministic-calibration}
\sunl
=
\frac{N_i}{\mu_{\mathrm{target}}\sqrt{2\eta\sum_{k=0}^{K-1}c^{2k}}}
\end{equation}
satisfies $(\eps,\del)$ per-instance unlearning for the request $(D,\{(x_i,y_i)\})$.
\end{corollary}
 
\begin{proof}
We apply the sparse-sequence accounting above with $s_{i,k}:=\Delta_{i,k}$, which is
deterministic since $\slearn=0$ and vanishes outside $\mathcal I_i$. No noise is injected
before $T$, so $V_{i,j}(\sunl)=0$ for $j\le T$: the feasibility
condition~\eqref{eq:minibatch-feasibility} holds trivially, the leading-zeros discussion
being immaterial here, and $V_i(\sunl)=2\eta\sunl^2\sum_{k=0}^{K-1}c^{2k}$. Hence
\[
\mu_i
=
\frac{N_i}{\sunl\sqrt{2\eta\sum_{k=0}^{K-1}c^{2k}}} ,
\]
which equals $\mu_{\mathrm{target}}$ for the
choice~\eqref{eq:minibatch-deterministic-calibration}, and the exact GDP-to-$(\eps,\del)$
conversion of Appendix~\ref{app:gdp-to-unlearning} concludes.
\end{proof}

At each update, the norms
$\|\nabla_\theta\ell(\theta_k;x_j,y_j)\|$ of the $b$ per-sample gradients of the current
batch are formed while these gradients are accumulated into the batch gradient, and each is
folded into a scalar accumulator attached to the corresponding point. Only the points of the
current batch are touched: keeping, for each point $j$, the index $l_j$ of the last update
that processed its batch and the value
$A_j=\sum_{k'\le l_j,\ k'\in\mathcal I_j}c^{\,l_j-k'}\Delta_{j,k'}$, the update at step $k$ is
\[
A_j\leftarrow c^{\,k-l_j}A_j+\Delta_{j,k},
\qquad
l_j\leftarrow k,
\qquad j\in B_{t_k},
\]
and, at deletion time, $N_j=c^{\,T+K-1-l_j}A_j$. This costs $O(1)$ per point and
per visit of its batch, hence $O(Tb)$ scalar operations and $O(n)$ memory over the whole run,
which leaves the training complexity $O(Tbpd)$ unchanged. A deletion request then costs $O(Kbpd)$, against $O(Tbpd)$ for retraining under the
same schedule, i.e., again a factor $T/K$, with each update cheaper by a factor $n/b$ than
in the full-batch procedure.
 
\subsection{Noisy Training: Ridge Regression}
\label{app:minibatch-ridge}
 
We now specialize to the multi-output ridge-regression setting of
Theorem~\ref{thm:ridge-gdp-main}, i.e., $\rho(\theta)=\tfrac\lambda2\|\theta\|_F^2$. Writing
$X_{B_t}\in\R^{|B_t|\times p}$ and $Y_{B_t}\in\R^{|B_t|\times d}$ for the rows of $X$ and $Y$
indexed by the batch $B_t$,
\[
\nabla f_{B_t}(\theta)
=
X_{B_t}^\top\!\bigl(X_{B_t}\theta-Y_{B_t}\bigr)+\lambda\theta ,
\qquad
\Delta_{i,k}
=
\eta\|x_i\|_2\|x_i^\top\theta_k-y_i\|_2\,\mathbf 1_{\{t_k=t_i\}} .
\]
The regularizer is applied at each of the $n_b$ batch visits, so that one full cycle applies the
gradient of $\sum_{j}\tfrac12\|x_j^\top\theta-y_j\|_2^2+\tfrac{n_b\lambda}{2}\|\theta\|_F^2$;
$\lambda$ is thus the per-batch regularization strength, and it is this value that enters the
contraction factor. Each batch objective is $\lambda$-strongly convex and
$L_t=\|X_{B_t}\|_{\mathrm{op}}^2+\lambda$-smooth, so with $\eta=1/L_{\max}$ the batchwise
version of Assumption~\ref{ass:contraction} holds with $c=1-\eta\lambda\le1$.
 
The only difference with the full-batch case is that the linear dynamics are now
\emph{time-in\-ho\-mo\-ge\-neous}, since the batch changes at each update. At learning step
$k\in\{0,\dots,T-1\}$, define
\[
A_k=I_p-\eta\bigl(X_{B_{t_k}}^\top X_{B_{t_k}}+\lambda I_p\bigr),
\qquad
C_k=X_{B_{t_k}}^\top Y_{B_{t_k}} .
\]
The iterates remain Gaussian, $\operatorname{vec}(\theta_k)\sim\mathcal N\bigl(
\operatorname{vec}(m_k),I_d\otimes\Sigma_k\bigr)$, with mean $m_k$ and covariance $\Sigma_k$
propagated through this sequence of batch-dependent linear maps,
\begin{equation}
\label{eq:minibatch-moments}
m_{k+1}=A_km_k+\eta C_k,
\qquad
\Sigma_{k+1}=A_k\Sigma_kA_k^\top+2\eta\slearn^2I_p,
\qquad
m_0=\theta_0,\quad \Sigma_0=0,
\end{equation}
which replaces the closed forms of Theorem~\ref{thm:ridge-gdp-main}, recovered when the cycle
consists of a single batch and $A_k$ no longer depends on $k$. The residual of the deleted
point is therefore Gaussian as well,
\[
r_{i,k}=x_i^\top\theta_k-y_i\sim\mathcal N(u_{i,k},v_{i,k}I_d),
\qquad
u_{i,k}=x_i^\top m_k-y_i,
\qquad
v_{i,k}=x_i^\top\Sigma_kx_i ,
\]
so that $\|r_{i,k}\|_2^2/v_{i,k}\sim\chi'^2_d\bigl(\|u_{i,k}\|_2^2/v_{i,k}\bigr)$; we write
$q_{i,k}(\cdot)$ for the quantile function of this non-central chi-squared distribution. Since
only the $V_i$ steps of $\mathcal I_i$ carry a nonzero sensitivity, the union bound is taken
over $\mathcal I_i$ rather than over all $T$ learning steps, and spends $\ds/V_i$ per visited
step instead of $\ds/T$.
 
\begin{corollary}[Mini-batch extension of Theorem~\ref{thm:ridge-gdp-main}: ridge regression]
\label{cor:minibatch-ridge}
Consider multi-output ridge regression trained by the cyclic mini-batch
procedure~\eqref{eq:minibatch-update} with learning noise $\slearn>0$ for $T$ updates, so that
$c=1-\eta\lambda\le1$. Let $(x_i,y_i)$ be the point to unlearn, $(\eps,\del)$ the target
guarantee, $K\ge1$ the unlearning horizon, and choose $(\dm,\ds)$ with $\dm+\ds=\del$.
 
\paragraph{(i) Bounding $\Delta_{i,k}$ with high probability.}
Let $u_{i,k}=x_i^\top m_k-y_i$ and $v_{i,k}=x_i^\top\Sigma_kx_i$ be formed from the
recursion~\eqref{eq:minibatch-moments}, and set
\[
s_{i,k}^{\ds}
:=
\eta\,\|x_i\|_2\sqrt{v_{i,k}\,q_{i,k}\!\left(1-\tfrac{\ds}{V_i}\right)}
\quad\text{for }k\in\mathcal I_i,\ k\ge1,
\qquad
s_{i,k}^{\ds}:=0
\quad\text{for }k\notin\mathcal I_i,
\]
with $s_{i,0}^{\ds}:=\eta\|x_i\|_2\|x_i^\top\theta_0-y_i\|_2$ if $\tau_i=0$, the
initialization being deterministic. Then
$\Prob\bigl(\forall k<T,\ \Delta_{i,k}\le s_{i,k}^{\ds}\bigr)\ge1-\ds$.
 
\paragraph{(ii) Calibrating the unlearning noise.}
Let $N_{i,j}$, $N_i$, $V_i$ be the sums~\eqref{eq:minibatch-prefix} formed from
$\{s_{i,k}^{\ds}\}_{k<T}$, and let
$\mu_{\mathrm{target}}:=\sup\{\mu>0:\epsgdp(\mu,\dm)\le\eps\}$. Suppose the learning noise
satisfies
\begin{equation}
\label{eq:minibatch-sufficient}
\slearn^2
\le
\min_{\tau_i<j\le T}
\frac{N_i\,N_{i,j}}
{2\eta\,\mu_{\mathrm{target}}^2\sum_{\tau_i\le k<j}c^{\,2(T+K-1-k)}}\,,
\end{equation}
checkable in $O(T)$ time. Then the mini-batch unlearning procedure run for $K$ updates with
noise level
\[
\sunl:=\inf\Bigl\{\sigma>0:\ N_i/\sqrt{V_i(\sigma)}\le\mu_{\mathrm{target}}\Bigr\}
\]
satisfies $(\eps,\del)$ per-instance unlearning for the request $(D,\{(x_i,y_i)\})$.
\end{corollary}
 
\begin{proof}
(i) By~\eqref{eq:minibatch-moments} the iterate $\theta_k$ is Gaussian, being an affine image
of the learning noises, so the residual law and the non-central chi-squared law stated above
hold; the definition of the quantile gives
$\Prob(\Delta_{i,k}\le s_{i,k}^{\ds})\ge1-\ds/V_i$ at every $k\in\mathcal I_i$ with $k\ge1$,
while the bound is exact at $k=0$ and the sensitivity vanishes outside $\mathcal I_i$. A union
bound over the at most $V_i$ visited steps yields the claim.
 
(ii) The argument is that of Theorem~\ref{thm:ridge-gdp-main} and requires no modification.
The virtual-mechanism construction of Appendix~\ref{app:proof-thm-ridge} only uses that
$\{s_{i,k}^{\ds}\}$ is a deterministic sequence bounding the sensitivities on an event of
probability at least $1-\ds$, and that the clipping of the deleted point's contribution is
never invoked on the retain side, where the point is absent; both remain true here. It
therefore yields a total-variation gap of at most $\ds$ on the deletion side, an exact
identity on the retain side, and reduces the claim to
$(\eps,\dm)$ per-instance unlearning for the virtual mechanism, whose sensitivities are
deterministic. For the latter, the sparse-sequence accounting recalled
in~\eqref{eq:minibatch-prefix}--\eqref{eq:minibatch-feasibility} applies: the definition of
$\sunl$ gives $V_i(\sunl)\ge N_i^2/\mu_{\mathrm{target}}^2$, while for $\tau_i<j\le T$,
\[
V_{i,j}(\sunl)
=2\eta\slearn^2\!\!\sum_{\tau_i\le k<j}\!\!c^{\,2(T+K-1-k)}
\le
\frac{N_iN_{i,j}}{\mu_{\mathrm{target}}^2}
\]
by~\eqref{eq:minibatch-sufficient}, so that the feasibility
condition~\eqref{eq:minibatch-feasibility} holds and
$\mu_i=N_i/\sqrt{V_i(\sunl)}\le\mu_{\mathrm{target}}$. Combining the two bounds gives
$\dm+\ds=\del$, as in Appendix~\ref{app:proof-thm-ridge}.
\end{proof}

Nothing is stored along the training trajectory: the bounds $s_{i,k}^{\ds}$ depend on the data
and on the schedule only, through the recursion~\eqref{eq:minibatch-moments}. In the
full-batch case $M$ is fixed, so one diagonalization of $X^\top X$ turns $u_{i,k}$ and
$v_{i,k}$ into scalar recursions; here the maps $A_k$ change at every update, so
\eqref{eq:minibatch-moments} must be propagated, at $O(bpd)$ per step for the mean and
$O(bp^2)$ for the covariance. This pass is the same for every point, hence run once and
shared by all requests: as it advances, the scalars $\|u_{j,k}\|_2^2$ and
$v_{j,k}=x_j^\top\Sigma_kx_j$ are recorded for the $b$ points of the current batch, at a
further $O(bp^2)$ per step, and the covariance is discarded. A single request then costs
$O(V_i)$ quantile evaluations and $O(T)$ operations to calibrate $\sunl$, so it is again
dominated by the $K$ unlearning updates, $O(Kbpd)$ against $O(Tbpd)$ for retraining.

\paragraph{Experiments.}
Figure~\ref{fig:nyuv2-minibatch-tradeoff} reports the privacy--utility trade-off of the
cyclic mini-batch extension for $b\in\{256,512,1024,2048\}$ and $K=80$, with per-batch
regularization $\lambda=0.1\,n/n_b$, so that one full cycle applies the ridge objective of
Appendix~\ref{app:nyuv2-implementation}. For each $b$, $T$ is chosen so that the learned model 
reaches approximately the training RMSE of the full-batch model ($\approx0.79$).
The per-instance heterogeneity persists for every batch size, and larger batches yield
better curves for small $\eps$.
Table~\ref{tab:nyuv2-minibatch-cost} reports the numbers of processed examples: the
mini-batch runs process about five times fewer examples than the $800$ full-batch steps,
which reflects their shorter training, while the per-request unlearning
cost $bK$ grows with $b$, from half a pass over the data ($b=256$) to four passes
($b=2048$).

\begin{figure}[!htbp]
    \centering
    \begin{subfigure}[t]{0.49\linewidth}
        \centering
      \includegraphics[width=\linewidth]{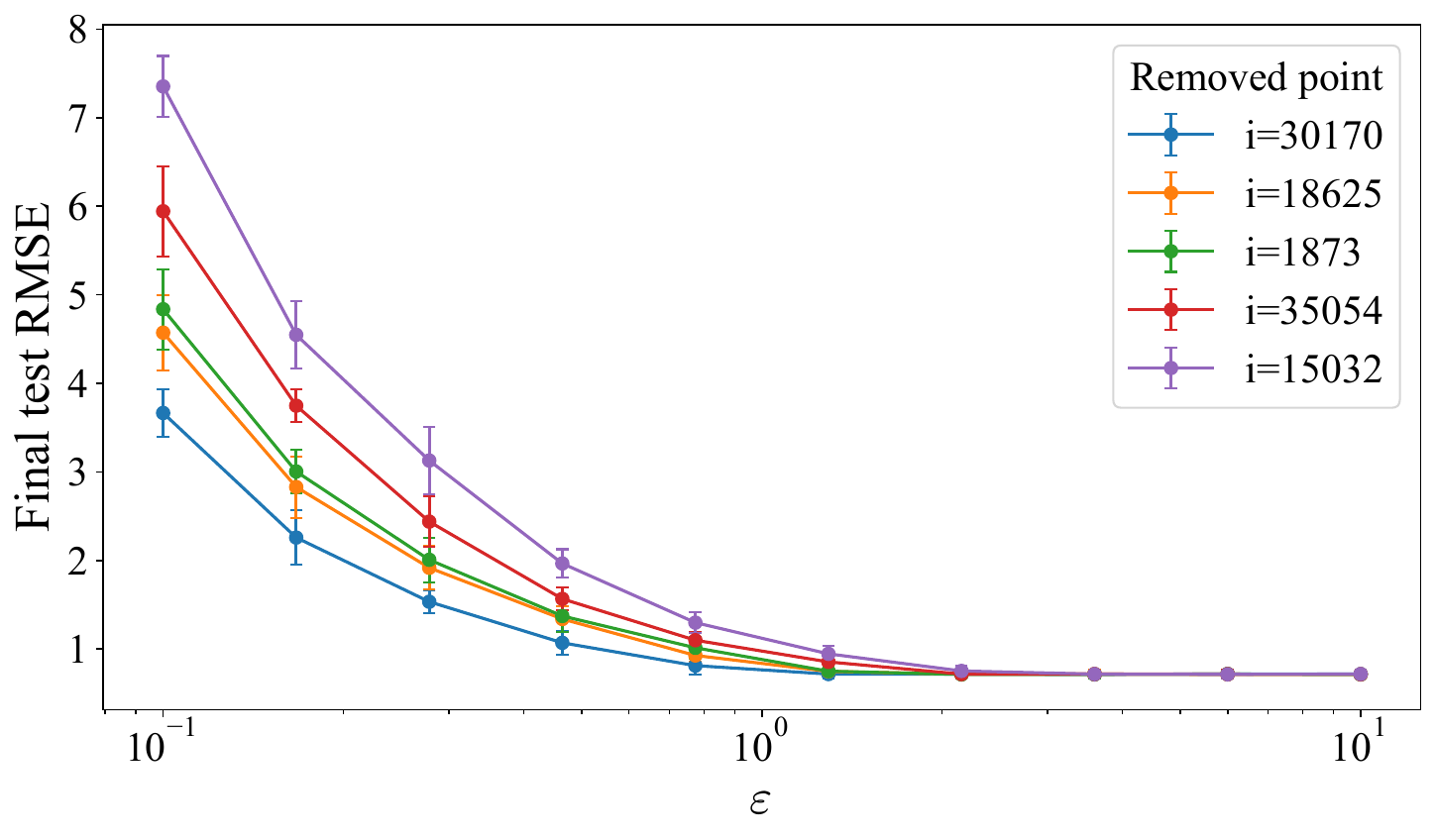}
        \caption{$b=256$ ($150$ batches), \\$T = 24\,900$, $K = 80$}
        \label{fig:nyuv2-minibatch-tradeoff-b256}
    \end{subfigure}
   \hfill
    \begin{subfigure}[t]{0.49\linewidth}
        \centering
        \includegraphics[width=\linewidth]{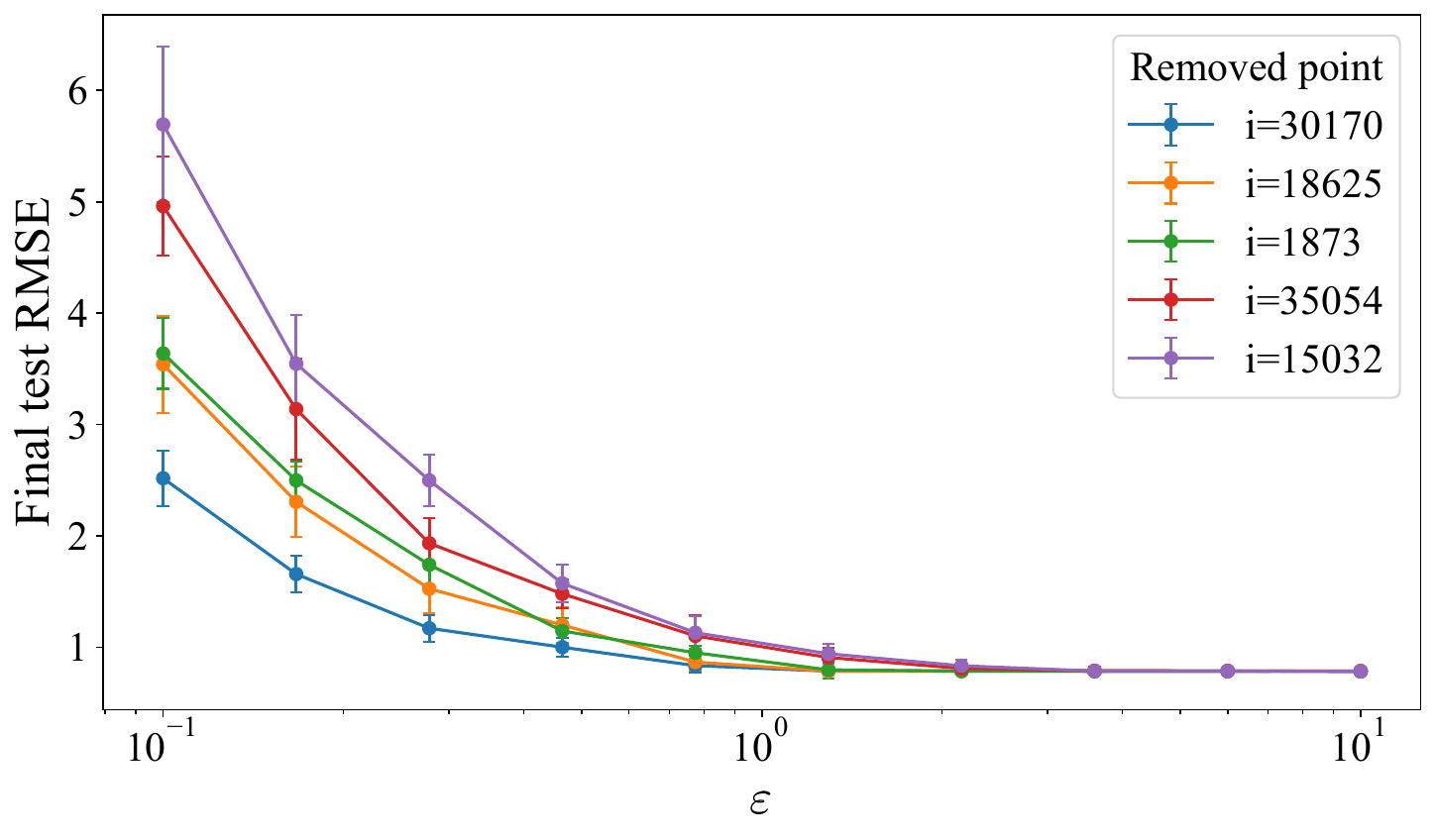}
       \caption{$b=512$ ($75$ batches), \\$T = 11\,325$, $K = 80$}
        \label{fig:nyuv2-minibatch-tradeoff-b512}
    \end{subfigure}
   \hfill
    \begin{subfigure}[t]{0.49\linewidth}
        \centering
      \includegraphics[width=\linewidth]{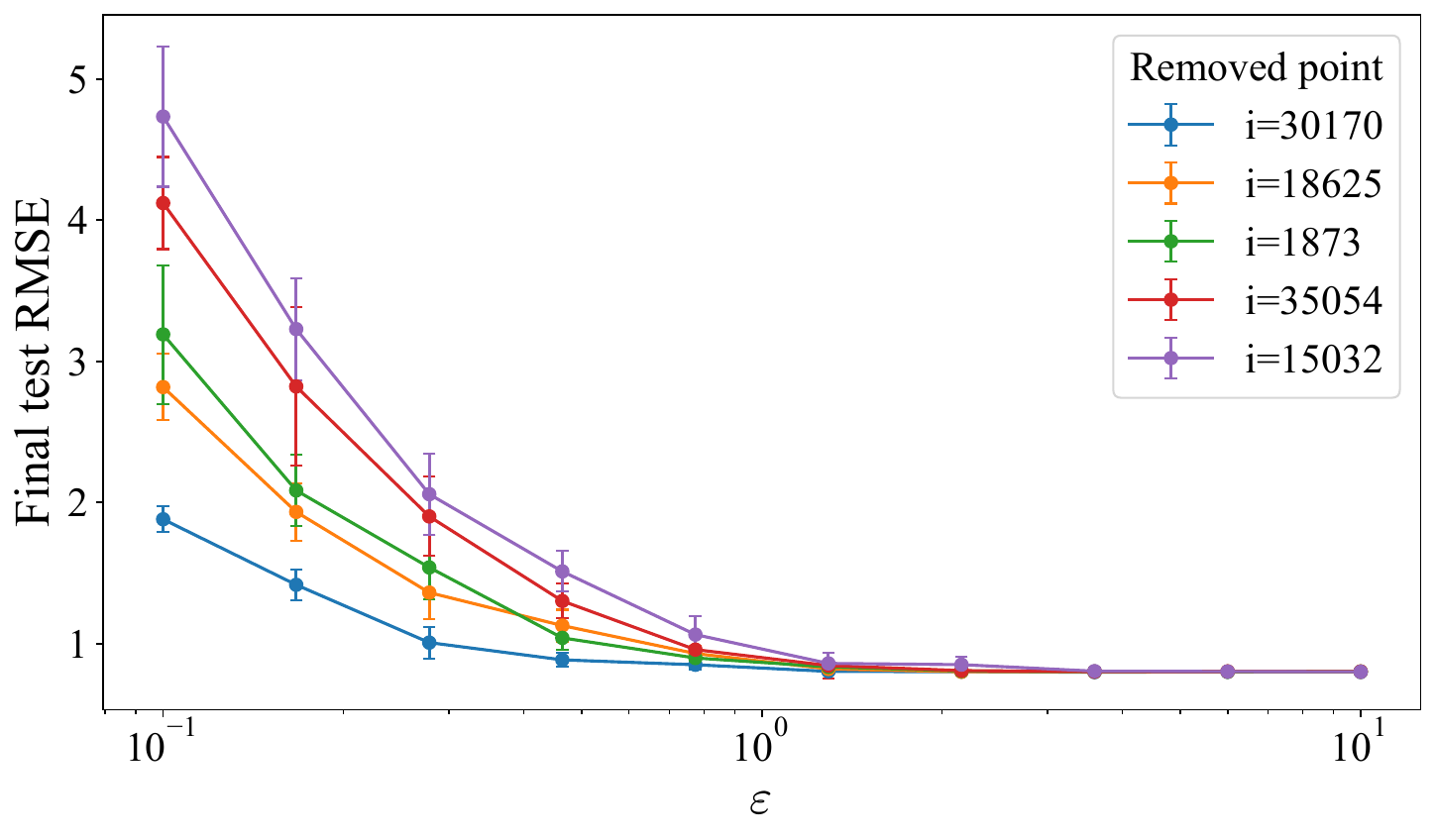}
        \caption{$b=1024$ ($38$ batches), \\$T= 5548$, $K = 80$}
        \label{fig:nyuv2-minibatch-tradeoff-b1024}
    \end{subfigure}
   \hfill
    \begin{subfigure}[t]{0.49\linewidth}
        \centering
        \includegraphics[width=\linewidth]{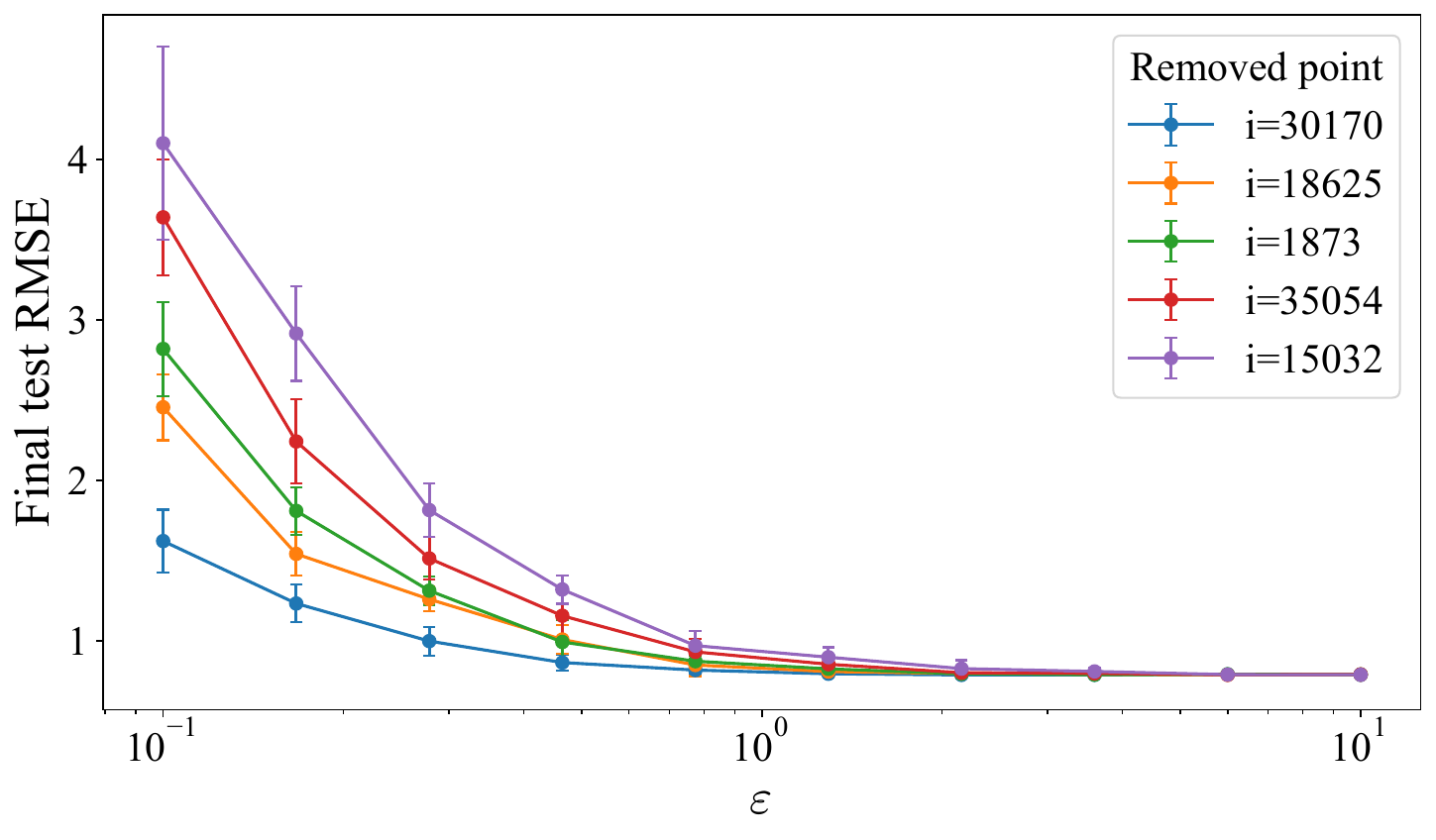}
        \caption{$b=2048$ ($19$ batches), \\$T = 2717$, $K = 80$}
        \label{fig:nyuv2-minibatch-tradeoff-b2048}
    \end{subfigure}

    \caption{Per-instance privacy--utility trade-off for the cyclic mini-batch extension on
    NYU-Depth, shown for four representative batch sizes $b$. For each batch size, the number
    of learning updates $T$ is chosen to reach the same training RMSE ($\approx0.79$) at the end
    of training.}
    \label{fig:nyuv2-minibatch-tradeoff}
\end{figure}

\begin{table}[!htbp]
\centering
\caption{Computational cost of the cyclic mini-batch experiments in 
Figure~\ref{fig:nyuv2-minibatch-tradeoff} and the full-batch
experiment in Figure~\ref{fig:nyuv2-privacy-utility} for fixed $K=80$. The four mini-batch
runs are stopped at the same pre-unlearning training RMSE ($\approx0.79$). We report the
number of examples processed during learning and unlearning, using
$b\times T$ and $b\times K$ as a simple cost proxy.}
\label{tab:nyuv2-minibatch-cost}
\begin{tabular}{c c c c c c}
\toprule
Batch size $b$ & $T$ & Learning cost $bT$ & $K$ & Unlearning cost $bK$ & Total cost $b(T+K)$ \\
\midrule
$256$  & $24\,900$ & $6\,374\,400$ & 80 & $20\,480$  & $6\,394\,880$ \\
$512$  & $11\,325$  & $5\,798\,400$ & 80 & $40\,960$  & $5\,839\,360$ \\
$1024$ & $5\,548$  & $5\,681\,152$ & 80 & $81\,920$  & $5\,763\,072$ \\
$2048$ & $2\,717$  & $5\,564\,416$ & 80 & $163\,840$ & $5\,728\,256$ \\
$n$ & $800$  & $30\,720\,000$ & 80 & $3\,072\,000$ & $33\,792\,000$ \\
\bottomrule
\end{tabular}
\end{table}
\section{Implementation Details}
\subsection{Learning, Calibration, and Targeted Unlearning}
\label{app:algo}

Using Theorem~\ref{thm:ridge-gdp-main} yields an explicit and fully
implementable procedure for per-instance Langevin unlearning in the case of ridge regression (PILU).
It has two parts. Algorithm~\ref{alg:learning} is run once on the full dataset, before any
deletion request is known: it trains the model and performs a request-independent
preprocessing, and it neither computes nor stores any quantity attached to a particular data
point. Algorithm~\ref{alg:unlearning} is run when a deletion request for point $i$ arrives: it
forms the per-instance sensitivity bounds of Theorem~\ref{thm:ridge-gdp-main} from closed-form
quantities, calibrates $\sunl$, and performs the $K$ unlearning steps. When $\slearn=0$,
Theorem~\ref{thm:deterministic} applies instead: the exact sensitivities of all $n$ points are
accumulated during training in the scalars $N_j$ (Section~\ref{sec:deterministic}), and the
calibration step of Algorithm~\ref{alg:unlearning} reduces to~\eqref{eq:mu-deterministic}.

\begin{algorithm}[!htbp]
\caption{PILU: learning phase (run once, before any deletion request)}
\label{alg:learning}
\begin{algorithmic}

\STATE {\bfseries Input:}
dataset $\D=\{(x_j,y_j)\}_{j=1}^n$,
learning horizon $T$,
step size $\eta$,
regularization $\lambda$,
learning noise $\slearn$,
initial parameter $\theta_0$.

\STATE {\bfseries Output:}
trained parameter $\theta_T$;
statistics $(U,\Lambda,B)$ kept for future deletion requests.

\STATE {\bfseries Preprocessing (request-independent).}
Form
\[
A=X^\top X+\lambda I_p,
\qquad
B=X^\top Y,
\]
and compute the eigendecomposition $A=U\Lambda U^\top$, in which $M=I_p-\eta A$ is diagonal.

\STATE {\bfseries Learning phase (full dataset).}
Initialize $\theta\gets\theta_0$.

\FOR{$k=0$ {\bfseries to} $T-1$}
  \STATE Perform one Langevin step on $\D$:
  \[
  \theta\gets
  \theta
  -\eta\nabla_\theta f_{\D}(\theta)
  +\sqrt{2\eta}\,\slearn\,\xi_k,
  \qquad \xi_k\sim\mathcal N(0,I).
  \]
\ENDFOR

\STATE \textbf{return}
$\theta_T\gets\theta$ and $(U,\Lambda,B)$.

\end{algorithmic}
\end{algorithm}

\begin{algorithm}[!htbp]
\caption{PILU: processing a deletion request for point $i$}
\label{alg:unlearning}
\begin{algorithmic}

\STATE {\bfseries Input:}
dataset $\D$;
trained parameter $\theta_T$ and statistics $(U,\Lambda,B)$ returned by
Algorithm~\ref{alg:learning}, together with its hyperparameters $(T,\eta,\lambda,\slearn,\theta_0)$;
deletion index $i$;
unlearning horizon $K$;
target guarantee $(\eps,\del)$ and split $\del=\dm+\ds$.

\STATE {\bfseries Output:}
unlearned parameter $\theta_{T+K}$,
calibrated noise level $\sunl$.

\STATE {\bfseries Sensitivity bounds (Theorem~\ref{thm:ridge-gdp-main}(i)).}
Set
\[
s_{i,0}^{\ds}
=
\eta\|x_i\|_2\|x_i^\top\theta_0-y_i\|_2.
\]

\FOR{$k=1$ {\bfseries to} $T-1$}
  \STATE Compute $(u_{i,k},v_{i,k})$ in closed form (Lemma~\ref{lem:gaussian-kronecker-proof}) in the
  eigenbasis $U$, without simulating the Langevin trajectory, and set
  \[
  s_{i,k}^{\ds}
  =
  \eta\|x_i\|_2
  \sqrt{
  v_{i,k}\,
  q_{i,k}\!\left(1-\frac{\ds}{T}\right)
  },
  \]
  where $q_{i,k}$ is the quantile function of
  $\chi'^2_d(\|u_{i,k}\|_2^2/v_{i,k})$.
\ENDFOR

\STATE {\bfseries Noise calibration (GDP accounting).}
Let
\[
c=1-\eta\lambda,
\qquad
\mu_{\mathrm{target}}
=
\sup\{\mu>0:\epsgdp(\mu,\dm)\le\eps\},
\]
and form $N_{i,j}$ and $N_i$ from $\{s_{i,k}^{\ds}\}_{k<T}$ as in~\eqref{eq:prefix}.

\IF{condition~\eqref{eq:sufficient-main} holds}
  \STATE Set $\sunl$ to the smallest $\sigma$ such that
  $N_i/\sqrt{V(\sigma)}\le\mu_{\mathrm{target}}$, that is,
  \[
  \sunl^2
  =
  \frac{N_i^2/\mu_{\mathrm{target}}^2-\sum_{k=0}^{T-1}2\eta\slearn^2c^{2(T+K-1-k)}}
  {2\eta\sum_{k=0}^{K-1}c^{2k}}.
  \]
\ELSE
  \STATE Let $\mu_i(\sigma)=\sqrt{Q_i^\star(\sigma)}$, where $Q_i^\star(\sigma)$ is the optimal
  value of the quadratic program~\eqref{eq:shift-allocation-qp} with $\sigma_k=\slearn$ for
  $k<T$ and $\sigma_k=\sigma$ for $k\ge T$. Since $\mu_i$ is nonincreasing in $\sigma$, find by
  bisection the smallest $\sunl$ such that $\mu_i(\sunl)\le\mu_{\mathrm{target}}$.
\ENDIF

\STATE (Optional) Round $\sunl$ up to a public grid of noise levels
(Proposition~\ref{prop:grid}).

\STATE {\bfseries Targeted unlearning (dataset $\Di$).}
Initialize $\theta\gets\theta_T$.

\FOR{$k=T$ {\bfseries to} $T+K-1$}
  \STATE Perform one Langevin step on $\Di$:
  \[
  \theta\gets
  \theta
  -\eta\nabla_\theta f_{\Di}(\theta)
  +\sqrt{2\eta}\,\sunl\,\xi_k,
  \qquad \xi_k\sim\mathcal N(0,I).
  \]
\ENDFOR

\STATE \textbf{return}
$\theta_{T+K}\gets\theta$, $\sunl$.

\end{algorithmic}
\end{algorithm}

\subsection{Implementation Details for the Experiments on NYU-Depth}
\label{app:nyuv2-implementation}

\paragraph{Dataset and feature extraction.}
We use the NYU-Depth V2 dataset~\citep{silberman2012indoor}. 
RGB images are resized to $224\times224$, and frozen DINOv2 
features~\citep{oquab2024dinov2} are extracted using a \texttt{dinov2\_vitb14} 
backbone. For each image, the backbone returns a $16\times16$ grid of patch tokens,
yielding $P=256$ patch-level representations. In addition, we append the
normalized class token to each patch embedding, so that each patch is
represented by a $1536$-dimensional feature vector before adding the bias term.
Depth supervision is defined at the patch level. For each image, the resized
depth map is partitioned into the same $16\times16$ grid, and the target of a
patch is taken to be the median valid depth value inside that patch.

\paragraph{Splits and dimensions.}
We keep $150$ images for training, $25$ for validation and $25$ for test, splitting at
the image level before flattening the patch dimension. This yields
$n=38400$ training patches, $6400$ validation patches and $6400$ test patches. The
validation images are used only to select the regularization parameter, and all test
errors reported in the paper are computed on the test images. The patch features are standardized
using the training set statistics only, while the scalar depth targets are left
unstandardized. A constant bias feature is appended to each patch representation.
The effective linear input dimension is therefore $p=1537$.
The final ridge regression task is a scalar-output problem with
$n=38400$, $p=1537$, and $d=1$.

\paragraph{Regularization and step size.}
The regularization parameter is a modeling choice rather than a parameter of the
certificate, and we select it as for any ridge regression, by validation of the learned
model alone. For $\lambda\in\{10^{-6},10^{-5},\dots,10\}\times n$, the minimizer of the
ridge objective is computed in closed form, without noise and without unlearning, and
evaluated on the $25$ validation images, disjoint from the training and test images.
Table~\ref{tab:lambda-validation} reports the results. The validation RMSE decreases slowly
with $\lambda$, from $0.734$\,m for $\lambda\le10^{-4}\,n$ to its minimum $0.706$\,m at
$\lambda=0.1\,n$, and increases sharply beyond, reaching $1.461$\,m at $\lambda=n$. We
therefore retain $\lambda=0.1\,n=3840$. This choice uses only the validation error of the
learned model, not the unlearning results.

The objective is $m$-strongly convex with $m=\lambda$ and $L$-smooth with
$L=\lambda_{\max}(X^\top X)+\lambda\approx3.94\times10^{6}$, and we set the step size to
$\eta=1/L$, so that $c=1-\eta\lambda$ with $\eta\lambda\approx10^{-3}$. The contraction
is therefore moderate over the learning horizon ($c^{T}\approx0.46$ for $T=800$) and weak
over the unlearning horizon ($c^{K}\approx0.93$ for $K=80$).

\begin{table}[!htbp]
\centering
\caption{Validation of the regularization parameter on NYU-Depth: training and validation
RMSE (in meters) of the ridge minimizer, computed in closed form without noise and without
unlearning. Bold: minimum of the validation RMSE, which is the retained value.}
\label{tab:lambda-validation}
\begin{tabular}{c c c c}
\toprule
$\lambda/n$ & $\lambda$ & Train RMSE & Validation RMSE \\
\midrule
$10^{-6}$ & $0.0384$     & $0.731$ & $0.734$ \\
$10^{-5}$ & $0.384$      & $0.731$ & $0.734$ \\
$10^{-4}$ & $3.84$       & $0.731$ & $0.734$ \\
$10^{-3}$ & $38.4$       & $0.731$ & $0.732$ \\
$10^{-2}$ & $384$        & $0.732$ & $0.720$ \\
$10^{-1}$ & $3840$       & $0.787$ & $\mathbf{0.706}$ \\
$1$       & $38\,400$   & $1.582$ & $1.461$ \\
$10$      & $384\,000$  & $2.664$ & $2.554$ \\
\bottomrule
\end{tabular}
\end{table}

\paragraph{Splitting $\delta$ into $\ds+\dm$.}
The total $\delta$ is conventionally set to $1/n$ in the DP and certified unlearning
literature and we follow that, splitting it evenly, $\ds=\dm=1/(2n)$. The split matters
little, and this is visible in the formulas: $\dm$ enters through the GDP-to-DP
conversion, which behaves as $\epsgdp(\mu,\dm)=O(\mu\sqrt{\log(1/\dm)})$ for small
$\dm$, while $\ds$ enters the sensitivity bounds only through the quantile
$z_{1-\ds/T}=O(\sqrt{\log(T/\ds)})$, which moreover multiplies a term proportional to
$\slearn$. Shifting the budget trades one logarithm
against the other. Concretely, with $n=38\,400$ and $T=800$, moving to a $1{:}10$ split
in either direction changes $z_{1-\ds/T}$ from $5.53$ to $5.42$ or $5.82$, and
$\sqrt{2\log(1/\dm)}$ from $4.74$ to $5.09$ or $4.62$, that is by at most a few percent
each. With deterministic training (Theorem~\ref{thm:deterministic}) the question does not
arise, since $\ds=0$.

\paragraph{Learning noise $\slearn$.}
This parameter appears on both sides: increasing it strengthens the training-time DP
guarantee (Appendix~\ref{app:ridge-dp-guarantee}) and contributes to the unlearning
accounting through the denominator $V(\sunl)$ of $\mu_i$, but it inflates the variances
$v_{i,k}$ and hence the sensitivity bounds entering $N_i$, and it also tightens the
feasibility condition~\eqref{eq:sufficient-condition}. We take $\slearn=0.1$ in all
experiments, a value small enough for the learned model to retain its utility, since the
released model should not pay for a request that may never occur. We do not tune it
further, because the optimal trade-off depends on \emph{which} point is deleted, which is
unknown at training time. The limiting case $\slearn=0$ is Theorem~\ref{thm:deterministic}.
In all our full-batch experiments, the feasibility
condition~\eqref{eq:feas-crit} holds at the noise level that is used (the implementation
checks a sufficient condition for each request), so the certificate is
the closed-form one. Condition~\eqref{eq:sufficient-main} itself holds for all $n$ training
points at $\eps\le1$ (Appendix~\ref{app:gain-uniform}). 

\paragraph{Training horizon $T$ and unlearning horizon $K$.}
Longer training improves the learned model but accumulates sensitivity in the numerator
$N_i$ of $\mu_i$ and thus degrades the certificate. We do not read this as an instruction to
under-train in anticipation of a hypothetical request. Training error plateaus, and
the practical rule is to stop once the learned model is satisfactory, by the usual
criterion of early stopping on a validation set; our analysis adds a reason not to
iterate far past convergence, which should not be a binding constraint. As for $K$, the
gains are monotone but with rapidly diminishing returns, and the whole point of
unlearning is that $K\ll T$; see Section~\ref{sec:experiments-linear}.

\paragraph{Choice of representative indices.} 
We first rank training patches by their per-sample gradient norm averaged along the training
trajectory. On NYU-Depth, this distribution is extremely spread out: in our
run, the smallest norm is $3.6$, while the largest one is $254$. Using the full range would
therefore produce overly stretched
privacy--utility plots and make representative comparisons difficult to read.
To focus on the regime where most points lie, we restrict attention to the
subset of training patches whose average gradient norm is at most $60$, which
covers $35\,685$ data points, roughly $93\%$ of the $n=38\,400$ training patches. We then sort these
remaining points by average gradient norm and select five representative indices as the
empirical quantiles $q\in\{0,0.25,0.5,0.75,1\}$ within this truncated set. The average is
taken along one noisy training trajectory, so the ranking depends slightly on the random
seed and on the hardware; the five indices were selected once and are kept fixed in all
experiments. This ranking is a faithful proxy for the certificate: over all training points,
its Spearman correlation with the calibrated noise level $\sunl$ is $0.99$
(Appendix~\ref{app:gain-uniform}). The
corresponding patches are shown in Figure~\ref{fig:representative-samples}, 
and the corresponding image patches, together with their ground-truth and predicted
patch-level depth maps, are shown in Figure~\ref{fig:nyuv2-patches-prediction}.
For completeness, we also report in Figure~\ref{fig:nyuv2-loglog-all} the
privacy--utility trade-off obtained from the full gradient-norm range, displayed
on log--log scales to accommodate the much larger spread induced by the most
influential points.

\begin{figure}[!htbp]
    \centering
    \includegraphics[width=\linewidth]{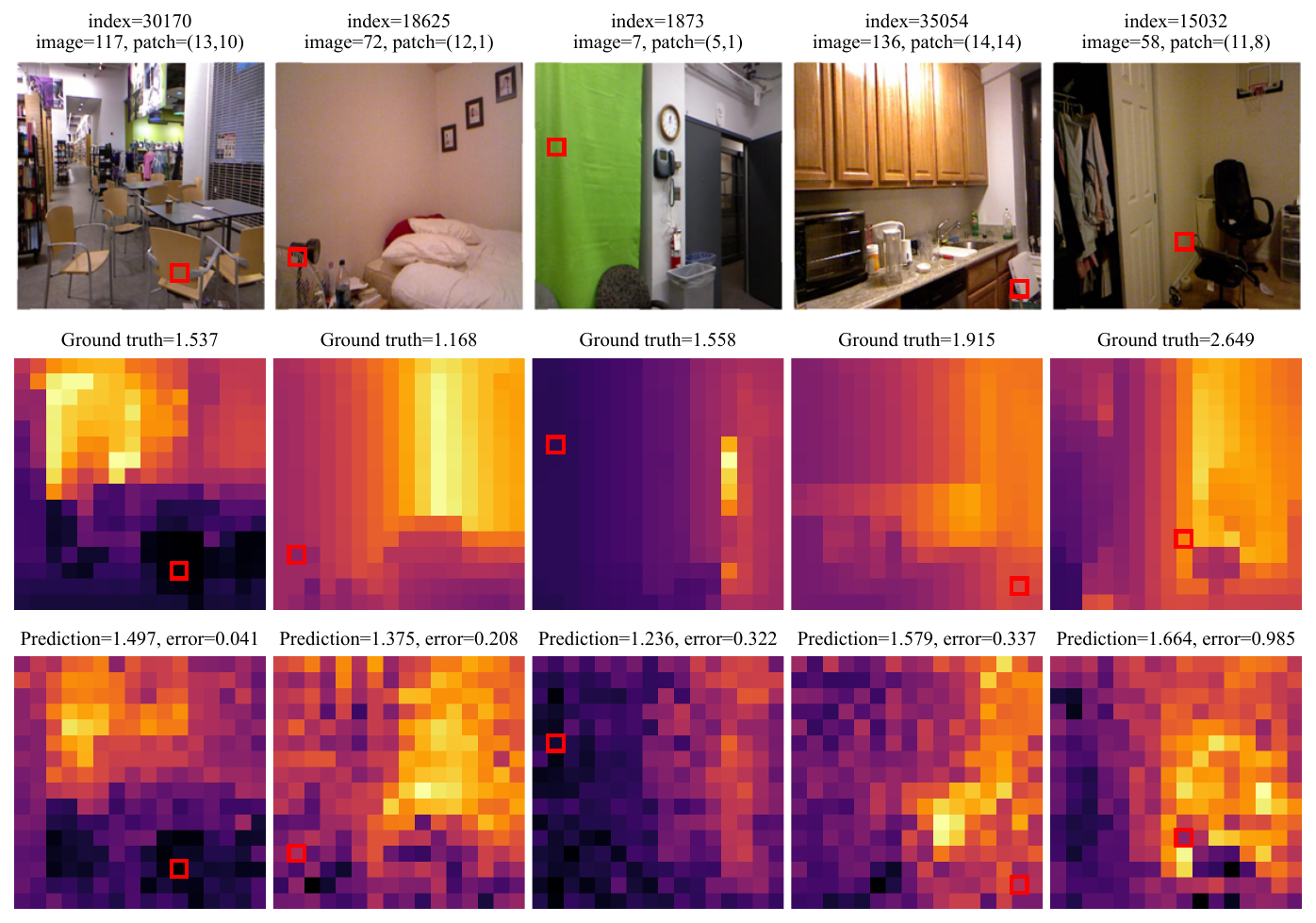}
    \caption{
    Visual inspection of the five representative NYU-Depth training patches
    selected from the truncated gradient-norm spectrum. For each selected
    flat index, the top row shows the corresponding RGB image with the
    selected $16\times16$ patch highlighted in red. The middle row shows the
    ground-truth patch-level depth map obtained by taking the median valid
    depth in each patch. The bottom row shows the corresponding prediction of
    the ridge model. Ground-truth depths, predictions, and absolute errors are
    reported in meters for the highlighted patches. This qualitative view
    illustrates that the selected indices cover patches with different image
    locations, depths, and prediction errors.
    }
    \label{fig:nyuv2-patches-prediction}
\end{figure}

\begin{figure}[!htbp]
    \centering
    \begin{subfigure}[t]{0.49\linewidth}
        \centering
        \includegraphics[width=\linewidth]{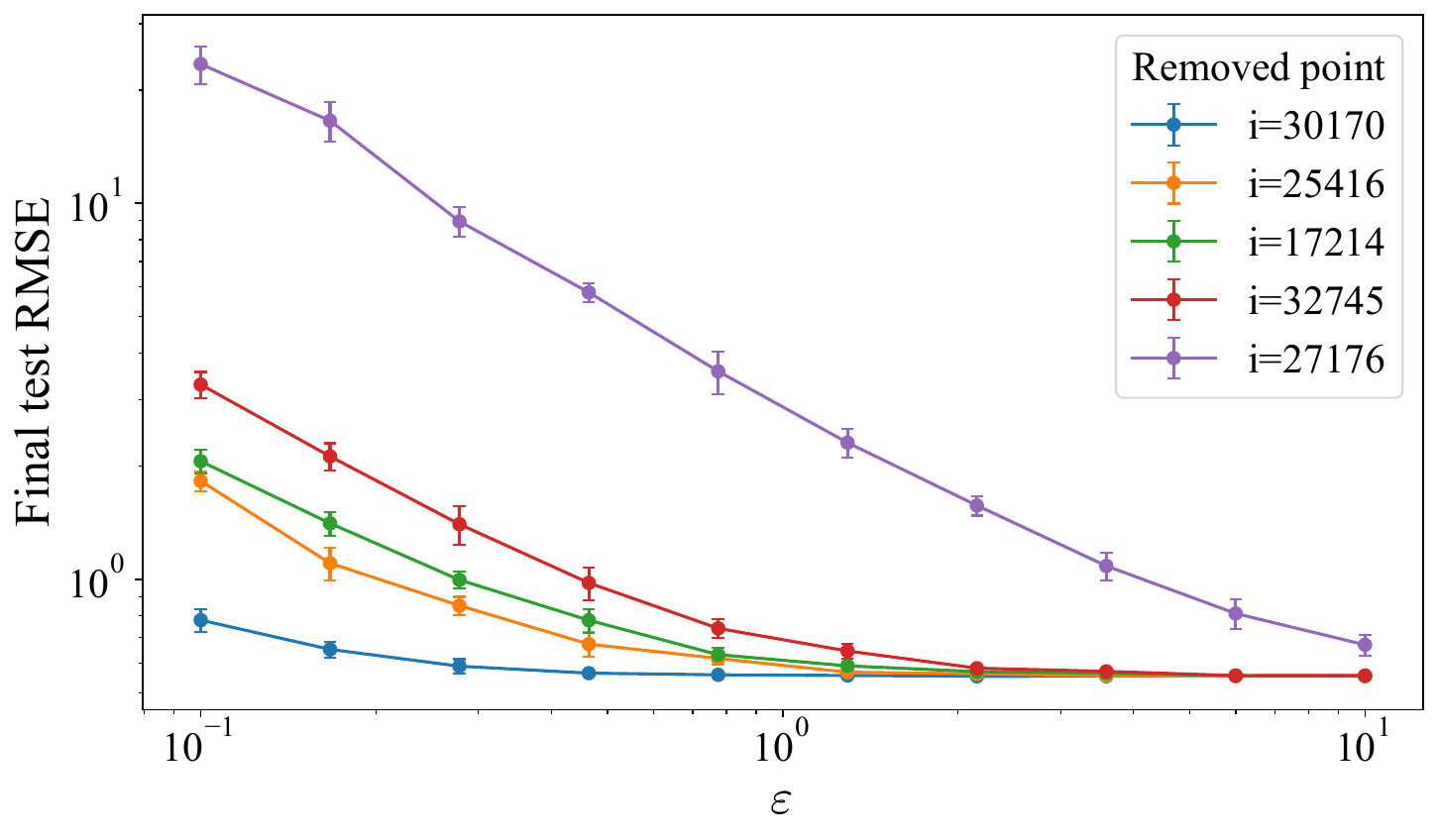}
        \caption{Final test RMSE versus $\varepsilon$ (log--log scale).}
        \label{fig:nyuv2-loglog-rmse}
    \end{subfigure}
    \hfill
    \begin{subfigure}[t]{0.49\linewidth}
        \centering
        \includegraphics[width=\linewidth]{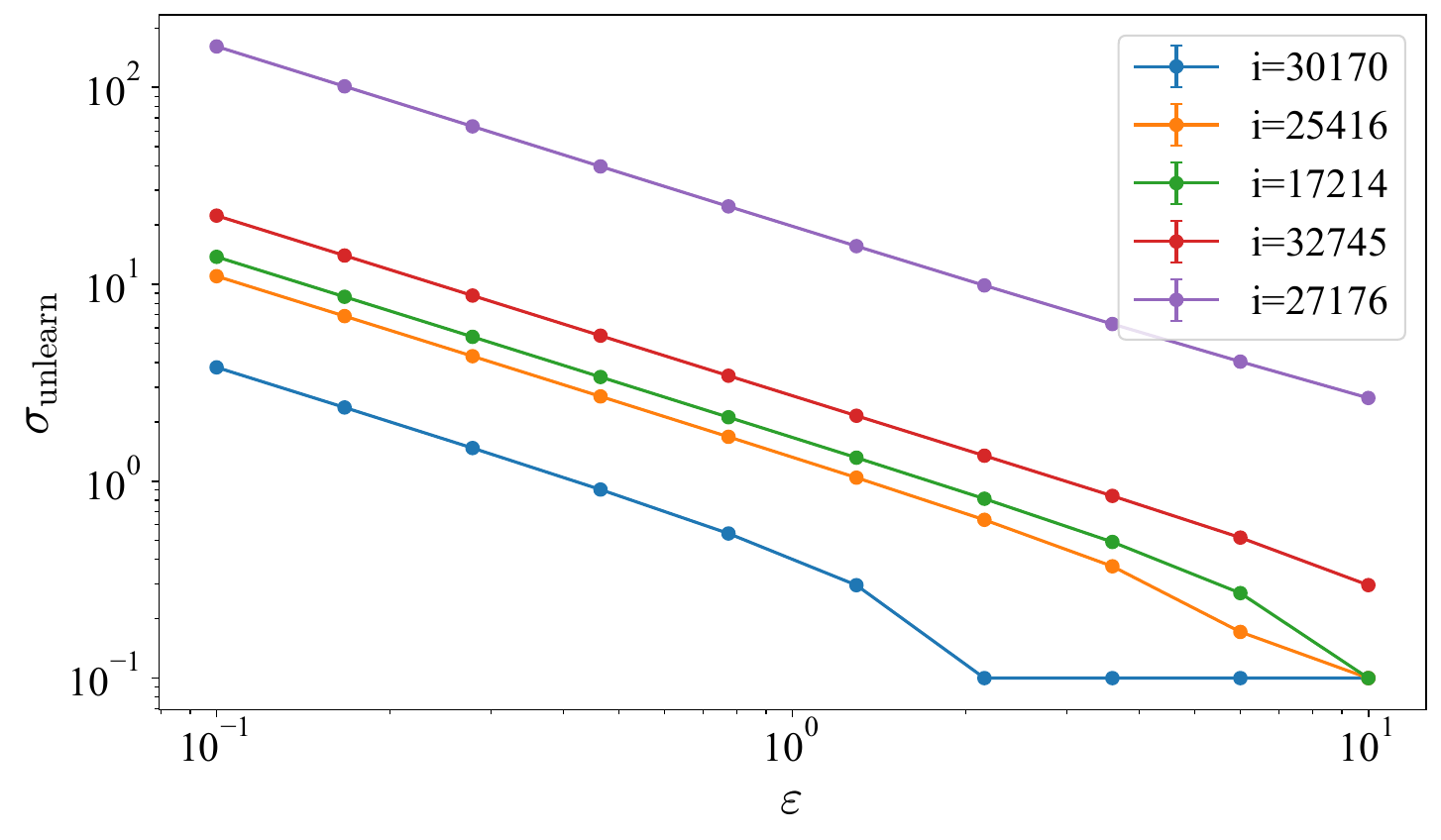}
        \caption{Calibrated $\sigma_{\mathrm{unlearn}}$ versus $\varepsilon$ (log--log scale).}
        \label{fig:nyuv2-loglog-sigma}
    \end{subfigure}
    \caption{Full-range privacy--utility trade-off on NYU-Depth for the five
    representative indices selected from the full gradient-norm spectrum. Left:
    utility improves as the target privacy budget $\varepsilon$ increases.
    Right: the calibrated unlearning noise decreases with $\varepsilon$, with
    large variability across removed points (a factor of $43$ between the two extreme
    points at $\varepsilon=0.1$).}
    \label{fig:nyuv2-loglog-all}
\end{figure}

\paragraph{Implementation details for Figure~\ref{fig:nyuv2-hp-sensitivity-map}.}
For each representative index, we plot $10$ independent Langevin trajectories, using
$T=800$ learning iterations and $\sigma_{\mathrm{learn}}=0.1$.

\paragraph{Implementation details for Figure~\ref{fig:nyuv2-privacy-utility}.}
\label{app:fig-privacy-utility}
We use $\sigma_{\mathrm{learn}}=0.1$.
Each value reports the mean over $10$ independent runs; error bars show one standard deviation.
In addition to final test RMSE, we report in Figure~\ref{fig:sigma-per-instance}
the values of the unlearning noise level $\sigma_{\mathrm{unlearn}}$ required to achieve
a target privacy budget $\varepsilon$ for individual data points.
Each curve corresponds to a distinct removal request, and illustrates how the noise
calibration varies across points.
We highlight the magnitude of this heterogeneity:
for a fixed privacy level $\varepsilon$, the required $\sigma_{\mathrm{unlearn}}$
differs by a factor of about ten between the least and most influential of the five points
(Figure~\ref{fig:sigma-per-instance}): the ratio is $10.1$ at $\varepsilon=0.1$ ($38.2$
against $3.8$) and stays between $10.1$ and $12.4$ for all $\varepsilon\le1.3$.
For larger $\varepsilon$, the easiest point reaches $\slearn=0.1$ from $\varepsilon\approx2.2$ onwards
(the plateau visible in Figure~\ref{fig:nyuv2-loglog-sigma}), and the ratio then lies
between $5.9$ and $23$.
Appendix~\ref{app:gain-uniform} reports the same quantity for all training points.

\begin{figure}[!htbp]
    \centering
    \includegraphics[width=0.5\linewidth]{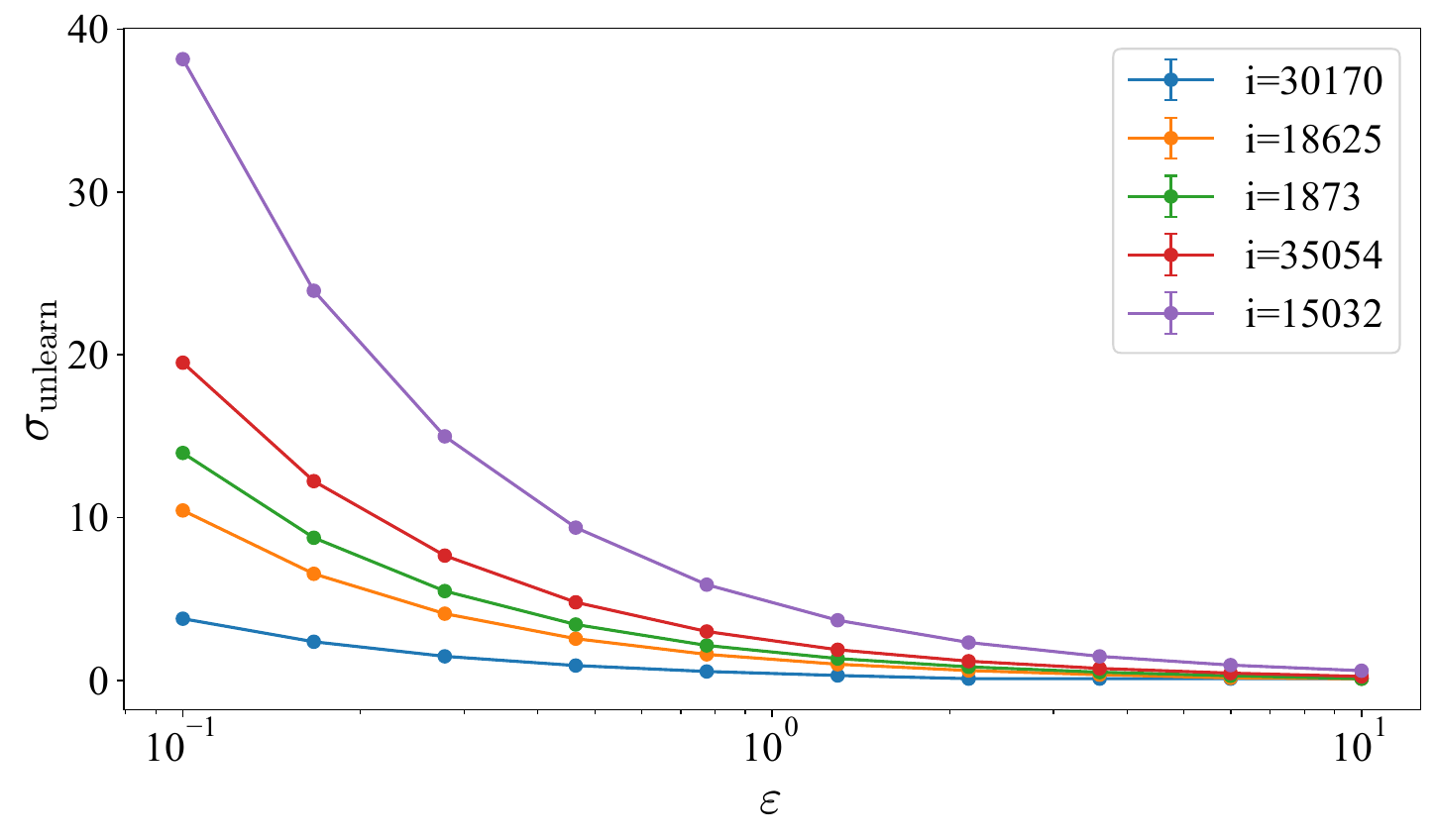}
    \caption{
    Required unlearning noise level $\sigma_{\mathrm{unlearn}}$ as a function of the
    target privacy budget $\varepsilon$ for representative removal requests (NYU-Depth).
    Each curve corresponds to a distinct unlearned data point.
    For a fixed privacy target, the required noise varies substantially across points,
    with differences of about a factor of ten between the least and most influential
    examples.
    }
    \label{fig:sigma-per-instance}
\end{figure}

\paragraph{Implementation details for Figure~\ref{fig:ablation-TK}.}
We use $\sigma_{\mathrm{learn}}=0.1$, with $T=800$ when $K$ varies and $K=80$ when $T$ varies.
Each value reports the mean over $20$ independent runs for the ablation on $K$ and over $10$
independent runs for the ablation on $T$; error bars show one standard deviation.

\paragraph{Langevin unlearning baseline
\citep{chien2025langevinunlearningnewperspective}.}
\label{app:langevin-baseline}

We implement the Langevin-unlearning baseline of
\citet{chien2025langevinunlearningnewperspective}, with minor adaptations to match
our experimental setting. Whereas their experiments use $\ell_2$-regularized
logistic regression, we use the same ridge-regression model as for PILU and clip
per-example gradients at a constant $C$. Indeed, for convex generalized linear models, per-example gradient clipping
preserves convexity of the induced objective \citep{song2021evading}; together
with the $\ell_2$ regularizer, the resulting objective therefore remains strongly
convex. We also use omission adjacency,
$\Di=\D\setminus\{(x_i,y_i)\}$, rather than the replacement convention of
\citet{chien2025langevinunlearningnewperspective}.

The only substantive difference in the privacy accounting is that we evaluate
the learning algorithm after its actual finite horizon $T$, rather than replacing
it by its stationary $T\to\infty$ limit. This finite-time treatment is already
covered by the strongly convex analysis used by
\citet{chien2025langevinunlearningnewperspective}: in the proof of their
Theorem~3.3, the strongly convex case is obtained from Theorem~D.6 of
\citet{ye2022differentially}, which gives a geometrically converging
last-iterate RDP bound for noisy gradient descent. Under our omission and
sum-objective conventions, the one-step sensitivity is bounded by
$\eta C$, and we retain the corresponding finite-horizon geometric factor,
\begin{equation}
\varepsilon_{0,T}^{(\alpha)}(\sigma)
=
\frac{\alpha\eta C^2}{4\sigma^2}
\frac{1-\exp(-m\eta T)}
     {1-\exp(-m\eta)}.
\label{eq:langevin-finite-T}
\end{equation}
The subsequent $K$ unlearning iterations are accounted for exactly as in
Langevin unlearning, using the strongly convex privacy-recovery factor
\begin{equation}
\varepsilon_{K}^{(\alpha)}(\sigma)
=
\exp\!\left(-\frac{K\eta m}{\alpha}\right)
\varepsilon_{0,T}^{(\alpha)}(\sigma),
\end{equation}
followed by the standard RDP-to-$(\varepsilon,\delta)$ conversion
\begin{equation}
\varepsilon(\sigma,\delta)
=
\min_{\alpha>1}
\left\{
\varepsilon_{K}^{(\alpha)}(\sigma)
+
\frac{\log(1/\delta)}{\alpha-1}
\right\}.
\end{equation}

\paragraph{Implementation details for Figure~\ref{fig:nyuv2-ours-vs-langevin}.}
Both methods use $T=800, K=80$ and the same hyperparameters $\eta=1/L$, $\lambda=0.1\,n$. 
Each value reports the mean over $10$ independent runs; error bars show one standard deviation.

For PILU, we use $\sigma_{\mathrm{learn}}=0.1$ and $\sigma_{\mathrm{unlearn}}$ 
is calibrated for each removed index and each target $\varepsilon$ with Theorem~\ref{thm:ridge-gdp-main}.

For the Langevin baseline, we clip at $C=42$, close to the $85$th percentile,
over training points, of the per-sample gradient norms averaged along the training trajectory.
For each target $\varepsilon$, we choose the smallest common Langevin noise level
$\sigma$ satisfying this bound and use it during both learning and unlearning,
as in the original baseline. 

\paragraph{Learning-time privacy baseline (DP-GD) \citep{Abadi_2016,bok2024shiftedinterpolationdifferentialprivacy}.}
\label{app:learning-dp-baseline}

We compare against a learning-time privacy mechanism, which protects all training
points uniformly during optimization instead of performing a targeted post-hoc
deletion. Unlike PILU, this baseline does not adapt the privacy cost to the point
being removed. We run clipped noisy gradient descent on the ridge objective,
clipping the per-example gradients at constant $C$:
\[
\theta_{k+1}
=
\theta_k
-
\eta \Bigl(\sum_{j=1}^n\mathrm{clip}_C\bigl(\nabla \ell(\theta_k; x_j,y_j)\bigr)
+ \lambda\theta_k + \zeta_k\Bigr),
\qquad
\zeta_k\sim\mathcal N(0,\sigma^2 I).
\]
Removing one point changes the update by at most $\eta C$, while the Gaussian
perturbation has standard deviation $\eta\sigma$. We account for the privacy of the
released model $\theta_{T}$ in two ways, both in Gaussian differential privacy.

\emph{Composition.} Each iteration is a Gaussian mechanism with sensitivity $C$ and
noise standard deviation $\sigma$, hence $(C/\sigma)$-GDP. By adaptive composition
\citep{dong2019gdp}, the $T$ iterations are $\mu_{\mathrm{comp}}$-GDP with
$\mu_{\mathrm{comp}} = C\sqrt{T}/\sigma$, and the released model inherits this
guarantee by post-processing. This accountant certifies the whole sequence of
iterates.

\emph{Last iterate.} Clipping preserves the convexity and smoothness of the
data-fitting term for generalized linear models \citep{song2021evading}, so the
update map is $c$-contractive with $c=1-\eta\lambda$. Applying the shifted
interpolation accountant of \citet{bok2024shiftedinterpolationdifferentialprivacy}
with constant sensitivity $\eta C$, that is, the analysis of
Appendix~\ref{app:accounting} with $s_{i,k}=\eta C$ and noise $\eta\sigma$ at every
step, the final iterate alone is $\mu_{\mathrm{last}}$-GDP with
\[
\mu_{\mathrm{last}}
=
\frac{C}{\sigma}\,
\frac{\sum_{k=0}^{T-1}c^{k}}{\bigl(\sum_{k=0}^{T-1}c^{2k}\bigr)^{1/2}}
=
\frac{C}{\sigma}
\sqrt{\frac{(1+c)(1-c^{T})}{(1-c)(1+c^{T})}} .
\]
For $c<1$, this bound saturates as $T\to\infty$. For $c\to1$, it recovers
$\mu_{\mathrm{comp}}=C\sqrt{T}/\sigma$. In our setting, $\eta\lambda T\approx0.8$ and
$\mu_{\mathrm{last}}\approx0.976\,\mu_{\mathrm{comp}}$, so the two accountants yield nearly
identical noise levels. 

In both cases, given a target $(\varepsilon,\delta)$, we set $\sigma$ so that
$\mu=\mu_{\mathrm{target}}(\varepsilon,\delta)$, the largest $\mu$ such that $\mu$-GDP
implies $(\varepsilon,\delta)$-DP (Appendix~\ref{app:gdp}). This is the conversion
used for our certificate, so all methods are compared under the same accounting.
Since the noise level is fixed before training, regardless of which point is later
deleted, DP-GD yields a single global privacy--utility trade-off. It should be read
as a uniform learning-time privacy baseline, not as a targeted post-hoc unlearning
method.

\paragraph{Implementation details for Figure~\ref{fig:nyuv2-ours-vs-dpgd}.}
All methods use $T=800$, $\eta=1/L$ and $\lambda=0.1\,n$, and PILU uses $K=80$. Each
value reports the mean over $10$ independent runs, and error bars show one standard
deviation. For PILU, we use $\sigma_{\mathrm{learn}}=0.1$, and $\sigma_{\mathrm{unlearn}}$
is calibrated for each removed index and each target $\varepsilon$ with
Theorem~\ref{thm:ridge-gdp-main}. For DP-GD, we clip at $C=42$, close to the $85$th percentile,
over training points, of the per-sample gradient norms averaged along the training trajectory,
and calibrate $\sigma$ before training under each accountant to reach the target $\varepsilon$.

\paragraph{Newton step baseline \citep{guo2020certified}.}
\label{app:newton-baseline}

The \emph{Newton update removal mechanism} was introduced by
\citet{guo2020certified}. The idea is to first train a model
with a method called objective perturbation \citep{chaudhuri2011}, then remove a point through a one-step Newton
update, and finally control the quality of this approximation through the norm of
the gradient residual.

More precisely, the training objective is perturbed by a random linear term:
\begin{equation}
L_b(\theta;\D)
=
\sum_{j=1}^n \ell(\theta^\top x_j,y_j)
+
\frac{\lambda n}{2}\|\theta\|_2^2
+
b^\top \theta,
\label{eq:perturbed-objective-guo}
\end{equation}
where $b$ is sampled once at training time from a centered Gaussian distribution
of standard deviation $\sigma_{\mathrm{Guo}}$. 
The randomness is thus injected once in the objective
rather than at every gradient step. The model $\theta^\star$ is then obtained by
approximately minimizing $L_b(\theta;\D)$.

To remove one data point $(x_i,y_i)$, \citet{guo2020certified}
define
\[
\Delta
=
\lambda \theta^\star + \nabla \ell((\theta^\star)^\top x_i,y_i),
\]
and denote by
\[
H_{\theta^\star} = \nabla^2 L(\theta^\star;\D')
\qquad\text{with}\qquad
\D'=\D\setminus\{(x_i,y_i)\},
\]
the Hessian of the retained objective evaluated at $\theta^\star$.
The removal update is the one-step Newton correction
\[
\theta^{-}
=
\theta^\star + H_{\theta^\star}^{-1}\Delta.
\]
As emphasized by \citet{guo2020certified}, this correction is motivated by the influence-function
view of the removed sample \citep{kohliang2017}. 

A central point in \citet{guo2020certified} is that certified removal does not follow
from the Newton update alone: even a small gradient residual $\|\nabla L_b(\theta^{-};\D')\|$ 
may still leak information about the
removed point. The role of objective perturbation is precisely to mask the gradient residual. 
Then, if
\[
\|\nabla L_b(\theta^{-};\D')\|_F \le \varepsilon',
\]
the residual bound is translated into a certified-removal parameter through the
objective-perturbation calibration
\begin{equation}
\varepsilon
=
\frac{\varepsilon'}{\sigma_{\mathrm{Guo}}}
\sqrt{2\log\!\left(\frac{1.5}{\delta}\right)}.
\label{eq:epsilon-guo}
\end{equation}

As highlighted by \citet{guo2020certified}, for ridge regression the Newton
correction recovers the exact leave-one-out solution when applied at the exact
minimizer of the perturbed least-squares objective. Indeed, since the objective
is quadratic, the Newton step is exact and the retained-set gradient residual
vanishes. In contrast, in our finite-time setting, the model is only an
approximate minimizer after $T$ gradient steps. Consequently, the gradient
residual after applying the Newton step $\|\nabla L(\theta_T^{-};\D')\|_F$ 
(with no objective perturbation) remains non-zero even for large training
horizons; see the middle column in Table~\ref{tab:guo-gradient-residual-vs-T}. Thus, to obtain valid
unlearning guarantees with the Newton step method by \citet{guo2020certified}, 
we must add noise $b$ via objective perturbation and bound the gradient residual after
applying the Newton step.

\begin{table}[!htbp]
\centering
\caption{Gradient residual $\|\nabla L(\theta_T^{-};\D')\|$ with no objective perturbation or $\|\nabla L_b(\theta_T^{-};\D')\|$ with objective perturbation, after the Guo Newton step for different
training horizons $T$ on NYU-Depth.}
\label{tab:guo-gradient-residual-vs-T}
\begin{tabular}{rcc}
\toprule
$T$ & No obj. perturbation ($b=0$) & $\sigma_{\mathrm{Guo}}=0.5$ ($b\neq0$)\\
\midrule
$100$  & $35901.23$ & $35901.33$ \\
$200$  &$12629.25$ & $12629.27$ \\ 
$500$   &$1496.45$ &$1496.44$ \\  
$1000$ & $451.64$ &$451.61$ \\ 
$2000$  & $75.09$ &$75.09$ \\ 
$5000$  & $1.06$ &$1.07$ \\ 
$10000$  & $0.47$ &$0.47$ \\ 
$20000$  & $0.47$ &$0.47$ \\ 
\bottomrule
\end{tabular}
\end{table}

For ridge regression, let
\[
L_b(\theta;\D)
=
\frac12\|X\theta-y\|_F^2
+
\frac{\lambda n}{2}\|\theta\|_F^2
+
\langle b,\theta\rangle,
\qquad
H=X^\top X+\lambda n I .
\]
Gradient descent with $L$-smoothness, step size $\eta\le 1/L$, and
$m$-strong convexity gives the finite-time bound
\[
\|\nabla L_b(\theta_T;\D)\|_F
\le
L(1-\eta m)^T
\frac{
\|X^\top y\|_F+\|b\|_F
}{m}.
\]
For the quadratic ridge objective, the Newton correction on the retained
objective eliminates the leave-one-out term exactly. Hence,
\[
\|\nabla L_b(\theta_T^-;\D')\|_F
=
\|\nabla L_b(\theta_T;\D)\|_F
\le
L(1-\eta m)^T
\frac{
\|X^\top y\|_F+\|b\|_F
}{m}.
\]
Using the high-probability Gaussian bound
\[
\|b\|_F
\le
\sigma_{\mathrm{Guo}}
\left(\sqrt{pd}+\sqrt{2\log(1/\rho)}\right),
\]
we obtain the following high-probability bound on the gradient residual:
\[
\varepsilon'_T(\rho)
=
L(1-\eta m)^T
\frac{
\|X^\top y\|_F+
\sigma_{\mathrm{Guo}}
\left(\sqrt{pd}+\sqrt{2\log(1/\rho)}\right)
}{m}.
\]

However, this high-probability bound is extremely conservative.
Therefore, in our experiments, for the finite-$T$ ridge baseline, we use the actual retained
perturbed gradient residual
\[
\widehat{\varepsilon}'_T
=
\|\nabla L_b(\theta_T^{-};\D')\|_F
\]
and plug $\widehat{\varepsilon}'_T$ into \eqref{eq:epsilon-guo}.

\paragraph{Implementation details for Figure~\ref{fig:nyuv2-ours-vs-guo}.}
We use the same hyperparameters $\eta=1/L$ and 
$\lambda_{\text{PILU}} = n \lambda_{\text{Guo}} = 0.1\,n$, that is, $\lambda_{\text{Guo}}=0.1$.

For PILU, the learning and unlearning noise levels are fixed:
$\sigma_{\mathrm{learn}}=\sigma_{\mathrm{unlearn}}=0.5$, and we use a single
corrective unlearning step ($K=1$), for a fair comparison. For each value of $T$, the resulting per-instance
$(\varepsilon_i,\delta)$ unlearning guarantees are obtained using the reversed
Theorem~\ref{thm:ridge-gdp-main} for fixed
$i,T,K,\sigma_{\mathrm{learn}},\sigma_{\mathrm{unlearn}}$.

For the Newton-step baseline, we sample a single Gaussian perturbation with
scale $\sigma_{\mathrm{Guo}}=0.5$ and train the perturbed objective
\eqref{eq:perturbed-objective-guo}. For each value of $T$, we
apply the one-step Newton removal update and measure the retained-set gradient
residual; see the last column in Table~\ref{tab:guo-gradient-residual-vs-T}. The resulting
unlearning guarantee $\varepsilon$ is computed by plugging this residual
$\widehat{\varepsilon}'_T$ into \eqref{eq:epsilon-guo}.

Both guarantees are reported at fixed $\del=1/n$.

\section{Additional Experiments}
\label{app:additional-experiments}

\subsection{Gain over Uniform Calibration}
\label{app:gain-uniform}

Section~\ref{sec:experiments-linear} follows five representative points. Here we compute,
for each of the $n=38\,400$ training patches of NYU-Depth, the noise level $\sigma_i$ that
Theorem~\ref{thm:ridge-gdp-main} assigns to its deletion, and compare it with a uniform
level. The setting is that of Section~\ref{sec:experiments-linear} ($\lambda=0.1\,n$,
$\slearn=0.1$, $T=800$, $K=80$, $\del=1/n$ with $\ds=\dm$), with $\eps\in\{0.1,0.5,1\}$.
The bounds $s_{i,k}^{\ds}$ are closed-form functions of the data, so no trajectory is
simulated and the computation takes a few seconds (Appendix~\ref{app:compute}). The levels
of the five representative points coincide with those of
Figure~\ref{fig:sigma-per-instance}, and condition~\eqref{eq:sufficient-main} holds for
every training point at the three values of $\eps$.

For the uniform level $\sigma_{\mathrm{unif}}$, we replace the sequence of point $i$ by
$\bar s_k:=\max_{1\le j\le n}s_{j,k}^{\ds}$, the smallest sequence valid for all training
points, and calibrate in the same way (condition~\eqref{eq:sufficient-main} holds for it
as well). This choice favors uniform calibration, since a level valid for every dataset
would be larger. We call $g_i:=\sigma_{\mathrm{unif}}/\sigma_i\ge1$ the \emph{gain} of
point $i$.

Table~\ref{tab:nyu-gain} and Figure~\ref{fig:nyu-gain} give the distribution of $\sigma_i$
and $g_i$ over the training set. At $\eps=0.1$, $\sigma_{\mathrm{unif}}=161.5$ and the
median of $\sigma_i$ is $14.8$. The median gain is $10.9$, and the gain is at least $2$
for $98.9\%$ of the points, at least $5$ for $89.6\%$ and at least $10$ for $55.8\%$. The
uniform level is set by very few points: the gain is below $2$ for $433$ patches
($1.1\%$), all outside the truncated set from which the representative points are selected
(Appendix~\ref{app:nyuv2-implementation}). Within this set, which contains $92.9\%$ of the
training points, the gain is at least $3.9$. The five representative points have gains
$42.7$, $15.5$, $11.6$, $8.3$ and $4.2$ (markers in Figure~\ref{fig:nyu-sigma-ecdf}), and
over the training set the Spearman correlation between $\sigma_i$ and the average gradient
norm used to select them is $0.99$. The gain barely depends on $\eps$
(Figure~\ref{fig:nyu-gain-ecdf}). Only the largest gains increase, from $45$ at $\eps=0.1$
to $52$ at $\eps=1$, because the learning noise matters for the least influential points
at large $\eps$.

\begin{table}[!htbp]
\centering
\caption{Gain $g_i=\sigma_{\mathrm{unif}}/\sigma_i$ of per-instance over uniform calibration
on NYU-Depth, over all $n=38\,400$ training points and over the truncated set from which
the representative points are selected ($92.9\%$ of the points). The last three columns
give the share of points whose gain is at least $2$, $5$ and $10$.}
\label{tab:nyu-gain}
\begin{tabular}{c c c l c c c c c c}
\toprule
 & & & & \multicolumn{3}{c}{Gain} & \multicolumn{3}{c}{Share of points with gain} \\
\cmidrule(lr){5-7}\cmidrule(lr){8-10}
$\eps$ & $\sigma_{\mathrm{unif}}$ & Median $\sigma_i$ & Points & min & median & max
 & $\ge2$ & $\ge5$ & $\ge10$ \\
\midrule
$0.1$ & $161.5$ & $14.8$ & all       & $1.00$ & $10.93$ & $45.17$ & $98.9\%$  & $89.6\%$ & $55.8\%$ \\
      &         &        & truncated & $3.89$ & $11.54$ & $45.17$ & $100\%$ & $96.4\%$ & $60.0\%$ \\
$0.5$ & $37.1$  & $3.39$ & all       & $1.00$ & $10.95$ & $46.73$ & $98.9\%$  & $89.6\%$ & $55.9\%$ \\
      &         &        & truncated & $3.89$ & $11.57$ & $46.73$ & $100\%$ & $96.4\%$ & $60.2\%$ \\
$1$   & $19.7$  & $1.79$ & all       & $1.00$ & $11.01$ & $51.89$ & $98.9\%$  & $89.6\%$ & $56.2\%$ \\
      &         &        & truncated & $3.89$ & $11.64$ & $51.89$ & $100\%$ & $96.4\%$ & $60.5\%$ \\
\bottomrule
\end{tabular}
\end{table}

\begin{figure}[!htbp]
    \centering
    \begin{subfigure}[t]{\linewidth}
        \centering
        \includegraphics[width=\linewidth]{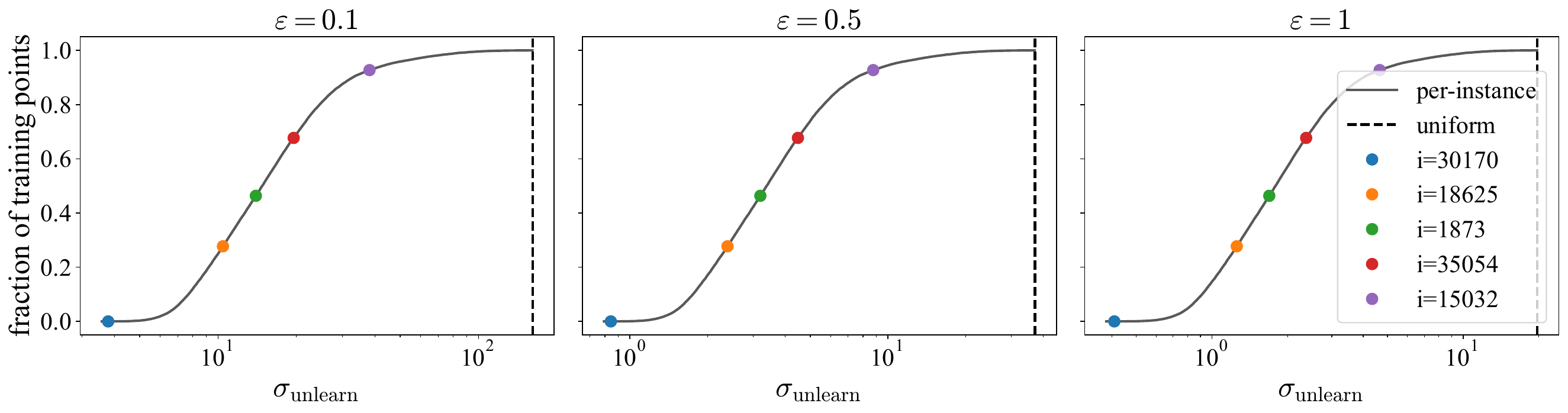}
        \caption{Distribution of the per-instance noise level $\sigma_i$ over the training
        points, for three values of $\eps$. Dashed line: uniform level
        $\sigma_{\mathrm{unif}}$. Markers: the five representative points of
        Section~\ref{sec:experiments-linear}.}
        \label{fig:nyu-sigma-ecdf}
    \end{subfigure}

    \vspace{0.5em}
    \begin{subfigure}[t]{0.49\linewidth}
        \centering
        \includegraphics[width=\linewidth]{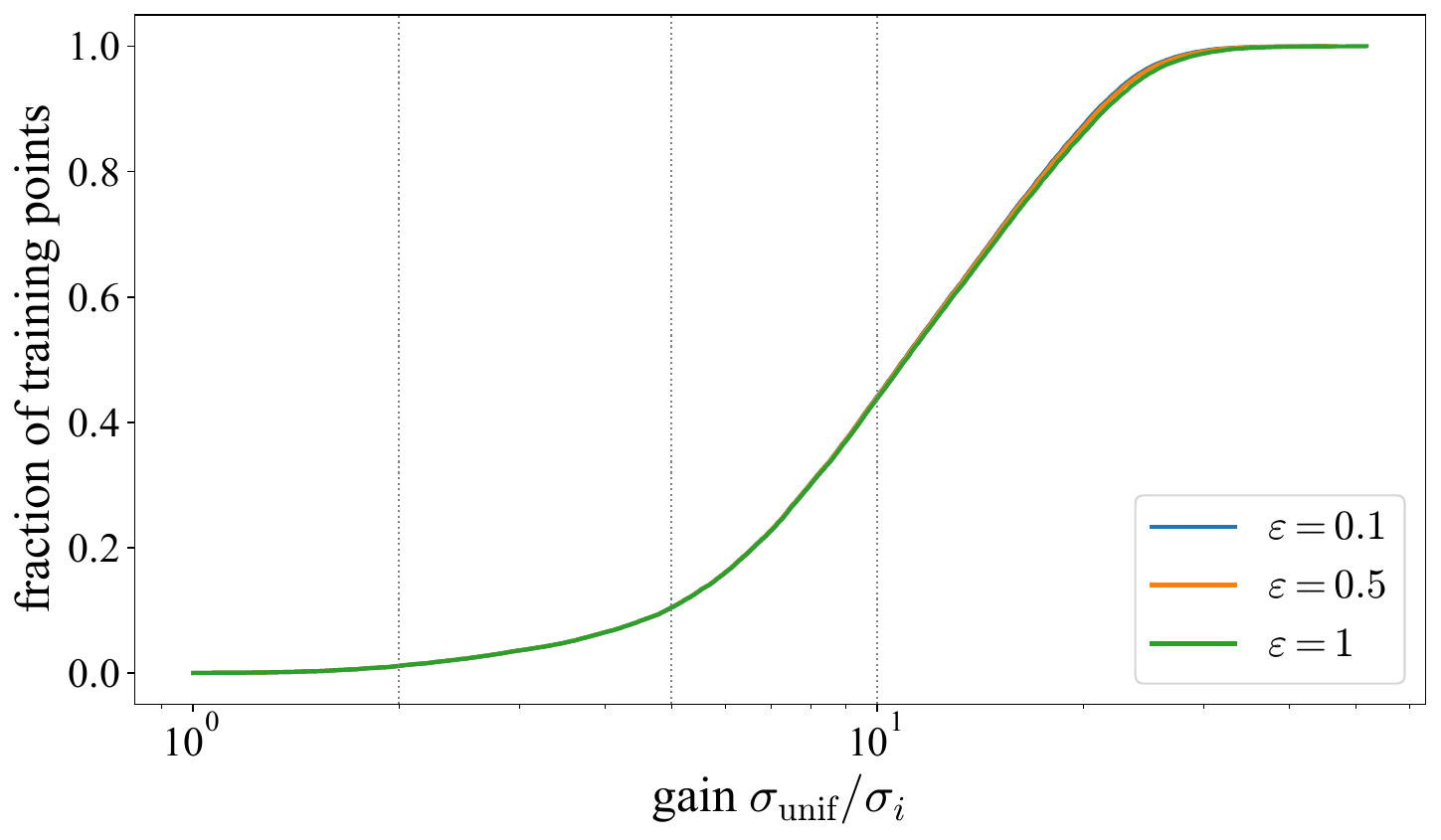}
        \caption{Distribution of the gain $g_i=\sigma_{\mathrm{unif}}/\sigma_i$. Dotted
        lines: gains of $2$, $5$ and $10$.}
        \label{fig:nyu-gain-ecdf}
    \end{subfigure}
    \hfill
    \begin{subfigure}[t]{0.49\linewidth}
        \centering
        \includegraphics[width=\linewidth]{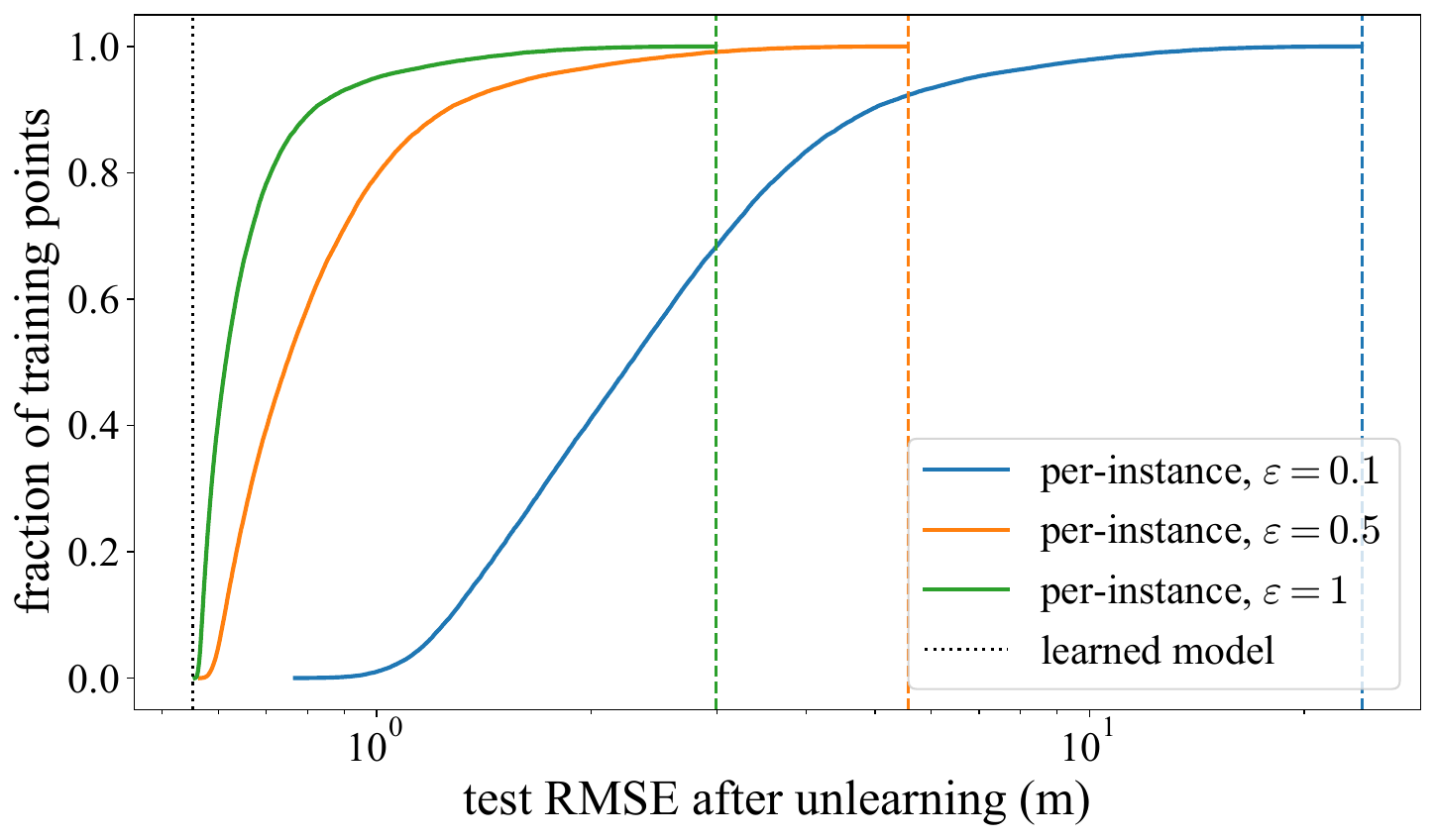}
        \caption{Distribution of the test RMSE after unlearning. Dashed lines:
        uniform calibration at the same $\eps$ (same color).}
        \label{fig:nyu-rmse-ecdf}
    \end{subfigure}
    \caption{Per-instance against uniform calibration over all $n=38\,400$ training points
    of NYU-Depth ($T=800$, $K=80$, $\slearn=0.1$). Each curve is an empirical cumulative
    distribution function over the training points.}
    \label{fig:nyu-gain}
\end{figure}

With the uniform level, the test RMSE after unlearning (Figure~\ref{fig:nyu-rmse-ecdf}) is
$24.1$\,m at $\eps=0.1$, $5.6$\,m at $\eps=0.5$ and $3.0$\,m at $\eps=1$, whatever the
deleted point, for depth targets below $10$\,m. With per-instance calibration, its median
over the deletion requests is $2.3$\,m, $0.75$\,m and $0.61$\,m, and $90\%$ of the
requests stay below $4.9$\,m, $1.25$\,m and $0.81$\,m. Table~\ref{tab:nyu-rmse-check}
reports Monte Carlo estimates for the five representative points.

\begin{table}[!htbp]
\centering
\caption{Test RMSE (in meters) after unlearning each of the five representative points of
NYU-Depth: Monte Carlo estimate (mean $\pm$ standard deviation over $5$ independent
learning and unlearning runs).}
\label{tab:nyu-rmse-check}
\begin{tabular}{r c c}
\toprule
Index & $\eps=0.1$ & $\eps=1$ \\
\midrule
$30170$ & $0.763\pm0.037$  & $0.554\pm0.006$  \\
$18625$ & $1.655\pm0.057$  & $0.580\pm0.009$  \\
$1873$  & $2.166\pm0.149$  & $0.608\pm0.014$  \\
$35054$ & $3.000\pm0.091$ &  $0.643\pm0.032$  \\
$15032$ & $5.752\pm0.418$  & $0.886\pm0.058$  \\
\bottomrule
\end{tabular}
\end{table}

\subsection{Deterministic Training: Softmax Regression on CIFAR-10}
\label{app:cifar}

Theorem~\ref{thm:deterministic} applies to any objective whose gradient maps are
contractive. We apply it to softmax regression on
CIFAR-10~\citep{krizhevsky2009learning}, with deterministic training, and compare
per-instance and uniform calibration over the whole training set. Each image is resized to
$224\times224$ and represented by the class token of the frozen \texttt{dinov2\_vitb14}
backbone of Section~\ref{sec:experiments-linear}, standardized with the training
statistics and completed with a constant feature ($p=769$, $d=10$ classes). We draw
$n=20\,000$ training images and $5\,000$ validation images from the $50\,000$ training
images, and compute all accuracies on the $10\,000$ test images. With
$\theta\in\R^{p\times d}$ and $e_y\in\R^d$ the one-hot encoding of the label $y$, the
objective is
\[
f_{\D}(\theta)
=
-\sum_{j=1}^n\log\bigl[\operatorname{softmax}(\theta^\top x_j)\bigr]_{y_j}
+\frac{\lambda}{2}\|\theta\|_F^2 ,
\]
minimized by full-batch gradient descent from $\theta_0=0$, without noise ($\slearn=0$).
It is $\lambda$-strongly convex and $L$-smooth with
$L=\tfrac12\lambda_{\max}(X^\top X)+\lambda$, so Assumption~\ref{ass:contraction} holds
with $\eta=1/L$ and $c=1-\eta\lambda$. We take $T=1000$, $K=100$, $\del=1/n$ and
$\lambda=0.01\,n$, selected on the validation images among
$\lambda\in\{10^{-4},\dots,1\}\times n$ (Table~\ref{tab:cifar-lambda}), which gives
$c\approx0.99935$. The learned model has a test accuracy of $98.1\%$.

\begin{table}[!htbp]
\centering
\caption{Validation of the regularization parameter on CIFAR-10: training and validation
accuracy of the model learned by $T=1000$ steps of gradient descent, without noise and
without unlearning. Bold: maximum of the validation accuracy, which is the retained value.}
\label{tab:cifar-lambda}
\begin{tabular}{c c c c}
\toprule
$\lambda/n$ & $\lambda$ & Train accuracy & Validation accuracy \\
\midrule
$10^{-4}$ & $2$        & $99.28\%$ & $98.16\%$ \\
$10^{-3}$ & $20$       & $99.27\%$ & $98.14\%$ \\
$10^{-2}$ & $200$      & $99.13\%$ & $\mathbf{98.20\%}$ \\
$10^{-1}$ & $2000$     & $98.44\%$ & $98.08\%$ \\
$1$       & $20\,000$  & $97.81\%$ & $97.72\%$ \\
\bottomrule
\end{tabular}
\end{table}

The per-instance sensitivity~\eqref{eq:sensitivity} is
\[
\Delta_{i,k}
=
\eta\,\|x_i\|_2\,\bigl\|\operatorname{softmax}(\theta_k^\top x_i)-e_{y_i}\bigr\|_2 ,
\]
and it is exact since training is deterministic. For all training points, it is
accumulated during the single training run through the recursion
$\tilde N_i\leftarrow c\,\tilde N_i+\Delta_{i,k}$ of Section~\ref{sec:deterministic}. A
deletion request for point $i$ is served with $N_i=c^K\tilde N_i$ and the noise
level~\eqref{eq:mu-deterministic}, denoted $\sigma_i$. We compare it with two uniform
levels obtained in the same way: one from the worst training point at each step,
$\max_j\Delta_{j,k}$, as in Appendix~\ref{app:gain-uniform}, and one from the Lipschitz
bound $\Delta_{i,k}\le\sqrt2\,\eta\max_j\|x_j\|_2$, which uses no information on the
trajectory. The noise level~\eqref{eq:mu-deterministic} is proportional to the weighted sum
of the sensitivities, so the gain $g_i=\sigma_{\mathrm{unif}}/\sigma_i$ depends neither on
$\eps$ nor on $\del$.

Table~\ref{tab:cifar-gain} and Figure~\ref{fig:cifar-gain-ecdf} give the distribution of
the gain. With respect to the worst training point, the median gain is $161$, and the gain
is at least $10$ for $92.1\%$ of the points and at least $2$ for $98.9\%$. At $\eps=0.1$,
for instance, $\sigma_{\mathrm{unif}}=94.8$ ($109.9$ with the Lipschitz bound) and the
median of $\sigma_i$ is $0.59$ (Figure~\ref{fig:cifar-sigma}). With the cross-entropy
loss, the gradient of a point that the model fits vanishes along training, whereas it
stays close to its largest possible value for the few points that remain misclassified.
As a result, $N_i$ mostly reflects how well the learned model fits the point: its Spearman
correlation with $1-p_{\theta_T}(y_i\mid x_i)$ is $0.98$, and the $173$ training points
misclassified by the learned model all have a gain below $2.2$. In ridge regression, the
residuals do not vanish and the bounds of Theorem~\ref{thm:ridge-gdp-main} carry an
additional noise term, which explains the narrower spread of
Appendix~\ref{app:gain-uniform}.

\begin{table}[!htbp]
\centering
\caption{Gain $g_i=\sigma_{\mathrm{unif}}/\sigma_i$ of per-instance over uniform calibration
for softmax regression on CIFAR-10 ($n=20\,000$), for the two uniform references: quantiles
over the training points, and share of points whose gain is at least $2$, $5$ and $10$. The
gain does not depend on $\eps$.}
\label{tab:cifar-gain}
\small\setlength{\tabcolsep}{4.5pt}
\begin{tabular}{l c c c c c c c c c c c}
\toprule
 & \multicolumn{8}{c}{Quantiles of the gain} & \multicolumn{3}{c}{Share with gain} \\
\cmidrule(lr){2-9}\cmidrule(lr){10-12}
Uniform reference & min & $10\%$ & $25\%$ & $50\%$ & $75\%$ & $90\%$ & $99\%$ & max
 & $\ge2$ & $\ge5$ & $\ge10$ \\
\midrule
Worst training point & $1.00$ & $13.8$ & $53.3$ & $161$ & $347$ & $555$ & $807$ & $1029$
 & $98.9\%$ & $95.7\%$ & $92.1\%$ \\
Lipschitz bound      & $1.16$ & $16.0$ & $61.7$ & $186$ & $402$ & $643$ & $935$ & $1192$
 & $99.3\%$ & $96.4\%$ & $93.0\%$ \\
\bottomrule
\end{tabular}
\end{table}

\begin{figure}[!htbp]
    \centering
    \begin{subfigure}[t]{0.49\linewidth}
        \centering
        \includegraphics[width=\linewidth]{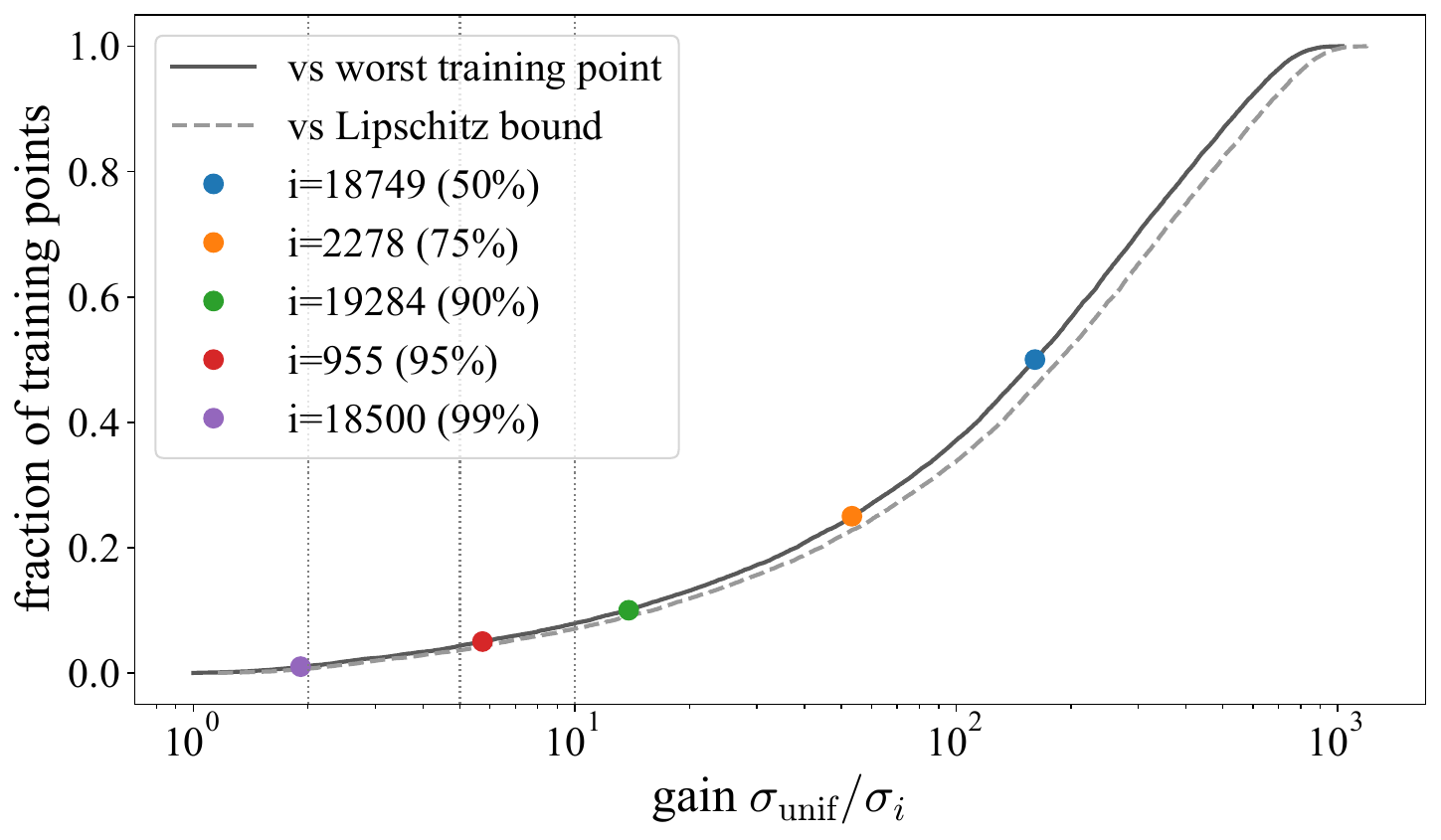}
        \caption{Distribution of the gain $g_i$ over the training points, for the two
        uniform references. Markers: the five points of
        Figure~\ref{fig:cifar-utility}, with their quantile of $N_i$. Dotted lines: gains
        of $2$, $5$ and $10$.}
        \label{fig:cifar-gain-ecdf}
    \end{subfigure}
    \hfill
    \begin{subfigure}[t]{0.49\linewidth}
        \centering
        \includegraphics[width=\linewidth]{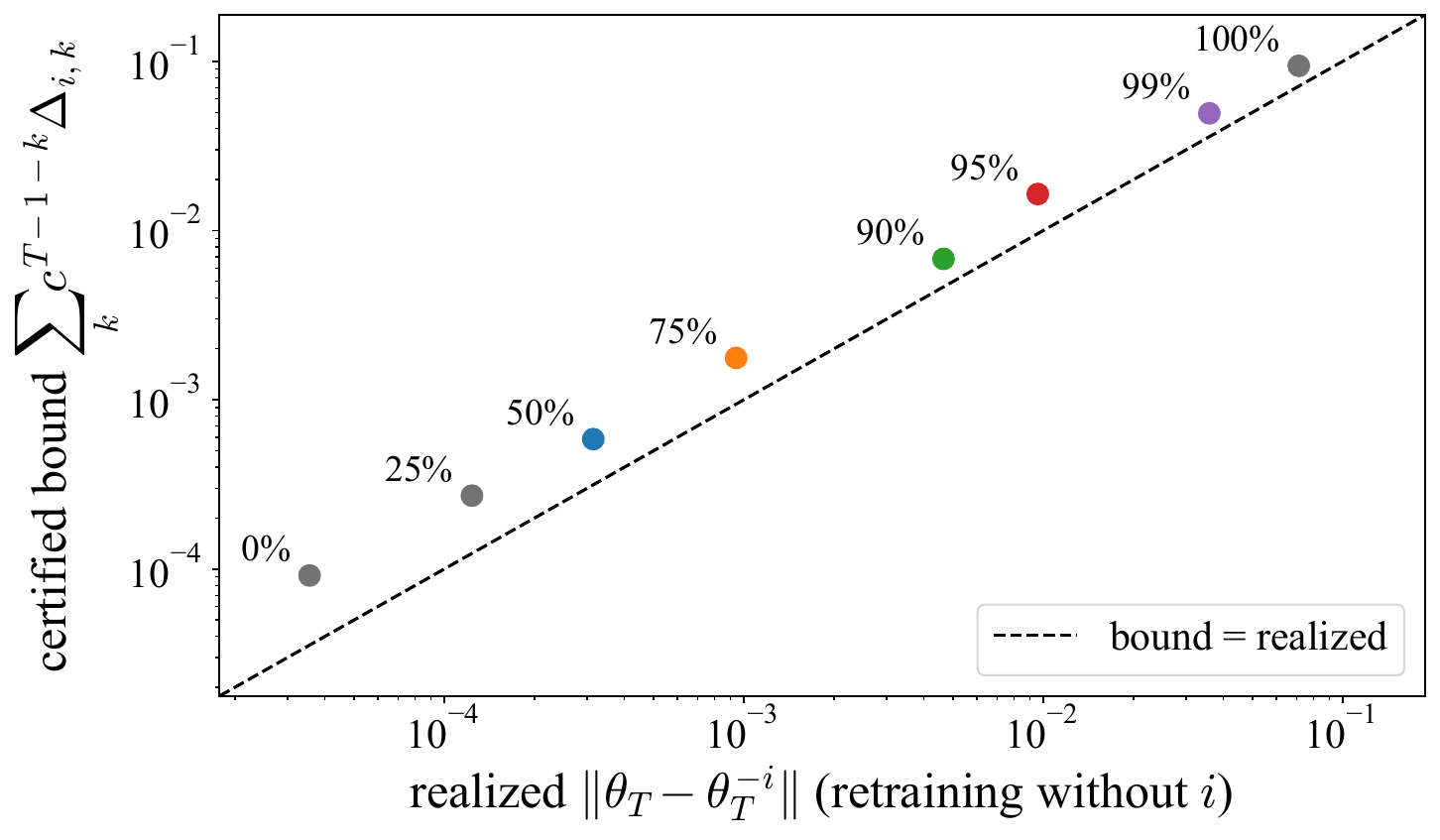}
        \caption{Certified bound $\tilde N_i$ against the realized displacement
        $\|\theta_T-\theta_T^{-i}\|$ obtained by retraining without point $i$, for eight
        points labeled by their quantile of $N_i$.}
        \label{fig:cifar-retraining}
    \end{subfigure}
    \caption{Theorem~\ref{thm:deterministic} on softmax regression on CIFAR-10: per-instance
    gain and retraining check.}
    \label{fig:cifar-certificate}
\end{figure}

To check that these gains reflect the influence of the points and not the slack of the
certificate, we retrain without point $i$ for eight points spread over the quantiles of
$N_i$. The proof of Theorem~\ref{thm:deterministic} bounds the displacement
$\|\theta_T-\theta_T^{-i}\|$ between the learned model and the model retrained without
point $i$ by $\tilde N_i$ (inequality~\eqref{eq:det-displacement}). The realized
displacement varies by a factor of $2000$ across these points, and the bound exceeds it by
a factor between $1.3$ and $2.6$ (Figure~\ref{fig:cifar-retraining},
Table~\ref{tab:cifar-retraining}).

\begin{table}[!htbp]
\centering
\caption{Retraining check on CIFAR-10: displacement
$\|\theta_T-\theta_T^{-i}\|$ realized by retraining without point $i$, and its certified
bound $\tilde N_i=\sum_{k<T}c^{T-1-k}\Delta_{i,k}$, for eight points taken at the indicated
quantiles of $N_i$.}
\label{tab:cifar-retraining}
\begin{tabular}{l c c c c c c c c}
\toprule
Quantile of $N_i$ & min & $25\%$ & $50\%$ & $75\%$ & $90\%$ & $95\%$ & $99\%$ & max \\
\midrule
Realized displacement ($\times10^{-3}$) & $0.035$ & $0.124$ & $0.314$ & $0.942$ & $4.64$ & $9.59$ & $35.9$ & $71.3$ \\
Certified bound $\tilde N_i$ ($\times10^{-3}$) & $0.092$ & $0.271$ & $0.585$ & $1.77$ & $6.80$ & $16.4$ & $49.3$ & $94.1$ \\
Ratio & $2.58$ & $2.19$ & $1.86$ & $1.88$ & $1.46$ & $1.71$ & $1.37$ & $1.32$ \\
\bottomrule
\end{tabular}
\end{table}

Figure~\ref{fig:cifar-utility} shows the test accuracy after unlearning for the two uniform
calibrations and for five points at the $50\%$, $75\%$, $90\%$, $95\%$ and $99\%$ quantiles
of $N_i$ (gains $161$, $53$, $13.8$, $5.7$ and $1.9$), averaged over $20$ draws of the
unlearning noise. We take these points in the upper half of the distribution because the
median point already stays within $0.1$ point of the learned model at every $\eps$. At
$\eps=0.1$, uniform calibration brings the test accuracy down to $41.8\%$ ($38.2\%$ with
the Lipschitz bound), and it still costs $16$ points at $\eps=1$ ($82.0\%$). With
per-instance calibration, the $90\%$, $95\%$ and $99\%$ points reach $88.1\%$, $76.9\%$ and
$56.7\%$ at $\eps=0.1$, and $97.9\%$, $96.8\%$ and $89.1\%$ at $\eps=1$.

\begin{figure}[!htbp]
    \centering
    \begin{subfigure}[t]{0.49\linewidth}
        \centering
        \includegraphics[width=\linewidth]{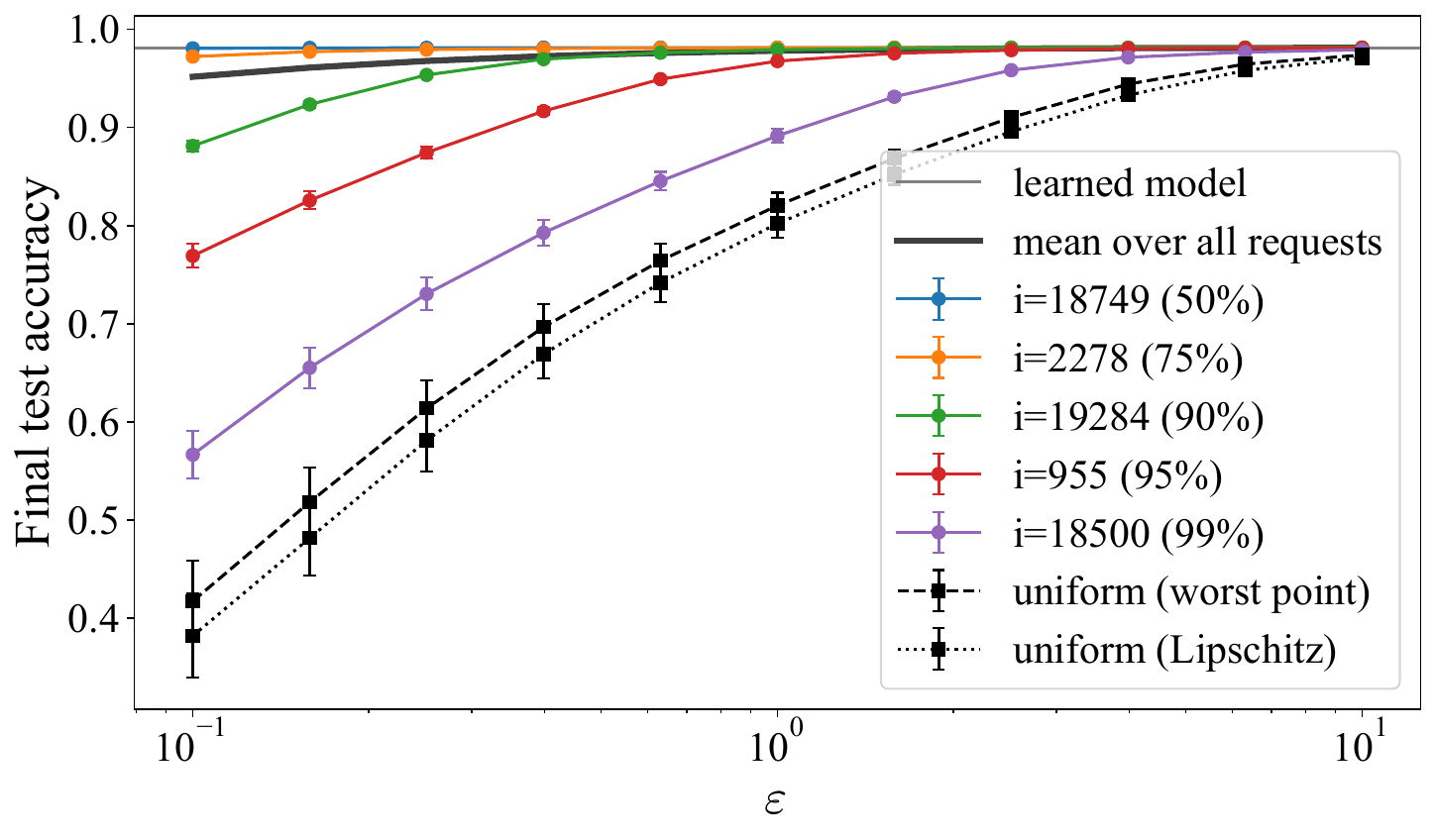}
        \caption{Test accuracy after unlearning against $\eps$, for five points at the
        indicated quantiles of $N_i$, for the two uniform calibrations, and on average over
        the $n$ possible deletion requests.}
        \label{fig:cifar-utility}
    \end{subfigure}
    \hfill
    \begin{subfigure}[t]{0.49\linewidth}
        \centering
        \includegraphics[width=\linewidth]{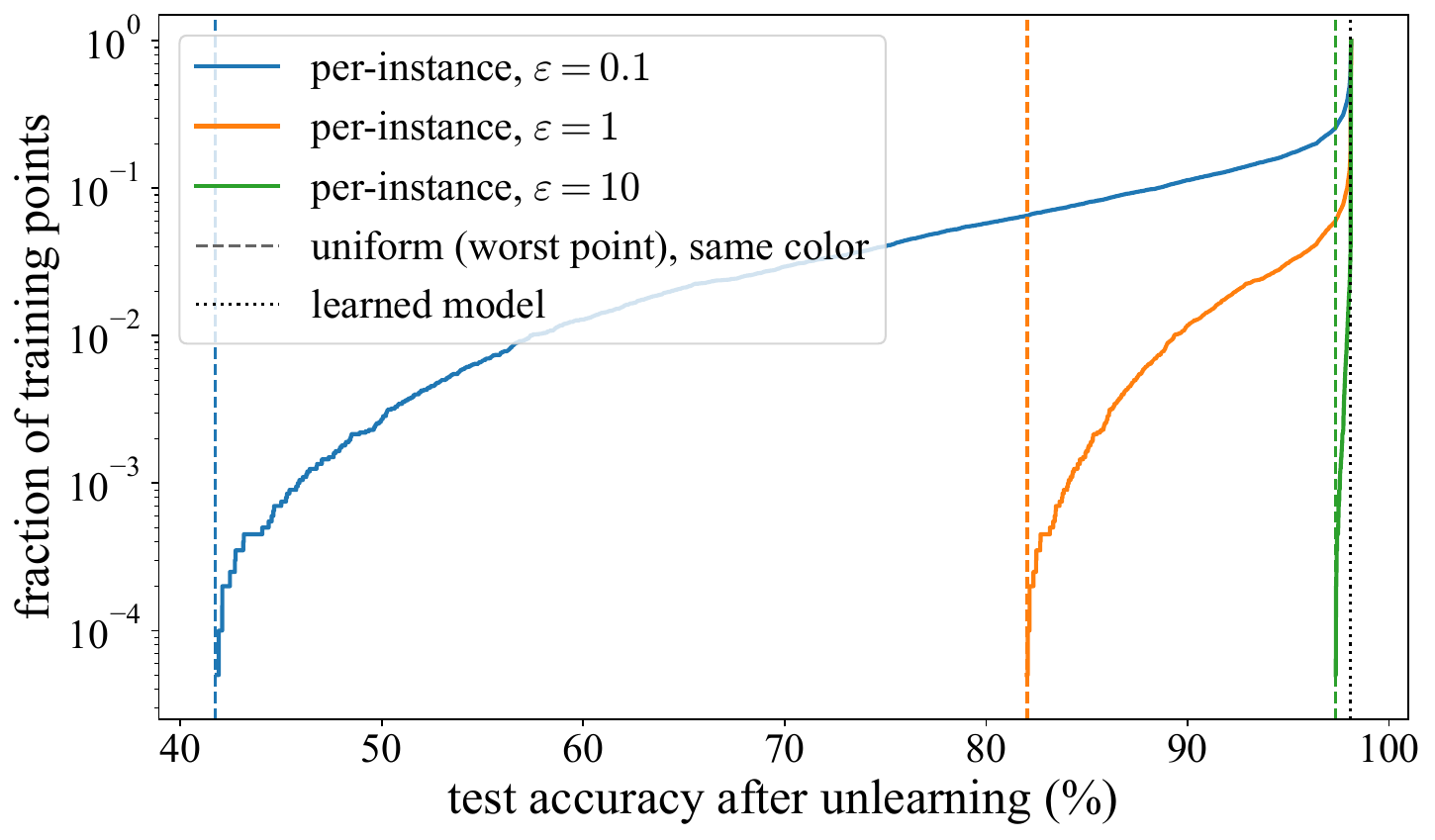}
        \caption{Distribution over the training points of the test accuracy after
        unlearning (cumulative distribution function, logarithmic scale). Dashed lines:
        uniform calibration (worst point) at the same $\eps$.}
        \label{fig:cifar-all-points}
    \end{subfigure}
    \caption{Privacy--utility trade-off for softmax regression on CIFAR-10 ($T=1000$,
    $K=100$, deterministic training). Error bars show one standard deviation over $20$
    draws of the unlearning noise.}
    \label{fig:cifar-privacy-utility}
\end{figure}

\begin{figure}[!htbp]
    \centering
    \includegraphics[width=0.5\linewidth]{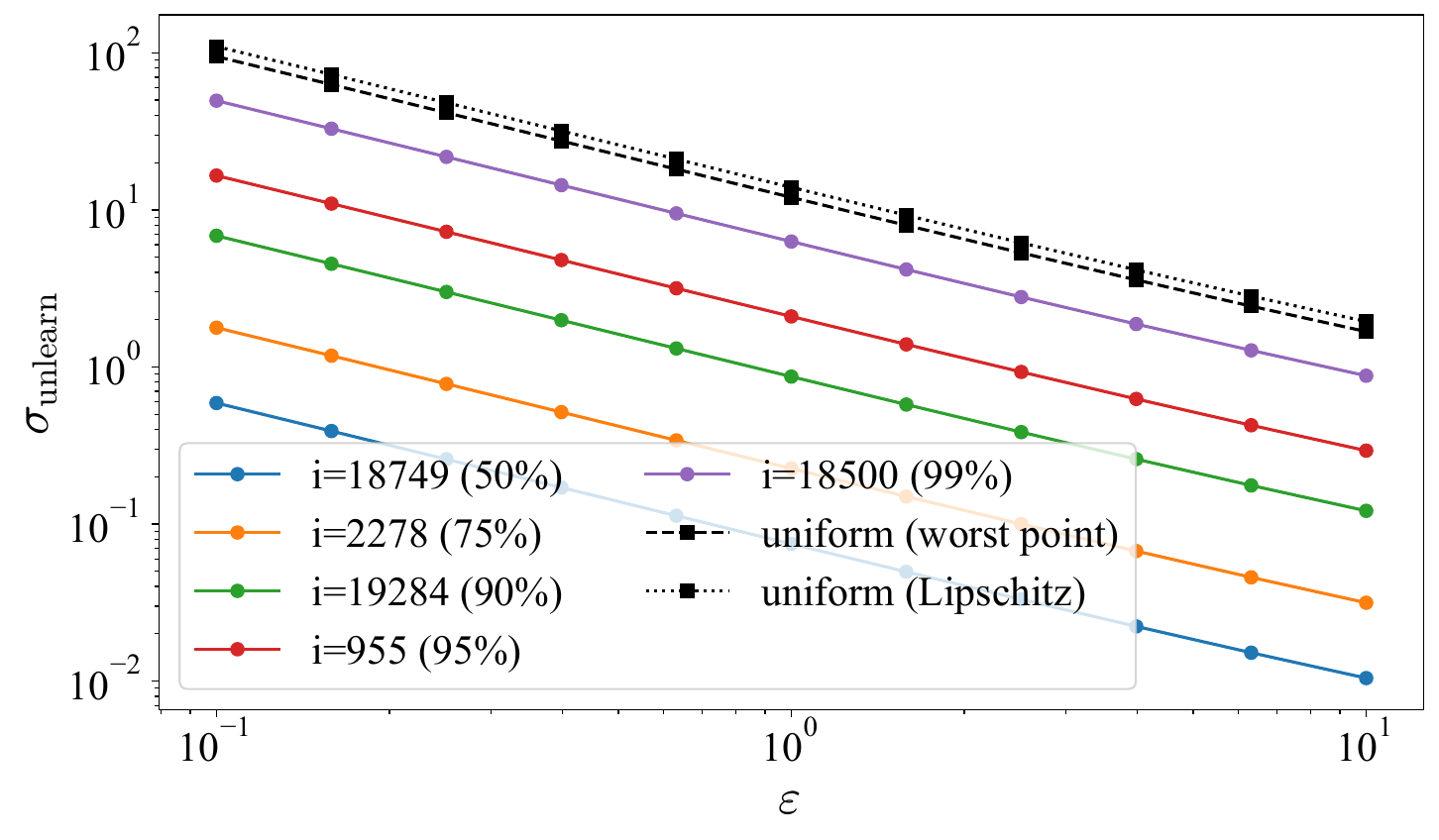}
    \caption{Unlearning noise level $\sunl$ against $\eps$ for the five points of
    Figure~\ref{fig:cifar-utility} and for the two uniform calibrations (CIFAR-10). The
    curves are parallel: the gain does not depend on $\eps$.}
    \label{fig:cifar-sigma}
\end{figure}

A single point out of $n$ is removed from the retain objective, so the accuracy after
unlearning depends on the deleted point essentially through its noise level. We estimate
the accuracy as a function of the noise level on a grid of $24$ levels, deleting the median
point, and read off this curve the accuracy after each of the $n$ possible deletion
requests; interpolating the curve reproduces the direct runs of
Figure~\ref{fig:cifar-utility} within $0.8$ point. At $\eps=0.1$, the test accuracy
averaged over all deletion requests is $95.2\%$, against $41.8\%$ for uniform calibration,
and $76\%$ of the requests cost at most one point of accuracy. At $\eps=1$, the average is
$97.8\%$ against $82.0\%$, and $94.5\%$ of the requests cost at most one point
(Table~\ref{tab:cifar-all-points}, Figure~\ref{fig:cifar-all-points}).

\begin{table}[!htbp]
\centering
\caption{Test accuracy (in $\%$) after unlearning on CIFAR-10, over the $n=20\,000$
possible deletion requests; the learned model has a test accuracy of $98.1\%$. The last
column is the share of requests that cost at most one point of accuracy.}
\label{tab:cifar-all-points}
\begin{tabular}{c c c c c c c c}
\toprule
 & \multicolumn{2}{c}{Uniform calibration} & \multicolumn{5}{c}{Per-instance calibration} \\
\cmidrule(lr){2-3}\cmidrule(lr){4-8}
 & & & Mean over & \multicolumn{3}{c}{Accuracy exceeded by} & Loss \\
$\eps$ & Lipschitz & Worst point & requests & $50\%$ & $90\%$ & $99\%$ of requests & $\le1$ point \\
\midrule
$0.1$ & $38.2$ & $41.8$ & $95.2$ & $98.0$ & $88.5$ & $57.4$ & $76.0\%$ \\
$1$   & $80.2$ & $82.0$ & $97.8$ & $98.1$ & $97.9$ & $89.3$ & $94.5\%$ \\
$10$  & $97.1$ & $97.4$ & $98.1$ & $98.1$ & $98.1$ & $97.9$ & $100\%$ \\
\bottomrule
\end{tabular}
\end{table}

\section{Computing Resources}
\label{app:compute}

All experiments reported in the paper were run on a single laptop with an Apple silicon
chip, in single precision with the MPS backend of PyTorch, except where stated otherwise;
no cluster was used.
Existing assets are credited in the paper and in this appendix:
NYU-Depth~\citep{silberman2012indoor} (original dataset; the official dataset
page does not specify a standard redistribution license),
CIFAR-10~\citep{krizhevsky2009learning} (original dataset; the official dataset page does not
specify a license and asks to cite the technical report),
DINOv2~\citep{oquab2024dinov2}
(Apache-2.0 for the code; CC-BY-4.0 for the released model assets). The DP-GD baseline~\citep{Abadi_2016}, the Langevin unlearning
baseline~\citep{chien2025langevinunlearningnewperspective} and the Newton-step
baseline~\citep{guo2020certified} are re-implemented from the published
descriptions and attributed accordingly.

\paragraph{Ridge regression experiments on NYU-Depth.}
The DINOv2 features of the $200$ images are extracted once and cached. The wall-clock
times below are those of the runs that produced the figures. One learning run
($n=38\,400$, $T=800$) takes about $3.7$ seconds: the $40$ trajectories simulated for each
of the five points of Figure~\ref{fig:nyuv2-hp-sensitivity-map} take $12$ minutes in total.
The privacy--utility sweep over $10$ values of $\varepsilon$ and $5$ representative points
with $10$ runs each (Figures~\ref{fig:nyuv2-privacy-utility}
and~\ref{fig:sigma-per-instance}) takes $5.5$ minutes, and the same sweep on the full-range
points (Figure~\ref{fig:nyuv2-loglog-all}) $5.7$ minutes. The ablation over $K$
(Figure~\ref{fig:ablation-K}) takes $9$ minutes and the ablation over $T$
(Figure~\ref{fig:ablation-T}) $5.4$ minutes. The comparison with Langevin unlearning
(Figure~\ref{fig:nyuv2-ours-vs-langevin}) takes $44$ minutes, and the comparison with
learning-time privacy (Figure~\ref{fig:nyuv2-ours-vs-dpgd}) $34$ minutes, of which $28$
minutes for the two DP-GD accountants. The four cyclic mini-batch experiments
(Figure~\ref{fig:nyuv2-minibatch-tradeoff}) take $10$ minutes in total, from $4.2$ minutes
for $b=256$ to $1.7$ minutes for $b=2048$. For the comparison with the Newton step of
\citet{guo2020certified} (Figure~\ref{fig:nyuv2-ours-vs-guo}), the PILU guarantees take
$1.9$ minutes, and the Newton-step baseline, trained up to $T=20\,000$ steps, a few
additional minutes. Altogether, the experiments on NYU-Depth run in about two and a half
hours.

\paragraph{Additional experiments.}
The closed-form sensitivity bounds and noise levels of all $38\,400$ training points
(Appendix~\ref{app:gain-uniform}) are computed in $6.6$ seconds on CPU, including $1.4$
seconds for the eigendecomposition; the Monte Carlo check of
Table~\ref{tab:nyu-rmse-check} takes about $3$ minutes. For CIFAR-10
(Appendix~\ref{app:cifar}), extracting the DINOv2 features of the $60\,000$ images takes
less than half an hour, after which they are cached. Training with the accumulators of
the $n=20\,000$ points ($T=1000$) takes $5.2$ seconds. The Monte Carlo evaluations
($20$ runs per noise level, for the five representative points, the two uniform
calibrations and the $24$ noise levels behind Table~\ref{tab:cifar-all-points}) take
$89$ seconds in total. The validation of $\lambda$ (Table~\ref{tab:cifar-lambda}, five
training runs) and the eight retrainings of Table~\ref{tab:cifar-retraining} use the same
training loop without the accumulators, so each costs at most one such training run.

\end{document}